\documentclass[14pt a4paper]{article}
\usepackage{subfigure,graphicx}
\usepackage{float}
\usepackage{booktabs}
\usepackage{color,multirow}
\usepackage{makecell}
\usepackage{flafter}
\usepackage{ifthen}
\usepackage{cite} 
\usepackage{url}  
\usepackage{multirow}
\usepackage{makecell}
\usepackage{pifont}
\usepackage{threeparttable}
\usepackage[T1]{fontenc}
\usepackage{authblk}
\usepackage{graphicx}
\usepackage{epstopdf}
\usepackage{geometry}
\usepackage{amsmath,amssymb}
\usepackage{bm}
\usepackage{algorithm}
\usepackage{algpseudocode}
\usepackage{amsmath}
\usepackage{cite}
\usepackage[colorlinks,linkcolor=red,anchorcolor=gray,citecolor=green,urlcolor=blue]{hyperref}

\DeclareMathOperator*{\argmin}{argmin}

\newtheorem{remark}{Remark}
\newtheorem{lemma}{Lemma}
\newtheorem{theorem}{Theorem}
\newtheorem{definition}{Definition}

\DeclareMathOperator*{\argmax}{argmax}
\newsavebox\CBox
\def\textBF#1{\sbox\CBox{#1}\resizebox{\wd\CBox}{\ht\CBox}{\textbf{#1}}}
\newenvironment{proof}{{\noindent\it Proof.}\ }{\hfill $\square$\par}

\author{Ben-Zheng Li
\thanks{B.-Z. Li, X.-L. Zhao, and T.-Z Huang are with the School of Mathematical Sciences, University of Electronic Science and Technology of China, Chengdu, Sichuan 611731, P. R. China. E-mails: lbz1604179601@126.com, xlzhao122003@163.com, and tingzhuhuang@126.com.},
Xi-Le Zhao,
Teng-Yu Ji
\thanks{T.-Y Ji is with the School of Mathematics and Statistics, Northwestern Polytechnical University, Xian, 710072, Shanxi, P. R. China. E-mail:tengyu\_j66@126.com.},
Xiong-Jun Zhang
 \thanks{X.-J Zhang is with the School of Mathematics and Statistics and Hubei Key
Laboratory of Mathematical Sciences, Central China Normal University,
Wuhan 430079, China. E-mail: xjzhang@mail.ccnu.edu.cn.},
Ting-Zhu Huang}
\begin{document}
\title{Nonlinear Transform Induced Tensor Nuclear Norm for Tensor Completion}

\maketitle

\begin{abstract}

 The linear transform-based tensor nuclear norm (TNN) methods have recently obtained promising results for
tensor completion. The main idea of this type of methods is exploiting the low-rank structure of frontal
slices of the targeted tensor under the linear transform along the third mode. However, the low-rankness
of frontal slices is not significant under linear transforms family. To better pursue the low-rank approximation, we
propose a nonlinear transform-based TNN (NTTNN). More concretely, the proposed nonlinear transform is a composite transform consisting of the linear semi-orthogonal transform along the third mode and the element-wise nonlinear transform on frontal slices of the tensor under the linear semi-orthogonal transform, which are indispensable and complementary in the composite transform to fully exploit the underlying low-rankness. Based on the suggested low-rankness metric, i.e., NTTNN, we propose a low-rank tensor completion (LRTC) model. To tackle the resulting nonlinear and nonconvex optimization model, we elaborately design the proximal alternating minimization (PAM) algorithm and establish the theoretical convergence guarantee of the PAM algorithm. Extensive experimental results on hyperspectral images, multispectral images, and videos show that the our method outperforms linear transform-based state-of-the-art LRTC methods qualitatively and quantitatively.
\end{abstract}



\section{Introduction}
With the development of scientific computing, high-dimensional data structure is
becoming more and more complicated.
As the high-dimensional extension of vectors and matrices,
tensors can represent higher-dimensional data,
such as hyperspectral images (HSIs) \cite{baiminru,Hewei1}, multispectral images (MSIs) \cite{dingmengjsc}, and videos \cite{zhangxiongjunpami}, which play an increasingly important role in large-scale scientific computing.
However, tensor data frequently undergo missing entries or undersample problem due to various unpredictable or unavoidable situations when acquiring it.
The problem
of recovering missing entries via the observed incomplete tensor is called tensor
completion (TC) \cite{jsctc1}, which is a fundamental problem and has received
considerable attention in scientific computing \cite{jsc1,jsc2,jsc3,zengjinshan}. Generally, multi-dimensional data is internally correlated and the internal redundancy property could be measured by the powerful rank function. Therefore, the low-rank TC (LRTC) problem can be formulated as follows
\begin{figure}[htp!]
\vspace{-1cm}
\hspace{-1.5cm}
  \includegraphics[width=17cm,height=8cm]{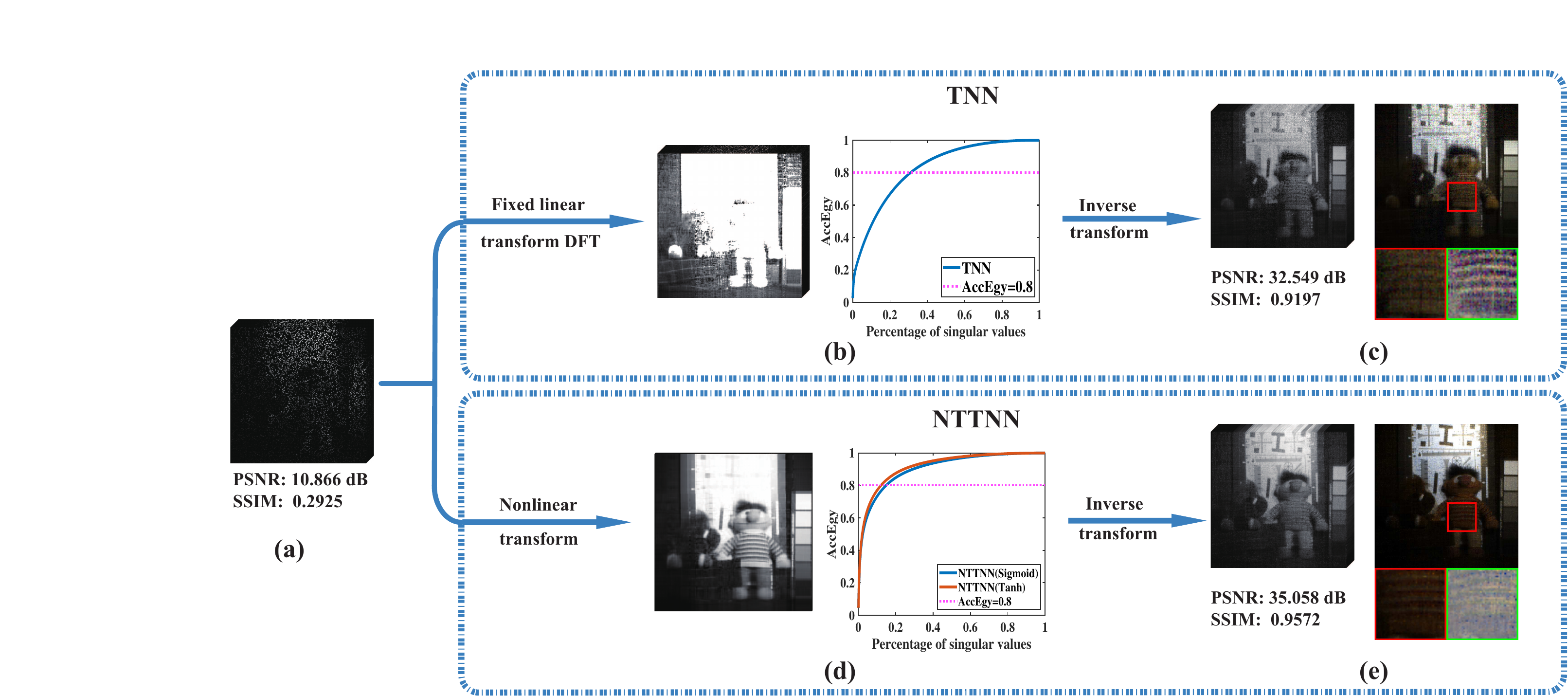}
 \caption{The pipeline of the linear transform-based TNN method and the proposed NTTNN method on MSI $\mathit{Toy}$ for sampling rate (SR) 10\%. (a)The observed incomplete tensor. (b)The transformed tensor under the DFT and the AccEgy with the corresponding percentage of singular values. (c)The recovered tensor by the TNN method and the corresponding pseudo-color image composed of the 1-st, 2-nd, and 5-th bands with the zoom-in region and the corresponding residual part.
 (d)The transformed tensor under the proposed nonlinear transform $\psi$ and the AccEgy with the corresponding percentage of singular values.
 (e)The recovered tensor by the NTTNN method and the corresponding pseudo-color image composed of the 1-st, 2-nd, and 5-th bands with the zoom-in region and the corresponding residual part.}
 \label{dongjitu}
\end{figure}
\begin{equation*}
    \begin{aligned}
    &\min \limits_{\mathcal{X}} \ \textrm{rank}(\mathcal{X}) \\
    & \text{s.t. } \ \mathcal{X}_{\Omega} = \mathcal{O}_{\Omega},
    \end{aligned}	
\end{equation*}
where, $\mathcal{X}$ and $\mathcal{O}$ denote the required and the observed tensors, respectively, and $\Omega$ is the index set of the observed elements.

 However, different from the matrix case, the tensor rank has no unique definition.
 The CP-rank \cite{CPrank} is defined as the minimum number of rank-one tensors that generate the target tensor, which has been successfully applied in LRTC \cite{CP1,xuejize2}. However, to determine the CP-rank of the target tensor is NP-hard \cite{CPNPhard}. The Tucker-rank is defined as a vector constituted of ranks of each mode-$k$ matricization of the tensor. The Tucker-rank has been applied in LRTC problem by minimizing its convex surrogate \cite{HaLRTC} or non-convex surrogates \cite{Tucker1,Tucker2}.
 Moreover, a series of tensor network decomposition-based ranks are proposed, such as tensor train (TT)-rank \cite{TTrank}, tensor ring
(TR)-rank \cite{TR}, and fully-connected tensor network (FCTN)-rank \cite{FCTN}. All of them have been achieved great success in
higher-order LRTC \cite{2017TIPTT,FCTN,TRLRF}.

 Recently, the tensor tubal-rank \cite{klimer} has been proposed,
 which avoids the loss of inherent information in unfolding of the target tensor.
 Since minimizing the tubal-rank of the target tensor is NP-hard, Zhang $\mathit{et}$ $\mathit{al}.$ \cite{zhangtnn} proposed a convex surrogate, the tensor nuclear norm (TNN) of underlying tensor, to solve LRTC problem.
 The TNN-based LRTC model could be mathematically rewritten as
 \begin{equation}
    \begin{aligned}
    &\min_{\mathcal{X}}\|\mathcal{X}\|_{\operatorname{TNN}}\\
    &\text{s.t. } \ \mathcal{X}_{\Omega} = \mathcal{O}_{\Omega},
    \end{aligned}\label{TNN1}	
    \end{equation}
where $\|\mathcal{X}\|_{\operatorname{TNN}}$ is TNN of $\mathcal{X}$ (see Def. \ref{tnn}).
 Given $\mathcal{X}\in\mathbb{R}^{n_1\times{n_2}\times{n_3}}$, we define $\mathcal{Z} = \mathcal{X}\times_3\textbf{F}_{n_3}$,
 where $\textbf{F}_{n_3}$ and $\times_3$ respectively denotes the Discrete Fourier Transform (DFT) and the mode-3 product of a tensor and a matrix. Since $\textbf{F}_{n_3}$ is invertible, we have $\mathcal{X}=\mathcal{Z}\times_3\textbf{F}_{n_3}^{-1}$. Combining the definition of TNN, the problem (\ref{TNN1}) is equivalent to the following problem:
 \begin{equation}
    \begin{aligned}\label{TNN2}
    &\min \limits_{\mathcal{Z}} \sum_{i=1}^{n_3}\|\mathbf{Z}_i\|_* \\
    &\ \text{s.t.} (\mathcal{Z}\times_3\textbf{F}_{n_3}^{-1})_{\Omega} = \mathcal{O}_{\Omega},
    \end{aligned}
\end{equation}
where $\textbf{F}_{n_3}^{-1}$ denotes the inverse DFT matrix \cite{klimer}, $\|\cdot\|_*$ is the matrix nuclear norm, and $\mathbf{Z}_i$ denotes $i$-th frontal slice of $\mathcal{Z}$.
The problem (\ref{TNN2}) implies that low-tubal-rank structure could be characterized by the summation of nuclear norm of frontal slices under the linear
DFT.
\begin{figure}[htp!]
  \centering
  \includegraphics[width=14cm,height=5cm]{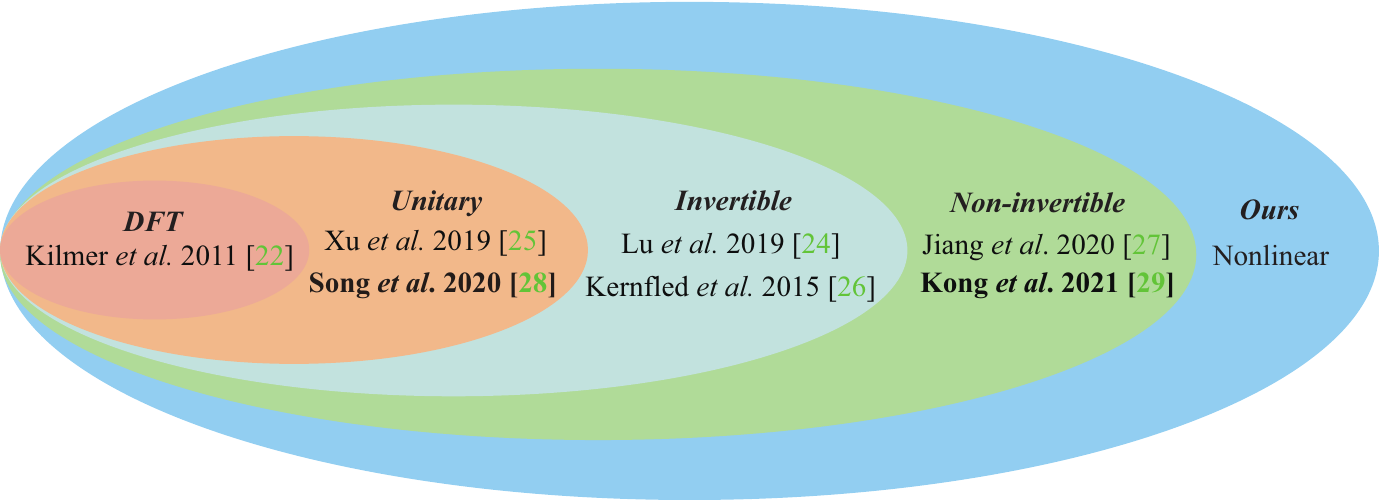}
 \caption{The evolution of transform-based TNN methods. The learned transforms are highlighted in boldface, while the other transforms are pre-defined.}\label{yanhuatu}
\end{figure}

To obtain a better low-rank approximation of frontal slices of the transformed tensor, researchers consider different linear transforms instead of DFT.
The discrete cosine transform (DCT) was proposed
as an alternatively implemented of DFT in \cite{DCT_Lu,DCT_based}, which enjoys less computational cost and obtains a better low-rank approximation compared with DFT.
Latterly, Kernfeld $\mathit{et}$ $\mathit{al}.$ \cite{svd_unitary_transform}
and Lu $\mathit{et}$ $\mathit{al}.$ \cite{DCT_Lu} noted that the DFT can be replaced by any invertible
transform.
Subsequently, Jiang $\mathit{et}$ $\mathit{al}.$ \cite{framelet_jiang} introduced the semi-invertible framelet transform,
which broke through the restriction of invertibility compared with invertible transform.
It is worth noting that all above mentioned transforms
are pre-defined and data-independent. Choosing above transforms are appealing
because it is cheap and robust to initialization.
Another prominent paradigm is to employ the learning-based methods
that further make transforms data-dependent. Song $\mathit{et}$ $\mathit{al}.$ \cite{unitary_transform} suggested the transform generated by singular values decomposition (SVD) of mode-3 unfolding of the observed
tensor. Kong $\mathit{et}$ $\mathit{al}.$ \cite{tensor_Q_rank} and Jiang $\mathit{et}$ $\mathit{al}.$ \cite{DTNN} further proposed
that the observed tensor is used to initialize the transform, which is updated in the algorithm.
The above learned methods have shown that the learned
transform is flexible for various types of data. The evolution of transform-based TNN
methods are summarised in Fig. \ref{yanhuatu}.


Although above transform-based TNN methods have shown effectiveness for LRTC, they still have much room for improvement.
Specifically, they pursue the low-rankness of frontal slices of the transformed tensor under a linear transform, which limits their recovered performance especially when sampling rate (SR) is low.

To obtain a better low-rank approximation of frontal slices of the transformed tensor,
we propose a nonlinear transform, i.e., $\psi:\mathcal{X}\in\mathbb{R}^{n_1\times{n_2}\times{n_3}}\mapsto\psi(\mathcal{X})\in\mathbb{R}^{n_1\times{n_2}\times{r}}$. More specifically,  $\psi(\mathcal{X})=\phi(\mathcal{X}\times_3\mathbf{T})$, where $\times_3$ denotes the mode-3 product of the tensor $\mathcal{X}$ and the matrix $\mathbf{T}$ (see Def. \ref{mode n product}), $\mathbf{T}\in\mathbb{R}^{r\times{n_3}}$ is a learned linear semi-orthogonal transform along the third mode, and $\phi:\mathbf{Z}_{i}\mapsto\phi(\mathbf{Z}_{i})$ is the element-wise nonlinear transform on frontal slices $\mathbf{Z}_i(i=1,\cdots,d)$ of the transformed tensor $\mathcal{Z}$ of $\mathcal{X}$ under the linear semi-orthogonal transform $\mathbf{T}$.
 In the proposed composite transform $\psi$, the learned linear semi-orthogonal transform $\mathbf{T}$ and the element-wise nonlinear transform $\phi$ are indispensable and complementary to fully exploit the underlying low-rankness (we have discussed the indispensability of $\mathbf{T}$ and $\phi$ in the Section \ref{disTandpsipart}).

Based on the proposed nonlinear transform, we propose a nonlinear transform-based tensor nuclear norm (NTTNN) to exploit the low-rankness of underlying tensor under the proposed nonlinear transform $\psi$ and establish the corresponding LRTC model. Fig. \ref{dongjitu} is the pipeline of the linear transform-based TNN method \cite{zhangtnn} and the proposed NTTNN method on MSI $\mathit{Toy}$ for SR 10\%.
From Fig. \ref{dongjitu}(b) and (d), we can observe that NTTNN with different nonlinear transforms $\phi$ (Tanh and Sigmoid) need less singular values than TNN that is based on linear transform when achieves the same accumulation energy ratio (AccEgy). In other words, the energy of singular values of nonlinear transformed tensor is more concentrated and NTTNN could make the frontal slices of underlying tensor more low-rank. As such, NTTNN can obtain the better low-rank approximation of the underlying tensor under the nonlinear transform $\psi$, which could recover more textures (see zoom-in regions and residual parts between ground truth and recovered tensors) and obtain a better recovered result than TNN, which is demonstrated in Fig. \ref{dongjitu}(c) and (e).
Moreover, the proposed NTTNN is a unified transform-based TNN family, which shows in Fig. \ref{yanhuatu}.

The main contributions of this paper is threefold:
\begin{itemize}
  \item We propose a nonlinear transformed TNN termed as NTTNN, which could enhance the low-rank approximation of the underlying tensor and can be regarded as a unified transform-based TNN family including many classic transform-based TNN methods.
  \item Based on NTTNN, we propose the LRTC model and develop an efficient multi-block proximal alternating minimization (PAM) algorithm with theoretical guarantee to solve the resulting model.
  \item Extensive experiments on HSIs, MSIs, and videos demonstrate that NTTNN outperforms linear transform-based state-of-the-art LRTC methods quantitatively and visually.
\end{itemize}

The rest part of the paper is arranged as follows. In Section \ref{section2}, we briefly introduce some essential notations and basic definitions used in this paper.
In Section \ref{section3}, we propose the NTTNN model for LRTC and establish the corresponding algorithm with theoretical convergence guarantee.
Section \ref{section4} evaluates the performance of the proposed model. Section \ref{section5} gives some discussions. Finally, Section \ref{section6} concludes this paper.
\section{Notations and Preliminaries}\label{section2}
In this part, we introduce some basic notations and definitions for developing the proposed nonlinear transform $\psi$ and NTTNN.

The basic notations used in this paper are presented in Table \ref{notations}.

\begin{table}[htp!]
\center
\caption{Basic notations.}\label{notations}\vspace{0.3cm}
\renewcommand\arraystretch{1.5}
\setlength{\tabcolsep}{1mm}
\begin{tabular}{c|c}
\Xhline{1.5pt}
Notations                        &       Explanations \\
\Xhline{1.5pt}
$\mathit{x}$, $\mathbf{x}$, $\mathbf{X}$, $\mathcal{X}$   & Scalar, vector, matrix, tensor. \\
$x_{ijk}$                        &       The $(i,j,k)$-th element of tensor $\mathcal{X}$.\\

$\mathcal{X}(:,:,i)$ or $\mathbf{X}_i$              &      The $\mathit{i}$-th frontal slice of tensor $\mathcal{X}$. \\
$\operatorname{Tr}(\mathbf{X})$   &      The trace of $\mathbf{X} \in \mathbb{R}^{n\times{n}}$ with $\operatorname{Tr}(\mathbf{X})=\sum_{i=1}^nx_{ii}$.\\
$\left \langle \mathcal{X}, \mathcal{Y} \right\rangle$     &    \!  The inner product of  \!$\mathcal{X}$ and  \!$\mathcal{Y}$ \! with $\left \langle  \mathcal{X},  \!\mathcal{Y} \right\rangle \!= \!\sum_{ijk}x_{ijk}y_{ijk}. $\\
$\|\mathcal{X}\|_F$               &      The Frobenius norm of $\mathcal{X}$ with $\|\mathcal{X}\|_F=\sqrt{\left \langle \mathcal{X}, \mathcal{X} \right\rangle}$.\\
$\|\mathbf{X}\|_*$                &      The nuclear norm of $\mathbf{X}$ with $\|\mathbf{X}\|_*=\operatorname{Tr}(\sqrt{\mathbf{X}^\top\mathbf{X}})$.\\
\Xhline{1.5pt}
\end{tabular}
\end{table}

Based on above basic notations, we give following definitions:

\begin{definition}(mode-$\mathit{k}$ product \cite{modenproduct})\label{mode n product} The mode-$k$ product of a tensor $\mathcal{Z}$ $\in$ $\mathbb{R}^{n_1 \times n_2 \times \cdots\times n_i}$ and a matrix $\mathbf{D}$ $\in$ $\mathbb{R}^{h_i\times{n_i}}$ is defined as
$$
\mathcal{X}=\mathcal{Z}\times_{k}\mathbf{D}=\mathtt{fold}_k(\mathbf{D}\mathbf{Z}_{(k)}),
$$
where $\mathbf{Z}_{(k)}$ is the mode-$k$ matricization of $\mathcal{Z}$ and $\mathtt{fold}_k(\cdot)$ is the corresponding inverse operator of matricization that rearranges elements of a matrix into a tensor.
\end{definition}
\begin{definition}\label{tnn}
(tensor nuclear norm (TNN) \cite{zhangtnn}) Let $\mathcal{X}$ $\in$ $\mathbb{R}^{n_1\times{n_2}\times{n_3}}$, the tensor nuclear norm of $\mathcal{X}$ is
$$
\|\mathcal{X}\|_{\rm{TNN}}=\sum_{i=1}^{n_3}\|\mathbf{{Z}}_{i}\|_\ast,
$$
where $\mathbf{Z}_i$ is the $i$-th frontal slice of the transformed tensor
$\mathcal{Z}=\mathcal{X}\times_3\mathbf{F}_{n_3}$.
\end{definition}
\section{Proposed Model and Solving Algorithm}\label{section3}
\subsection{Proposed Model}
The existing transform-based TNN methods employ linear transform to exploit the low-rankness of
frontal slices of the underlying tensor. However, the linear transform does not always make frontal slices of the underlying tensor obviously low-rank, which is shown in Fig. \ref{dongjitu}(d). To tackle this problem, we propose a nonlinear transform that is defined as
$$\psi(\mathcal{X})=(\phi \circ \mathbf{T})(\mathcal{X})=\phi(\mathcal{X}\times_3\mathbf{T}),$$
where, $\mathbf{T}\in\mathbb{R}^{r\times{n_3}}$ denotes a learned linear semi-orthogonal transform and satisfies $\mathbf{T}\mathbf{T}^\top=\mathbf{I}_{r\times{r}}$, and $\phi:\mathbf{Z}_{i}\mapsto\phi(\mathbf{Z}_{i})$ is the element-wise nonlinear transform on frontal slices $\mathbf{Z}_i(i=1,\cdots,d)$ of the transformed tensor $\mathcal{Z}$ of $\mathcal{X}$ under the linear semi-orthogonal transform $\mathbf{T}$.

\begin{remark}
 The introduced nonlinear composed transform, which consists of a nonlinear transform and a learned linear semi-orthogonal transform, can be interpreted as a single layer semi-orthogonal neural network \cite{orthogonal1,orthogonal2}. More specifically, the nonlinear transform is represented by the nonlinear activation function and  the linear semi-orthogonal transform corresponds to the semi-orthogonal fully connected layer, where the semi-orthogonality finds to be a favorable property for training deep convolutional neural network \cite{deepcnn}.
 \end{remark}

Based on the proposed nonlinear transform, we define the nonlinear transform-based tensor nuclear norm (NTTNN) to exploit the low-rankness of a tensor as follows:
$$\|\mathcal{X}\|_{\rm{NTTNN}}=\left\|\psi(\mathcal{X})\right\|_*
=\sum_{i=1}^{r}\left\|\phi(\mathbf{Z}_i)\right\|_*,$$
where $\mathcal{Z}=\mathcal{X}\times_3 \mathbf{T}$, $\|\cdot\|_*$ denotes the matrix nuclear norm, and $\mathbf{Z}_i$ denotes the $i$-th frontal slice of $\mathcal{Z}$.
Here, the nonlinear transform $\psi(\cdot)$ could make frontal slices of underlying tensor obviously low-rank, which is shown in Fig. \ref{dongjitu}(d) that the singular values of $\phi(\mathcal{Z})$ is more concentrated. As a result, the proposed NTTNN could better exploit the low-rankness of latent tensor than TNN-based methods under linear transforms along the third mode.

Based on NTTNN, we propose the following LRTC model:
\begin{equation}
\begin{aligned}
&\min_{\mathcal{X}, \mathcal{Z}, \mathbf{T}}  \sum_{i=1}^{r}\left\|\phi(\mathbf{Z}_i)\right\|_*\\
& ~~\text{s.t}.~~~ \mathcal{X}_{\Omega} = \mathcal{O}_{\Omega}, \mathcal{X} = \mathcal{Z}\times_3 \mathbf{T}^\top, \mathbf{T}\mathbf{T}^\top=\mathbf{I}_{r\times{r}},
\end{aligned}\label{model1}
\end{equation}
where, $\mathcal{Z}\in\mathbb{R}^{n_1\times n_2\times r}$, $\mathbf{T}\in\mathbb{R}^{r\times {n_3}}$ is the semi-orthogonal transform, and $\mathbf{I}\in\mathbb{R}^{r\times r}$ is an identity matrix.

\begin{remark}
NTTNN is a unified transform-based TNN family, which is shown in Fig. \ref{yanhuatu}. More specifically, when the nonlinear function $\phi(\cdot)$ in the proposed transform $\psi(\cdot)$ is defined as
$\phi(x)=x$, NTTNN is degraded to previous transform-based TNN methods according to the different $\mathbf{T}$:
if $\mathbf{T}$ is the fixed DFT, it is equivalent to the typical TNN method \cite{zhangtnn};
if $\mathbf{T}$ is the fixed DCT or the fixed framelet transform, it is equivalent to the DCT-based TNN methods \cite{DCT_Lu,DCT_based} or the framelet-based TNN method \cite{framelet_jiang}, respectively;
if $\mathbf{T}$ is the learned orthogonal transform or the learned semi-orthogonal transform, it is equivalent to the transform-based TNN method \cite{unitary_transform} or
the tensor Q-rank method \cite{tensor_Q_rank}, respectively.
\end{remark}

\subsection{Solving Algorithm}
By introducing indicator functions
$$
\Phi(\mathcal{X})=
\begin{cases}
   0, & \mathcal{X}_{\Omega}=\mathcal{O}_{\Omega}, \\
   +\infty, & \text{otherwise},
   \end{cases}
\begin{aligned}
\Psi(\textbf{T})=
\begin{cases}
   0, & \textbf{T}\textbf{T}^\top=\textbf{I}_{r\times{r}},  \\
   +\infty, & \text{otherwise},
   \end{cases}
\end{aligned}
$$
the problem (\ref{model1}) can be equivalently rewritten as follows:
\begin{equation}
\begin{aligned}
&\min_{\mathcal{X},\mathcal{Z},\textbf{T}}~~\sum_{i=1}^{r}\left\|\phi(\mathbf{Z}_i)\right\|_*
+\Phi(\mathcal{X})+\Psi(\textbf{T})\\
&~~\text{s.t.}~~~~~\mathcal{X} = \mathcal{Z} \times_3{\textbf{T}^\top}.
\label{mymodels2}
\end{aligned}
\end{equation}
To tackle this optimization problem, we introduce the auxiliary variable $\mathcal{Y}=\phi(\mathcal{Z})$ and lean upon the half quadratic splitting \cite{HQS} tips to transform the
constrained problem (\ref{mymodels2}) into the following unconstrained problem:
\begin{equation}
\min_{\mathcal{X},\mathcal{Z},\mathcal{Y},\textbf{T}}~~
\sum_{i=1}^{r}\left\|\mathbf{Y}_i\right\|_*
+\frac{\alpha}{2}\|\mathcal{X}-\mathcal{Z}\times_{3}\textbf{T}^\top\|_F^2
+\frac{\beta}{2}\|\mathcal{Y}-\phi(\mathcal{Z})\|_F^2
+\Phi(\mathcal{X})+\Psi(\textbf{T}),
\label{mymodels3}
\end{equation}
where $\alpha,\beta>0$ are two penalty parameters.

We denote the objective function of problem (\ref{mymodels3}) as $f(\mathcal{X},\mathcal{Y},\mathcal{Z},\textbf{T})$.
Under the proximal alternating minimization (PAM) algorithm \cite{PAM} framework, we can alternatively update each variable:
\begin{equation}
\left\{\begin{aligned}
\mathcal{X}^{k+1}\in\ &\argmin_{\mathcal{X}}\{f(\mathcal{X},\mathcal{Y}^k,\mathcal{Z}^{k},\textbf{T}^{k})+\frac{\rho_1}{2}\|\mathcal{X}-\mathcal{X}^{k}\|_F^2\},\\
\mathcal{Y}^{k+1}\in\ &\argmin_{\mathcal{Y}}\{f(\mathcal{X}^{k+1},\mathcal{Y},\mathcal{Z}^k,\textbf{T}^{k})+\frac{\rho_2}{2}\|\mathcal{Y}-\mathcal{Y}^{k}\|_F^2\},\\
\mathcal{Z}^{k+1}\in\ &\argmin_{\mathcal{Z}}\{f(\mathcal{X}^{k+1},\mathcal{Y}^{k+1},\mathcal{Z},\textbf{T}^k)+\frac{\rho_3}{2}\|\mathcal{Z}-\mathcal{Z}^{k}\|_F^2\},\\
\textbf{T}^{k+1}\in\ &\argmin_{\textbf{T}}\{f(\mathcal{X}^{k+1},\mathcal{Y}^{k+1},\mathcal{Z}^{k+1},\textbf{T})+\frac{\rho_4}{2}\|\textbf{T}-\textbf{T}^k\|_F^2\},\\
\end{aligned}
\right.\label{sequence}
\end{equation}
where $\rho_1$, $\rho_2$, $\rho_3$, and $\rho_4$ are four positive constants, and $k$ denotes the iteration number. Next, we give details for updating the subproblems about $\mathcal{X}$, $\mathcal{Y}$, $\mathcal{Z}$, and $\textbf{T}$.
\par $\bullet$ \textbf{Updating} $\mathcal{X}$ \textbf{subproblem}
\par The $\mathcal{X}$ subproblem is
\begin{equation}
\begin{aligned}\label{xsubproblem}
\argmin_{\mathcal{X}}\frac{\alpha}{2}\|\mathcal{X}-\mathcal{Z}^k\times_{3}{\textbf{T}^k}^\top\|_F^2
+\frac{\rho_1}{2}\|\mathcal{X}-\mathcal{X}^k\|_F^2+\Phi(\mathcal{X}).
\end{aligned}
\end{equation}
The closed-form solution of the problem (\ref{xsubproblem}) is
\begin{equation}\label{xsolve}
\mathcal{X}^{k+1}=\left(\frac{\alpha\mathcal{Z}^k\times_{3}{\textbf{T}^k}^\top+\rho_1\mathcal{X}^k}{\alpha+\rho_1}\right)_{\Omega^c}+\mathcal{O}_\Omega,
\end{equation}
where $\Omega^c$ denotes the complement set of $\Omega$.
\par $\bullet$ \textbf{Updating} $\mathcal{Y}$ \textbf{subproblem}
\par The $\mathcal{Y}$ subproblem is
$$
\begin{aligned}\label{ysubproblem}
&\argmin_{\mathcal{Y}}\sum_{i=1}^{r}\left\|\mathbf{Y}_i\right\|_*
+\frac{\beta}{2}\|\mathcal{Y}-\phi(\mathcal{Z}^k)\|_F^2+\frac{\rho_2}{2}\|\mathcal{Y}-\mathcal{Y}^k\|_F^2.
\end{aligned}
$$
The $\mathcal{Y}$ subproblem can be decomposed into the following $r$ subproblems:
\begin{equation}
\argmin_{\textbf{Y}_{i}}\|\textbf{Y}_{i}\|_*+\frac{\beta+\rho_2}{2}\|\textbf{Y}_{i}-\mathbf{H}_i\|_F^2,\label{rY}
\end{equation}
where $\mathbf{H}^k_i=\frac{\beta \phi(\textbf{Z}^k_i) + \rho_2 \textbf{Y}^{k}_i}{\beta+\rho_2}$.
By employing singular value thresholding (SVT) operator \cite{SVT}, the closed-form solution of each subproblem (\ref{rY}) is
\begin{equation}\label{ysolve}
\begin{aligned}
\textbf{Y}^{k+1}_i = \mathcal{T}_{\frac{1}{\beta+\rho_2}}(\mathbf{H}^k_i) =\hat{\textbf{U}}\mathcal{T}_{\frac{1}{\beta+\rho_2}}(\hat{\textbf{D}})\hat{\textbf{V}}^\top,
\end{aligned}
\end{equation}
where $(\hat{\textbf{U}},\hat{\textbf{D}},\hat{\textbf{V}})$ are derived from SVD of $\mathbf{H}^k_i$ and $\mathcal{T}_{\frac{1}{\beta+\rho}}(\hat{\textbf{D}})=\mathtt{diag}(\operatorname{max}(\sigma_j-\frac{1}{\beta+\rho_2},0))$.
\par $\bullet$ \textbf{Updating} $\mathcal{Z}$ \textbf{subproblem}
\par The $\mathcal{Z}$ subproblem is
$$
\begin{aligned}\label{zsubproblem}
&\argmin_{\mathcal{Z}}~\frac{\alpha}{2}\|\mathcal{X}^{k+1}-\mathcal{Z}\times_{3}{\textbf{T}^k}^\top\|_F^2+\frac{\beta}{2}\|\mathcal{Y}^{k+1}-\phi(\mathcal{Z})\|_F^2
+\frac{\rho_3}{2}\|\mathcal{Z}-\mathcal{Z}^k\|_F^2.
\end{aligned}
$$
The $\mathcal{Z}$ subproblem can be equivalently formulated as follows:
\begin{equation}
\begin{aligned}
&\argmin_{\textbf{Z}_{(3)}}~\frac{\alpha}{2}\|\textbf{Z}_{(3)}-{\textbf{T}^k}\textbf{X}^{k+1}_{(3)}\|_F^2+\frac{\beta}{2}\|\phi(\textbf{Z}_{(3)})-\textbf{Y}^{k+1}_{(3)}\|_F^2
+\frac{\rho_3}{2}\|\textbf{Z}_{(3)}-\textbf{Z}^k_{(3)}\|_F^2\\
&=\argmin_{\textbf{Z}_{(3)}}~\frac{\alpha+\rho_3}{2}\|\textbf{Z}_{(3)}-\mathbf{G}^k\|_F^2+\frac{\beta}{2}\|\phi(\textbf{Z}_{(3)})-\textbf{Y}^{k+1}_{(3)}\|_F^2,
\end{aligned}\label{zsubproblem2}
\end{equation}
where $\textbf{G}^k=\frac{\alpha{\textbf{T}^k}\textbf{X}^{k+1}_{(3)}+\rho_3\textbf{Z}^k_{(3)}}{\alpha+\rho_3}$. We denote the $(i,j)$-th element of $\mathbf{Z}^k_{(3)}$, $\mathbf{G}^k$, and $\mathbf{Y}^k_{(3)}$ as $z_{ij}^k=\mathbf{Z}^k_{(3)}(i,j)$, $g^k_{ij}=\mathbf{G}^k(i,j)$, and $y^k_{ij}=\mathbf{Y}^k_{(3)}(i,j)$, respectively. Then, the problem (\ref{zsubproblem2}) can be decomposed into $n_1n_2r$ one-dimensional nonlinear equations as follows:
\begin{equation}\label{zsolve}
\argmin_{z_{ij}}~\frac{\alpha+\rho_3}{2}(z_{ij}-g^k_{ij})^2+\frac{\beta}{2}(\phi(z_{ij})-y^{k+1}_{ij})^2,
\end{equation}
which can be solved by the Newton method.
\par $\bullet$ \textbf{Updating} $\textbf{T}$ \textbf{subproblem}
\par The $\textbf{T}$ subproblem is
\begin{equation}
\begin{aligned}\label{tsubproblem1}
\argmin_{\textbf{T}}~\frac{\alpha}{2}\|\mathcal{X}^{k+1}-\mathcal{Z}^{k+1}\times_{3}
\textbf{T}^\top\|_F^2+\frac{\rho_4}{2}
\|\textbf{T}-\textbf{T}^k\|_F^2+\Psi(\textbf{T}).
\end{aligned}
\end{equation}
Note that the problem (\ref{tsubproblem1}) can be equivalently transformed the following problem:
\begin{equation}
\begin{aligned}\label{Tmax}
&\argmin_{\textbf{T}}~\frac{\alpha}{2}\|\mathcal{X}^{k+1}-\mathcal{Z}^{k+1}\times_{3}
\textbf{T}^\top\|_F^2+\frac{\rho_4}{2}
\|\textbf{T}-\textbf{T}^k\|_F^2+\Psi(\textbf{T})\\
=&\argmin_{\textbf{T}}~\frac{\alpha}{2}\|\textbf{X}_{(3)}^{k+1}-\textbf{T}^\top\textbf{Z}_{(3)}^{k+1}\|_F^2+\frac{\rho_4}{2}\|\textbf{T}-\textbf{T}^k\|_F^2+\Psi(\textbf{T})\\
=&\argmin_{\textbf{T}}~\frac{\alpha}{2}\operatorname{Tr}[(\textbf{X}_{(3)}^{k+1}-\textbf{T}^\top\textbf{Z}_{(3)}^{k+1})^\top(\textbf{X}_{(3)}^{k+1}-\textbf{T}^\top\textbf{Z}_{(3)}^{k+1})]\\
&~~~~~~~~~~~+\frac{\rho_4}{2}\operatorname{Tr}[(\textbf{T}-\textbf{T}^k)^\top(\textbf{T}-\textbf{T}^k)]+\Psi(\textbf{T})\\
=&\argmax_{\textbf{T}}~\operatorname{Tr}[(\alpha\textbf{X}_{(3)}^{k+1}(\textbf{Z}_{(3)}^{k+1})^\top+\rho_4{\textbf{T}^k}^\top)\textbf{T}]-\Psi(\textbf{T}),
\end{aligned}
\end{equation}
where $\operatorname{Tr}(\mathbf{X})$ denotes the trace of matrix $\mathbf{X}$.
Supposing the SVD of $\alpha\textbf{X}_{(3)}^{k+1}(\textbf{Z}_{(3)}^{k+1})^\top+\rho_4{\textbf{T}^k}^\top$ is $\tilde{\textbf{U}}\tilde{\textbf{D}}\tilde{\textbf{V}}^\top$, we have
$$\operatorname{Tr}(\tilde{\textbf{U}}\tilde{\textbf{D}}\tilde{\textbf{V}}^\top\textbf{T})
=\operatorname{Tr}(\tilde{\textbf{D}}\tilde{\textbf{U}}\tilde{\textbf{V}}^\top\textbf{T}).$$
Since $\tilde{\textbf{D}}$ is the diagonal matrix, the maximization problem in (\ref{Tmax}) can be maximized when the diagonal elements of $\tilde{\textbf{U}}\tilde{\textbf{V}}^\top\textbf{T}$ is positive and maximum. By the Cauchy-Schwartz inequality, this is achieved when $\textbf{T}=(\tilde{\textbf{U}}\tilde{\textbf{V}}^\top)^\top$ in which case the diagonal elements are all 1. Hence the closed-form solution of (\ref{tsubproblem1}) is
\begin{equation}\label{tsolve}
\textbf{T}^{k+1}=\tilde{\textbf{V}}\tilde{\textbf{U}}^\top,
\end{equation}
where $\tilde{\textbf{U}}$ and $\tilde{\textbf{V}}$ are the orthogonal matrices obtained by the following SVD:
$$\alpha\textbf{X}_{(3)}^{k+1}(\textbf{Z}_{(3)}^{k+1})^\top+\rho_4{\textbf{T}^k}^\top=\tilde{\textbf{U}}\tilde{\textbf{D}}\tilde{\textbf{V}}^\top.$$

We summary the solving algorithm for NTTNN in Algorithm \ref{MBPAM}.
\begin{algorithm}[t]
\caption{The PAM-based solver for the proposed NTTNN model.}
\begin{algorithmic}
\State \textbf{Input:} The observed $\mathcal{O}\in\mathbb{R}^{n_1\times{n_2}\times{n_3}}$ , index set $\Omega$, the row number $r$ of the transform $\mathbf{T}$, proximal parameters $\alpha$, $\beta$, and $\rho_i(i=1,\cdots,4)$.
\State \textbf{Output:} The recovered third-order tensor $\mathcal{X}$.
\State \textbf{Initialization:}  $\mathcal{X}^{0}$, $\mathcal{Y}^0$, $\mathcal{Z}^0$, \textbf{T}$^0$;\vspace{2mm}
\State \textbf{While} $\frac{\left\|\mathcal{X}^{k+1}-\mathcal{X}^{k}\right\|_{F}}{\left\|\mathcal{X}^{k}\right\|_{F}} \leq   10^{-4}$ \textbf{do}\vspace{1mm}
\State ~~~~~~~~~~Update $\mathcal{X}^{k+1}$ via (\ref{xsolve});
\State ~~~~~~~~~~Update $\mathcal{Y}^{k+1}$ via (\ref{ysolve});
\State ~~~~~~~~~~Update $\mathcal{Z}^{k+1}$ via (\ref{zsolve});
\State ~~~~~~~~~~Update \textbf{T}$^{k+1}$ via (\ref{tsolve});
\State \textbf{end while}
\end{algorithmic}\label{MBPAM}
\end{algorithm}

\subsection{Convergence analysis}
Under the PAM algorithm framework, we establish the global convergence guarantee of Algorithm \ref{MBPAM} to solve (\ref{mymodels3}).
First of all, we denote following functions:
$$f(\mathcal{X},\mathcal{Y},\mathcal{Z},\textbf{T}) = \sum_{i=1}^{r}\left\|\mathbf{Y}_i\right\|_* + \frac{\alpha}{2}\|\mathcal{X}-\mathcal{Z}\times_3\textbf{T}^\top\|_F^2
+\frac{\beta}{2}\|\mathcal{Y}-\phi(\mathcal{Z})\|_F^2+\Phi(\mathcal{X})+\Psi(\textbf{T}),$$
and
$$
f_1(\mathcal{X},\mathcal{Y},\mathcal{Z},\textbf{T})=\frac{\alpha}{2}\|\mathcal{X}-\mathcal{Z}\times_3\textbf{T}^\top\|_F^2
+\frac{\beta}{2}\|\mathcal{Y}-\phi(\mathcal{Z})\|_F^2.
$$
\indent First, we introduce the necessary ingredients used for the convergence analysis.

\begin{definition}\label{KL}
(Kurdyka-\L ojasiewicz property\cite{PAM_convergence_codition}). The function $\psi(x):\mathbb{R}^n\rightarrow\mathbb{R}\cup+\infty$ is said to have the Kurdyka-$\L$ojasiewicz (K-\L) property at $x^*\in\operatorname{dom}(\partial \psi(x))$
if there exist $\eta \in (0,+\infty]$, a neighborhood $U$ of $x^*$ and a continuous concave function $\psi(x):[0,\eta)\rightarrow \mathbb{R}_+$ satisfy:
\begin{itemize}
  \item $\psi(0)=0$,
  \item $\psi(x)$ is $C^1$ on $(0,\eta]$,
  \item for any $x \in (0,\eta),\psi'(x)>0$,
  \item for any $x$ in $U\cap[\psi(x^*)<\psi(x)<\psi(x^*)+\eta]$
\end{itemize}
The proper lower semi-continuous functions are called K-\L~functions, if they satisfy K-\L~property at each point of $\operatorname{dom}(\partial \psi(x))$.
\end{definition}

\begin{definition}\label{semi-algebraic}
(Semi-algebraic set and semi-algebraic function\cite{PAM_convergence_codition})
If there exists a series of real polynomial functions $m_{ij}$ and $n_{ij}$ satisfy
$S=\cap_j\cup_i\{x\in\mathbb{R}^n:m_{ij}(x)=0,n_{ij}(x)<0\},$ then the subset $S \in \mathbb{R}$ is a semi-algebraic set.
If the graph $\{(x,y)\in \mathbb{R}^n\times\mathbb{R}, f(x) = y\}$ of the function $f$ is a
semi-algebraic set, then $f$ is a semi-algebraic function.
\end{definition}
\begin{remark}\label{remark2}
A semi-algebraic real valued function $f$ satisfies K-\L~property at each $x \in \operatorname{dom}(f)$, i.e., $f$ is a K-\L~function.
\end{remark}

 \begin{lemma}
 (Sufficient decrease lemma). For any $\rho_i>0$ $(i=1,2,3,4)$, the sequence $\{\mathcal{X}^k,\mathcal{Y}^k,\mathcal{Z}^k,\mathbf{T}^k\}$ that is generated by (\ref{sequence}) satisfies the following formulae:
 \begin{equation}
\left\{\begin{aligned}
&f(\mathcal{X}^{k+1},\mathcal{Y}^k,\mathcal{Z}^k,\mathbf{T}^k)+\frac{\rho_1}{2}\|\mathcal{X}^{k+1}-\mathcal{X}^k\|_F^2 \leq f(\mathcal{X}^k,\mathcal{Y}^k,\mathcal{Z}^k,\mathbf{T}^k),\\
&f(\mathcal{X}^{k+1},\mathcal{Y}^{k+1},\mathcal{Z}^k,\mathbf{T}^k)+\frac{\rho_2}{2}\|\mathcal{Y}^{k+1}-\mathcal{Y}^k\|_F^2 \leq f(\mathcal{X}^{k+1},\mathcal{Y}^k,\mathcal{Z}^k,\mathbf{T}^k),\\
&f(\mathcal{X}^{k+1},\mathcal{Y}^{k+1},\mathcal{Z}^{k+1},\mathbf{T}^k)+\frac{\rho_3}{2}\|\mathcal{Z}^{k+1}-\mathcal{Z}^k\|_F^2 \leq f(\mathcal{X}^{k+1},\mathcal{Y}^{k+1},\mathcal{Z}^k,\mathbf{T}^k),\\
&f(\mathcal{X}^{k+1},\mathcal{Y}^{k+1},\mathcal{Z}^{k+1},\mathbf{T}^{k+1})+\frac{\rho_4}{2}\|\mathbf{T}^{k+1}-\mathbf{T}^k\|_F^2 \leq f(\mathcal{X}^{k+1},\mathcal{Y}^{k+1},\mathcal{Z}^{k+1},\mathbf{T}^k).
\end{aligned}\right.
\end{equation}\label{descentlemma}
\end{lemma}
\begin{proof} Let $\mathcal{X}^{k+1}$, $\mathcal{Y}^{k+1}$, $\mathcal{Z}^{k+1}$, and $\textbf{T}^{k+1}$ be optimal solutions of the corresponding subproblem in (\ref{sequence}), then we have
$$
\left\{\begin{aligned}
&f(\mathcal{X}^{k+1},\mathcal{Y}^k,\mathcal{Z}^k,\mathbf{T}^k)+\frac{\rho_1}{2}\|\mathcal{X}^{k+1}-\mathcal{X}^k\|_F^2 \leq f(\mathcal{X}^k,\mathcal{Y}^k,\mathcal{Z}^k,\mathbf{T}^k),\\
&f(\mathcal{X}^{k+1},\mathcal{Y}^{k+1},\mathcal{Z}^k,\mathbf{T}^k)+\frac{\rho_2}{2}\|\mathcal{Y}^{k+1}-\mathcal{Y}^k\|_F^2 \leq f(\mathcal{X}^{k+1},\mathcal{Y}^k,\mathcal{Z}^k,\mathbf{T}^k),\\
&f(\mathcal{X}^{k+1},\mathcal{Y}^{k+1},\mathcal{Z}^{k+1},\mathbf{T}^k)+\frac{\rho_3}{2}\|\mathcal{Z}^{k+1}-\mathcal{Z}^k\|_F^2 \leq f(\mathcal{X}^{k+1},\mathcal{Y}^{k+1},\mathcal{Z}^k,\mathbf{T}^k),\\
&f(\mathcal{X}^{k+1},\mathcal{Y}^{k+1},\mathcal{Z}^{k+1},\mathbf{T}^{k+1})+\frac{\rho_4}{2}\|\mathbf{T}^{k+1}-\mathbf{T}^k\|_F^2 \leq f(\mathcal{X}^{k+1},\mathcal{Y}^{k+1},\mathcal{Z}^{k+1},\mathbf{T}^k).
\end{aligned}\right.
$$
The proof of the sufficient decrease lemma is completed.
\end{proof}
\begin{lemma}(Relative error lemma). Assuming that $\phi(\cdot)$ is a real analytic function, and continuous on its domain with Lipschitz
continuous on any bounded set. Then, the sequence $\{\mathcal{X}^k,\mathcal{Y}^k,\mathcal{Z}^k,\mathbf{T}^k\}$ obtained by (\ref{sequence}) is bounded, and for any $\rho_i>0$ $(i=1,2,3,4)$, there exist $V_i^{t+1}$ such that $\{\mathcal{X}^k,\mathcal{Y}^k,\mathcal{Z}^k,\mathbf{T}^k\}$ satisfies the following formulae:
\begin{equation}
\left\{\begin{aligned}
&\|V_{1}^{k+1}+\nabla_\mathcal{X}f_1(\mathcal{X}^{k+1},\mathcal{Y}^k,\mathcal{Z}^k,\mathbf{T}^k)\|_F\leq \rho_1\|\mathcal{X}^{k+1}-\mathcal{X}^k\|_F,\\
&\|V_{2}^{k+1}+\nabla_\mathcal{Y}f_1(\mathcal{X}^{k+1},\mathcal{Y}^{k+1},\mathcal{Z}^k,\mathbf{T}^k)\|_F\leq \rho_2\|\mathcal{Y}^{k+1}-\mathcal{Y}^k\|_F,\\
&\|V_{3}^{k+1}+\nabla_\mathcal{Z}f_1(\mathcal{X}^{k+1},\mathcal{Y}^{k+1},\mathcal{Z}^{k+1},\mathbf{T}^k)\|_F\leq \rho_3\|\mathcal{Z}^{k+1}-\mathcal{Z}^k\|_F,\\
&\|V_{4}^{k+1}+\nabla_\mathbf{T}f_1(\mathcal{X}^{k+1},\mathcal{Y}^{k+1},\mathcal{Z}^{k+1},\mathbf{T}^{k+1})\|_F\leq \rho_4\|\mathbf{T}^{k+1}-\mathbf{T}^k\|_F.
\end{aligned}\right.\label{Ni}
\end{equation}\label{H2}
\end{lemma}
\begin{proof}
Firstly, we prove $\{\mathcal{X}^k,\mathcal{Y}^k,\mathcal{Z}^k,\mathbf{T}^k\}$ obtained by (\ref{sequence}) is bounded.
Since
$$
\begin{aligned}
&\lim_{\|\mathcal{X}\|_F\rightarrow +\infty}\frac{\alpha}{2}\|\mathcal{X}-\mathcal{Z}\times_3\mathbf{T}^\top\|_F=+\infty,~~~
\lim_{\|\mathcal{Y}\|_F\rightarrow +\infty}\sum_{i=1}^{r}\left\|\mathbf{Y}_i\right\|_*=+\infty,\\
&\lim_{\|\mathcal{Z}\|_F\rightarrow +\infty}\frac{\alpha}{2}\|\mathcal{X}-\mathcal{Z}\times_3\mathbf{T}^\top\|_F=+\infty,~~~
\lim_{\|\mathbf{T}\|_F\rightarrow +\infty}\Psi(\mathbf{T})=+\infty,
\end{aligned}
$$
we can respectively obtain
$$
\begin{aligned}
&\lim_{\|\mathcal{X}\|_F\rightarrow +\infty}f(\mathcal{X},\mathcal{Y},\mathcal{Z},\mathbf{T})=+\infty,
\lim_{\|\mathcal{Y}\|_F\rightarrow +\infty}f(\mathcal{X},\mathcal{Y},\mathcal{Z},\mathbf{T})=+\infty,\\
&\lim_{\|\mathcal{Z}\|_F\rightarrow +\infty}f(\mathcal{X},\mathcal{Y},\mathcal{Z},\mathbf{T})=+\infty,
\lim_{\|\mathbf{T}\|_F\rightarrow +\infty}f(\mathcal{X},\mathcal{Y},\mathcal{Z},\mathbf{T})=+\infty.
\end{aligned}
$$
Therefore, we have the conclusion that $f(\mathcal{X}^{k+1},\mathcal{Y}^{k+1},\mathcal{Z}^{k+1},\mathbf{T}^{k+1})$ would approach infinity if $\{\mathcal{X}^k,\mathcal{Y}^k,\mathcal{Z}^k,\mathbf{T}^k\}$ is unbounded, i.e., the sequence $\{\mathcal{X}^k,\mathcal{Y}^k,\mathcal{Z}^k,\mathbf{T}^k\}$ is bounded if
$f(\mathcal{X}^{k+1},\mathcal{Y}^{k+1},\mathcal{Z}^{k+1},$ $\mathbf{T}^{k+1})$ is finite. Thus, we proof $f(\mathcal{X}^{k+1},\mathcal{Y}^{k+1},\mathcal{Z}^{k+1},$ $\mathbf{T}^{k+1})$ is finite in the following. According to \textbf{Lemma} \ref{descentlemma}, we have
$$
\begin{aligned}
f(\mathcal{X}^{k+1},\mathcal{Y}^{k+1},\mathcal{Z}^{k+1},\mathbf{T}^{k+1})&\leq
f(\mathcal{X}^{k+1},\mathcal{Y}^{k+1},\mathcal{Z}^{k+1},\mathbf{T}^{k+1})
+\frac{\rho_1}{2}\|\mathcal{X}^{k+1}-\mathcal{X}^k\|_F^2
+\frac{\rho_2}{2}\|\mathcal{Y}^{k+1}-\mathcal{Y}^k\|_F^2
\\
&+\frac{\rho_3}{2}\|\mathcal{Z}^{k+1}-\mathcal{Z}^k\|_F^2
+\frac{\rho_4}{2}\|\mathbf{T}^{k+1}-\mathbf{T}^k\|_F^2\\
&\leq f(\mathcal{X}^k,\mathcal{Y}^k,\mathcal{Z}^k,\mathbf{T}^k)\\
&\leq f(\mathcal{X}^k,\mathcal{Y}^k,\mathcal{Z}^k,\mathbf{T}^k)
+\frac{\rho_1}{2}\|\mathcal{X}^{k}-\mathcal{X}^{k-1}\|_F^2
+\frac{\rho_2}{2}\|\mathcal{Y}^{k}-\mathcal{Y}^{k-1}\|_F^2\\
&+\frac{\rho_3}{2}\|\mathcal{Z}^{k}-\mathcal{Z}^{k-1}\|_F^2
+\frac{\rho_4}{2}\|\mathbf{T}^{k}-\mathbf{T}^{k-1}\|_F^2\\
&\leq \\
& \cdots \\
&\leq
f(\mathcal{X}^{0},\mathcal{Y}^{0},\mathcal{Z}^{0},\mathbf{T}^{0}),
\end{aligned}
$$
then $f(\mathcal{X}^{k+1},\mathcal{Y}^{k+1},\mathcal{Z}^{k+1},\mathbf{T}^{k+1})$ is finite.

Therefore, we can conclude that $\{\mathcal{X}^k,\mathcal{Y}^k,\mathcal{Z}^k,\mathbf{T}^k\}$ obtained by (\ref{sequence}) is bounded.

Next, let $\mathcal{X}^{k+1}$, $\mathcal{Y}^{k+1}$, $\mathcal{Z}^{k+1}$, and $\mathbf{T}^{k+1}$ be optimal solutions of each subproblem in (\ref{sequence}). For $\mathcal{X}$, $\mathcal{Y}$, and $\mathbf{T}$ subproblems, we have
$$
\left\{
\begin{aligned}
&0 \in \partial \Phi(\mathcal{X}^{k+1}) +\alpha(\mathcal{X}^{k+1}-\mathcal{Z}^k\times_3{\mathbf{T}^{k}}^\top)+\rho_1(\mathcal{X}^{k+1}-\mathcal{X}^k),
\\&0 \in \partial (\sum_{i=1}^{r}\left\|\mathbf{Y}^{k+1}_i\right\|_*) + \beta(\mathcal{Y}^{k+1}-\phi(\mathcal{Z}^k))+\rho_2(\mathcal{Y}^{k+1}-\mathcal{Y}^k),
\\&0 \in \partial \Psi(\mathbf{T}^{k+1})-\alpha(\mathbf{X}^{k+1}_{(3)}-{\mathbf{T}^{k+1}}^\top\mathbf{Z}_{(3)}^{k+1})\mathbf{Z}^{k+1}_{(3)}
+\rho_4(\mathbf{T}^{k+1}-\mathbf{T}^k).
\end{aligned}\right.
$$
Then we can define $V_1$, $V_2$, and $V_4$ as
$$
\left\{
\begin{aligned}
&V_1^{k+1}=-\alpha(\mathcal{X}^{k+1}-\mathcal{Z}^k\times_3{\mathbf{T}^{k}}^\top)-\rho_1(\mathcal{X}^{k+1}-\mathcal{X}^k),
\\&V_2^{k+1} = -\beta(\mathcal{Y}^{k+1}-\phi(\mathcal{Z}^k))-\rho_2(\mathcal{Y}^{k+1}-\mathcal{Y}^k),
\\&V_4^{k+1}= \alpha(\mathbf{X}^{k+1}_{(3)}-{\mathbf{T}^{k+1}}^\top\mathbf{Z}_{(3)}^{k+1})\mathbf{Z}^{k+1}_{(3)}
-\rho_4(\mathbf{T}^{k+1}-\mathbf{T}^k).
\end{aligned}\right.
$$
Additionally, from the subproblem (\ref{zsolve}), we have
$$
0\in(\alpha+\rho_3)(z_{ij}^{k+1}-g_{ij}^k)+\beta\partial \phi(z_{ij}^{k+1})(\phi(z_{ij}^{k+1})-y_{ij}^{k+1}),
$$
thus we can define $V_3^{k+1}=0$.
Since $\phi(\cdot)$ is a real analytic function, and continuous on its domain with Lipschitz continuous on any bounded set, $\nabla f_1$ is Lipschitz continuous on any bounded set.
Since $\{\mathcal{X}^k,\mathcal{Y}^k,\mathcal{Z}^k,\mathbf{T}^k\}$ is bounded and $\nabla f_1$ is Lipschitz continuous on any bounded set, then for any $\rho_i>0$, the following formulae holds:
$$
\left\{\begin{aligned}
&\|V_{1}^{k+1}+\nabla_\mathcal{X}f_1(\mathcal{X}^{k+1},\mathcal{Y}^k,\mathcal{Z}^k,\mathbf{T}^k)\|_F\leq \rho_1\|\mathcal{X}^{k+1}-\mathcal{X}^k\|_F,\\
&\|V_{2}^{k+1}+\nabla_\mathcal{Y}f_1(\mathcal{X}^{k+1},\mathcal{Y}^{k+1},\mathcal{Z}^k,\mathbf{T}^k)\|_F\leq \rho_2\|\mathcal{Y}^{k+1}-\mathcal{Y}^k\|_F,\\
&\|V_{3}^{k+1}+\nabla_\mathcal{Z}f_1(\mathcal{X}^{k+1},\mathcal{Y}^{k+1},\mathcal{Z}^{k+1},\mathbf{T}^k)\|_F\leq \rho_3\|\mathcal{Z}^{k+1}-\mathcal{Z}^k\|_F,\\
&\|V_{4}^{k+1}+\nabla_\mathbf{T}f_1(\mathcal{X}^{k+1},\mathcal{Y}^{k+1},\mathcal{Z}^{k+1},\mathbf{T}^{k+1})\|_F\leq \rho_4\|\mathbf{T}^{k+1}-\mathbf{T}^k\|_F.
\end{aligned}\right.
$$
Therefore, the proof of the relative error lemma is completed.
\end{proof}

\par Next, we give the theoretical convergence guarantee of Algorithm \ref{MBPAM}.
\begin{theorem}\label{theorem4}
Assuming that the $\phi(\cdot)$ is a real analytic function and continuous on its domains with Lipschitz continuous on any bounded set, the bounded sequence $\{\mathcal{X}^k,\mathcal{Y}^k,\mathcal{Z}^k,\mathbf{T}^k\}$ obtained by Algorithm \ref{MBPAM} converges to a critical point of $f$.
\end{theorem}
\begin{proof}
To prove $\{\mathcal{X}^k,\mathcal{Y}^k,\mathcal{Z}^k,\mathbf{T}^k\}$ globally
converges to a critical point of $f(\mathcal{X}^k,\mathcal{Y}^k,\mathcal{Z}^k,\mathbf{T}^k)$, we require the following three key conditions:\\
 \indent $\bullet$ $f(\mathcal{X}^k,\mathcal{Y}^k,\mathcal{Z}^k,\mathbf{T}^k)$ is a proper lower semi-continuous function.\\
 \indent $\bullet$ $f(\mathcal{X}^k,\mathcal{Y}^k,\mathcal{Z}^k,\mathbf{T}^k)$ satisfies the K-\L{} property at each $\{\mathcal{X}^k,\mathcal{Y}^k,\mathcal{Z}^k,\mathbf{T}^k\}$ $\in$ $\operatorname{dom}(f)$.\\
 \indent $\bullet$ The sequence $\{\mathcal{X}^k,\mathcal{Y}^k,\mathcal{Z}^k,\mathbf{T}^k\}_{k\in \mathbb{N}}$ satisfies the sufficient decrease and relative error conditions.
 \par Firstly, it can be verified that $\frac{\alpha}{2}\|\mathcal{X}-\mathcal{Z}\times_3\textbf{T}^\top\|_F^2$
 and $\frac{\beta}{2}\|\mathcal{Y}-\phi(\mathcal{Z})\|_F^2$ are $C^1$ functions with locally Lipschitz continuous gradient, and $\Phi(\mathcal{X})$, $\Psi(\textbf{T})$, and $\sum_{i=1}^{r}\left\|\mathbf{Y}_i\right\|_*$ are proper lower semi-continuous. Therefore, $f(\mathcal{X},\mathcal{Y},\mathcal{Z},\mathbf{T})$ is the proper and lower semi-continuous function.
 \par Secondly, we prove $f$ satisfies K-\L~property at each point by verifying that the each part of $f(\mathcal{X},\mathcal{Y},\mathcal{Z},\mathbf{T})$ is the K-\L~function, where
 $$
 f(\mathcal{X},\mathcal{Y},\mathcal{Z},\mathbf{T})=\sum_{i=1}^{r}\left\|\mathbf{Y}_i\right\|_*
 +\frac{\alpha}{2}\|\mathcal{X}-\mathcal{Z}\times_3\textbf{T}^\top\|_F^2
 +\frac{\beta}{2}\|\mathcal{Y}-\phi(\mathcal{Z})\|_F^2
 +\Phi(\mathcal{X})+\Psi(\textbf{T}).
 $$
 Then, we verify each part as follows:

 (1)The matrix nuclear norm term $\sum_{i=1}^{r}\left\|\mathbf{Y}_i\right\|_*$ is a semi-algebraic function \cite{semi-algebraic}. According to \textbf{Remark} \ref{remark2}, $\sum_{i=1}^{r}\left\|\mathbf{Y}_i\right\|_*$ is a K-\L~function.

 (2)The Frobenius norm function $\frac{\alpha}{2}\|\mathcal{X}-\mathcal{Z}\times_3\textbf{T}^\top\|_F^2$
 is semi-algebraic \cite{semi-algebraic}. According to \textbf{Remark} \ref{remark2}, $\frac{\alpha}{2}\|\mathcal{X}-\mathcal{Z}\times_3\textbf{T}^\top\|_F^2$ is a K-\L~function.

 (3)$\Psi(\mathcal{X})$ and $\Phi(\mathbf{T})$ are semi-algebraic functions, since they are indicator functions with semi-algebraic sets \cite{semi-algebraic}. According to \textbf{Remark} \ref{remark2}, $\Psi(\mathcal{X})$ and $\Phi(\mathbf{T})$ are K-\L~functions.

 (4) According to the proof of Lemma 6 in \cite{nonlinearKL}, the nonlinear function $\frac{\beta}{2}\|\mathcal{Y}-\phi(\mathcal{Z})\|_F^2$ is a K-\L~function.

 Therefore, the function $f(\mathcal{X},\mathcal{Y},\mathcal{Z},\mathbf{T})$ is a K-\L~function.

 \par Thirdly, according to \textbf{Lemma} \ref{descentlemma} and \textbf{Lemma} \ref{H2}, the sequence $\{\mathcal{X}^k,\mathbf{Y}^k,\mathcal{Z}^k,\mathbf{T}^k\}$ satisfies the sufficient decrease and relative error conditions.

In summary, combining the three key conditions, the proposed algorithm satisfies Theorem 6.2 in \cite{PAM_convergence_codition}, thus, we can conclude that the sequence $\{\mathcal{X}^k,\mathcal{Y}^k,\mathcal{Z}^k,\mathbf{T}^k\}$ generated by Algorithm \ref{MBPAM} converges to a critical point of $f$.
\end{proof}

%
%
\section{Numerical Experiments}\label{section4}
In this part, we conduct numerical experiments on HSIs, MSIs, and videos for LRTC to test the performance of the proposed model. All experimental tensor data are prescaled to [0,1].
All numerical experiments
are implemented in Windows 10 64-bit and MATLAB R2019a on a desktop computer with an Intel(R) Core(TM) i7-8700K
CPU at 3.70 GHz with 32GB memory of RAM.

We compare the proposed method with four state-of-the-art methods, including t-SVD baseline method TNN \cite{zhangtnn}, DCT-based TNN method DCT-TNN \cite{DCT_Lu}, transform-based TNN method TTNN \cite{unitary_transform}, and dictionary-based TNN method DTNN \cite{DTNN}. For the compared methods, we make efforts to achieve their best result according to the authors suggestions. For our method, we set the proximal parameters $\rho_i=0.001$ $(i=1,2,3,4)$, the penalty parameters $\alpha$ and $\beta$ are selected from \{1,10,100\}.
For easy comparison, we use the hyperbolic tangent (Tanh) function as the nonlinear function $\phi(\cdot)$, i.e.,
$$
\phi(x)=\frac{e^x-e^{-x}}{e^x+e^{-x}}.
$$
Please see the comparison of different nonlinear functions in section \ref{nonlinearcomparison}.

Since the proposed NTTNN model is highly nonlinear and nonconvex, it is significant for our algorithm to employ an efficient initialization.
To efficiently obtain $\mathcal{X}^0$, we use a simple linear interpolation strategy, which is also used in \cite{chazhi}, for TTNN, DTNN, and our method.
The initialization for transform $\mathbf{T}$ is obtained from the left-singular vectors $\mathbf{U}$ of the SVD of $\mathbf{X}_{(3)}^0$, i.e., $\mathbf{T}^0=\mathbf{U}(:,1:r)^\top$.
Then, $\mathcal{Z}^0$ can be obtained by $\mathcal{Z}^{(0)}=\mathtt{fold}_3(\mathbf{Z}^0_{(3)})=\mathtt{fold}_3({\mathbf{T}^0}\mathbf{X}^0_{(3)})$, and $\mathcal{Y}^0=\phi(\mathcal{Z}^0)$.

The quality of recovered images is measured by the peak signal-to-noise ratio (PSNR) \cite{PSNRSSIM}, the structural similarity index (SSIM) \cite{PSNRSSIM}, and
the spectral angle mapper (SAM) \cite{SAM}. The PSNR and SSIM are defined as
$$
\text{PSNR}=10\log_{10}\frac{\text{MAX}^2_{\mathbf{X},\mathbf{X}^*}}{\|\mathbf{X}-\mathbf{X}^*\|_F^2}
$$
and
$$
\text{SSIM}=\frac{(2\mu_\mathbf{X}\mu_{\mathbf{X}^*})(2\sigma_{\mathbf{XX}^*}+c_2)}
{(\mu^2_\mathbf{X}+\mu^2_{\mathbf{X}^*}+c_1)(\sigma^2_\mathbf{X}+\sigma^2_{\mathbf{X}^*}+c_2)},
$$
respectively, where, $\mathbf{X}^*$ is the true image, $\mathbf{X}$ is the recovered image, $\text{MAX}_{\mathbf{X},\mathbf{X}^*}$ is the maximum pixel value of the images $\mathbf{X}$ and $\mathbf{X}^*$, $\mu_\mathbf{X}$ and $\mu_{\mathbf{X}^*}$ are the mean values of images
$\mathbf{X}$ and $\mathbf{X}^*$, $\sigma_\mathbf{X}$ and $\sigma_{\mathbf{X}^*}$ are the standard variances of $\mathbf{X}$ and $\mathbf{X}^*$, respectively, $\sigma_{\mathbf{XX}^*}$ is the covariance of $\mathbf{X}$ and $\mathbf{X}^*$, and $c_1$ and $c_2$ are positive constants. The SAM is defined as
$$
\text{SAM}=\cos^{-1}\frac{\sum_{i=1}^{n_1n_2}x_ix_i^*}
{(\sum_{i=1}^{n_1n_2}{x_i}^2)^{\frac{1}{2}}(\sum_{i=1}^{n_1n_2}{x^*_i}^2)^{\frac{1}{2}}},
$$
where $x_i$ and $x_i^*$ are pixel of $\mathbf{X}$ and $\mathbf{X}^*$, respectively.
By calculating average PSNR, SSIM and SAM values for all bands, we obtain PSNR, SSIM, and SAM values of a higher-order tensor.
Higher PSNR/SSIM values and lower SAM values indicate better reconstructions.

For all the methods, the relative error of the tensor $\mathcal{X}$ between two successive iterations defined by
$$
\frac{\|\mathcal{X}^{k+1}-\mathcal{X}^k\|_F}{\|\mathcal{X}^k\|_F}\leq 10^{-4}
$$
as the stopping criterion.
\subsection{Experiments on HSIs Data}
In this subsection, we use a sub-image of $\mathit{Washington~DC~Mall}(\mathit{WDC~Mall})$ of size $256\times{256}\times{100}$ and a sub-image of $\mathit{Pavia~City}$ of size $200\times{200}\times{80}$ to evaluate the performance of the proposed method.
Since the high redundancy between HSI's slices, we evaluate the performance of the proposed method on HSIs for extremely low sample ratios (SRs) 1\%, 5\%, and 10\%.

 Table \ref{HSItable} lists the numerical results by different methods, where the best results for each data are highlighted in bold. It can be observed that the proposed NTTNN consistently outperforms the compared methods in terms of PSNR, SSIM, and SAM values on all cases.

Fig. \ref{HSIpic} shows the recovered results of one band and the spectrum profiles of $\mathit{WDC~Mall}$ and $\mathit{Pavia~City}$ by different methods for SR = 1\%.
From the visual comparison, our method outperforms other compared methods in preserving image structures and details, e.g., the building in the zoom-in regions of $\mathit{WDC~Mall}$.
Moreover, NTTNN gives the closest spectrum profiles than those of other compared methods, which demonstrates that the nonlinear transform plays an important role in the spectrum profile recovery.

\begin{table}[htp!]
\centering
\renewcommand\arraystretch{1.5}
\setlength{\tabcolsep}{1.4mm}
\caption{The PSNR, SSIM, and SAM values of the recovered HSIs by different methods for different SRs.}\vspace{0.5cm}
\begin{tabular}{c|c|lll|lll|lll}
\Xhline{1.5pt}
\multirow{2}{*}{Data Index} & \multirow{2}{*}{methods} & \multicolumn{3}{c|}{SR=1\%} & \multicolumn{3}{c|}{SR=5\%} & \multicolumn{3}{c}{SR=10\%} \\ \cline{3-11}
                            &                          & PSNR    & SSIM    & SAM     & PSNR    & SSIM    & SAM     & PSNR     & SSIM    & SAM     \\ \Xhline{1.5pt}
\multirow{6}{*}{$\mathit{WDC~Mall}$}   & \multicolumn{1}{c|}{Observed} & 13.370  & 0.0083  & 1.4969  & 13.549  & 0.0278  & 1.3559  & 13.784   & 0.0491  & 1.2556  \\
                            & \multicolumn{1}{c|}{TNN}                      & 16.499  & 0.2185  & 0.5261  & 28.663  & 0.8002  & 0.1743  & 32.233   & 0.8974  & 0.1239  \\
                            & \multicolumn{1}{c|}{DCT-TNN}                  & 16.540  & 0.2191  & 0.5130  & 29.470  & 0.8272  & 0.1525  & 33.371   & 0.9199  & 0.1060  \\
                            & \multicolumn{1}{c|}{TTNN}                     & 22.646  & 0.5040  & 0.2780  & 32.062  & 0.9023  & 0.1105  & 37.835   & 0.9721  & 0.0605  \\
                            & \multicolumn{1}{c|}{DTNN}                     & 24.587  & 0.6179  & 0.2514  & 32.367  & 0.9051  & 0.1196  & 39.651   & 0.9788  & 0.0505  \\
                            & \multicolumn{1}{c|}{NTTNN}                    & \textBF{25.558}  & \textBF{0.6749}  & \textBF{0.1849}  & \textBF{36.402}  & \textBF{0.9643}  & \textBF{0.0536}  & \textBF{43.251}   & \textBF{0.9930}  & \textBF{0.0219}  \\ \hline
\multirow{6}{*}{$\mathit{Pavia~City}$} & \multicolumn{1}{c|}{Observed}                 & 13.321  & 0.0076  & 1.4991  & 13.500  & 0.0249  & 1.3556  & 13.735   & 0.0463  & 1.2547  \\
                            & \multicolumn{1}{c|}{TNN}                      & 15.643  & 0.1429  & 0.5632  & 28.355  & 0.8312  & 0.1879  & 32.055   & 0.9119  & 0.1533  \\
                            & \multicolumn{1}{c|}{DCT-TNN}                  & 16.464  & 0.1734  & 0.4430  & 30.899  & 0.9023  & 0.1273  & 37.125   & 0.9733  & 0.0767  \\
                            & \multicolumn{1}{c|}{TTNN}                     & 21.477  & 0.4029  & 0.1899  & 32.100  & 0.9237  & 0.1099  & 38.092   & 0.9787  & 0.0687  \\
                            & \multicolumn{1}{c|}{DTNN}                     & 23.190  & 0.5013  & 0.1719  & 31.840  & 0.9258  & 0.1006  & 38.416   & 0.9819  & 0.0616  \\
                            & \multicolumn{1}{c|}{NTTNN}                    & \textBF{24.405}  & \textBF{0.6587}  & \textBF{0.1438}  & \textBF{34.498}  & \textBF{0.9577}  &
                            \textBF{0.0812}  & \textBF{41.672}   & \textBF{0.9903}  & \textBF{0.0468}  \\ \Xhline{1.5pt}
\end{tabular}\label{HSItable}
\end{table}

\begin{figure*}[htp!]
\vspace{-0.7cm}
  \footnotesize
  \setlength{\tabcolsep}{1.2pt}
  \centering
  \begin{tabular}{ccccccc}
  \includegraphics[width=0.14\textwidth]{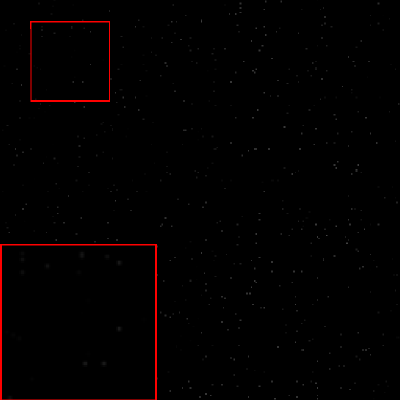}&
  \includegraphics[width=0.14\textwidth]{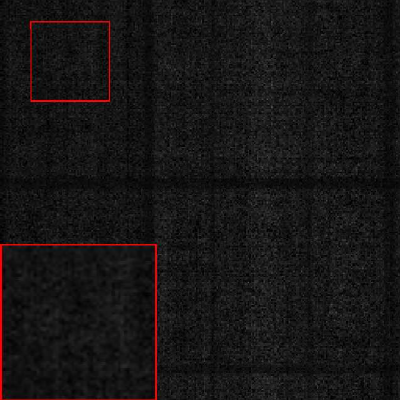}&
  \includegraphics[width=0.14\textwidth]{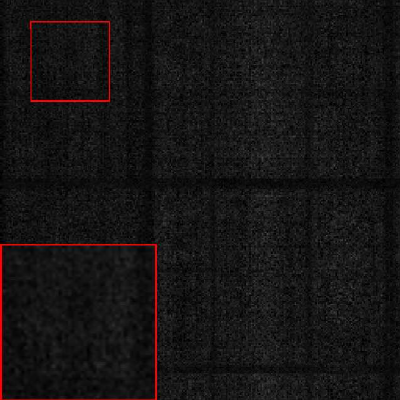}&
  \includegraphics[width=0.14\textwidth]{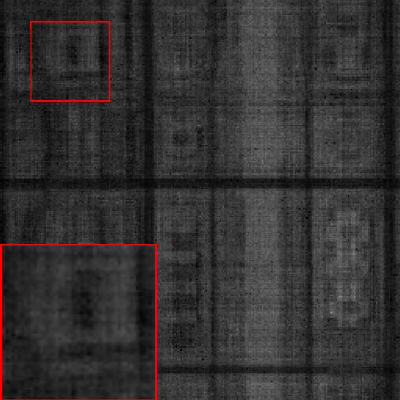}&
  \includegraphics[width=0.14\textwidth]{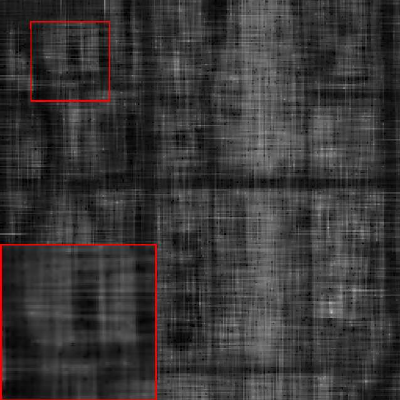}&
  \includegraphics[width=0.14\textwidth]{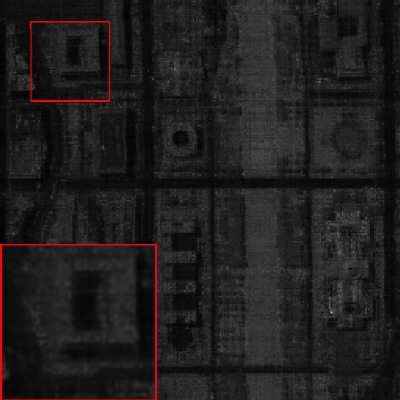}&
  \includegraphics[width=0.14\textwidth]{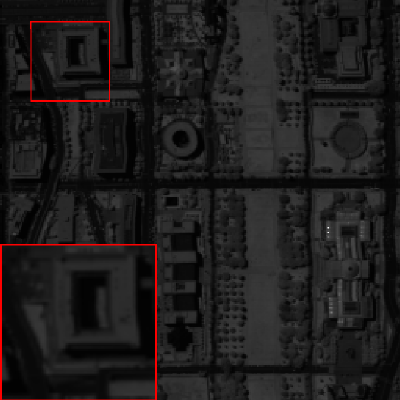}
    \\
  \includegraphics[width=0.14\textwidth]{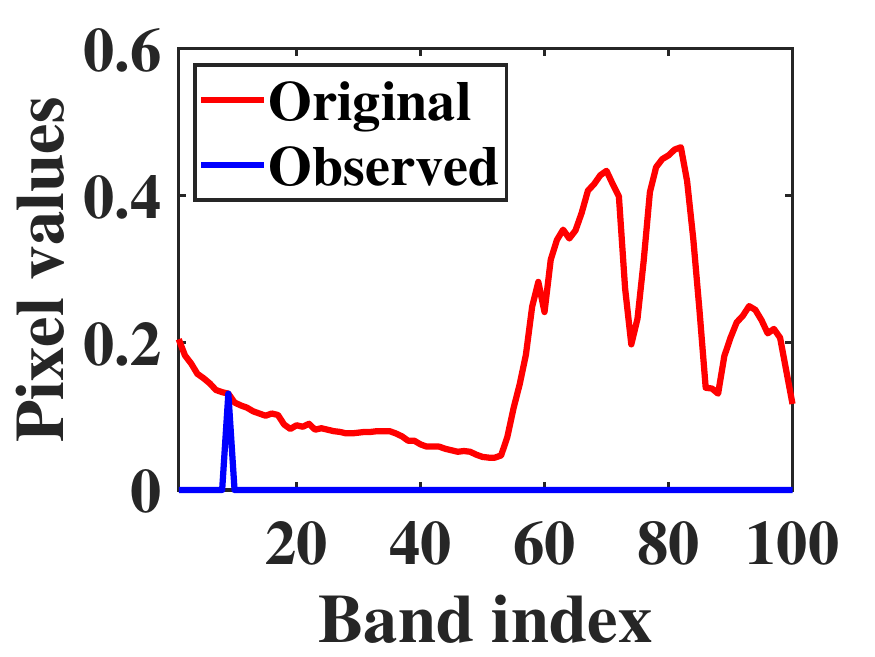}&
  \includegraphics[width=0.14\textwidth]{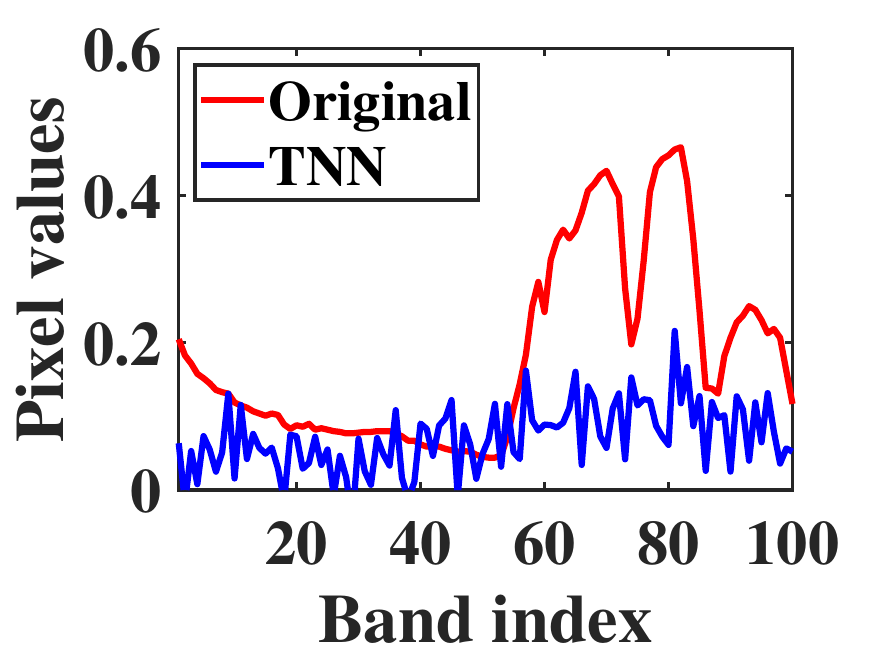}&
  \includegraphics[width=0.14\textwidth]{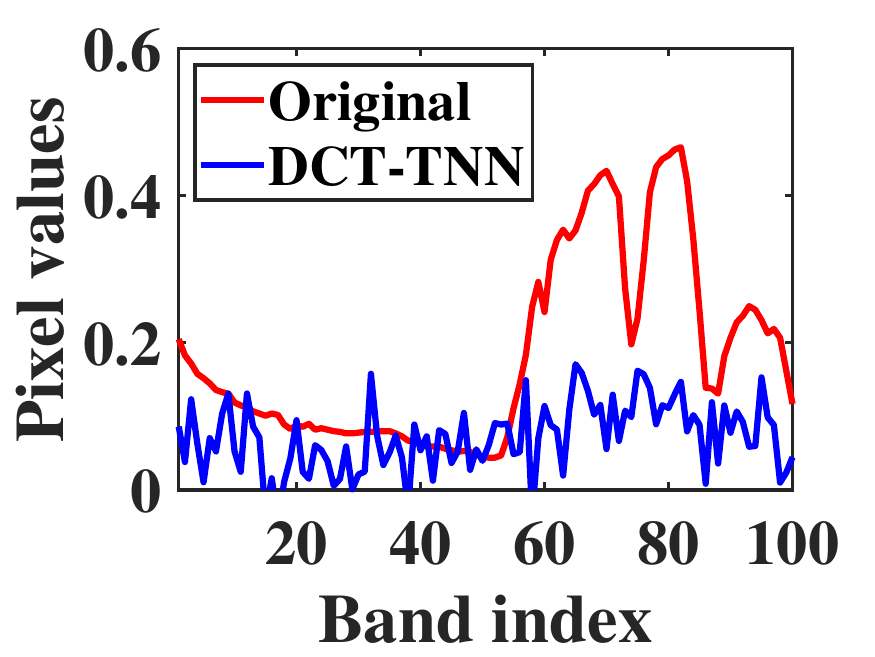}&
  \includegraphics[width=0.14\textwidth]{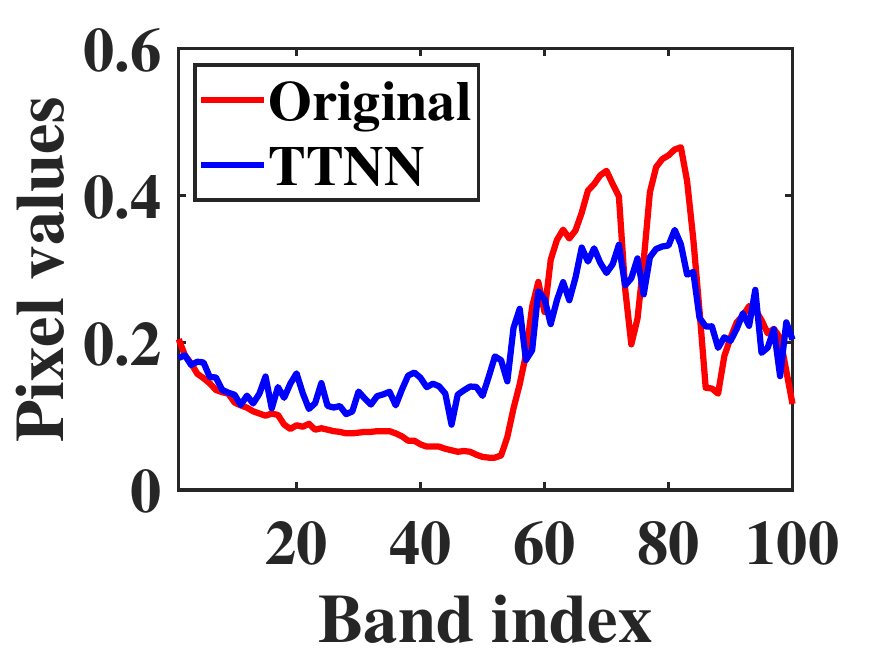}&
  \includegraphics[width=0.14\textwidth]{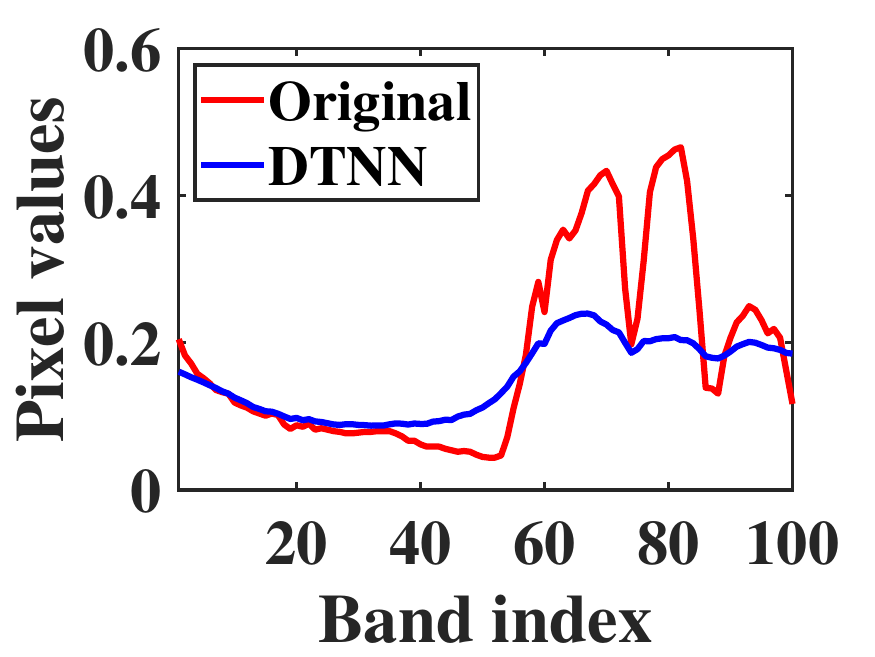}&
  \includegraphics[width=0.14\textwidth]{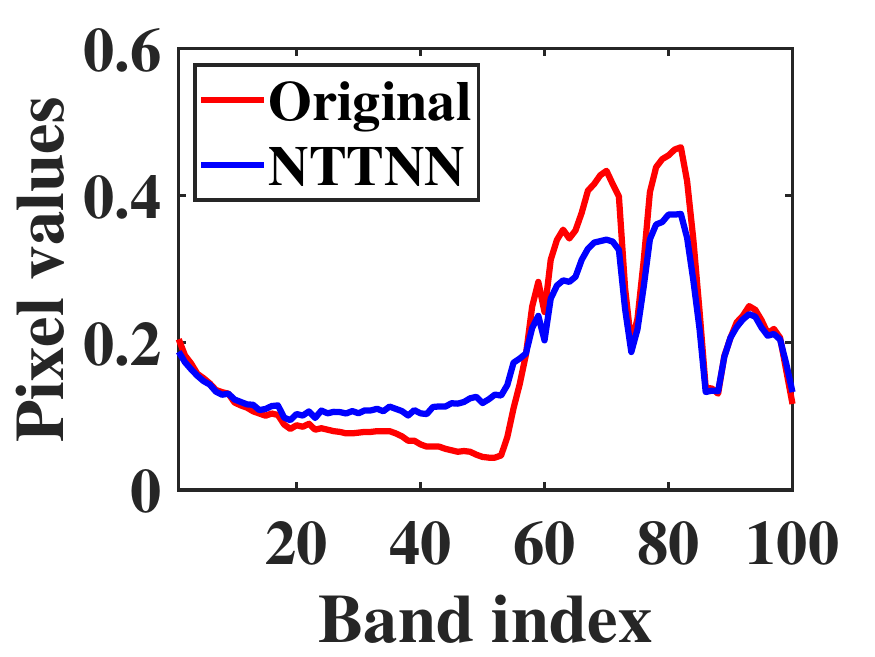}&
  \includegraphics[width=0.14\textwidth]{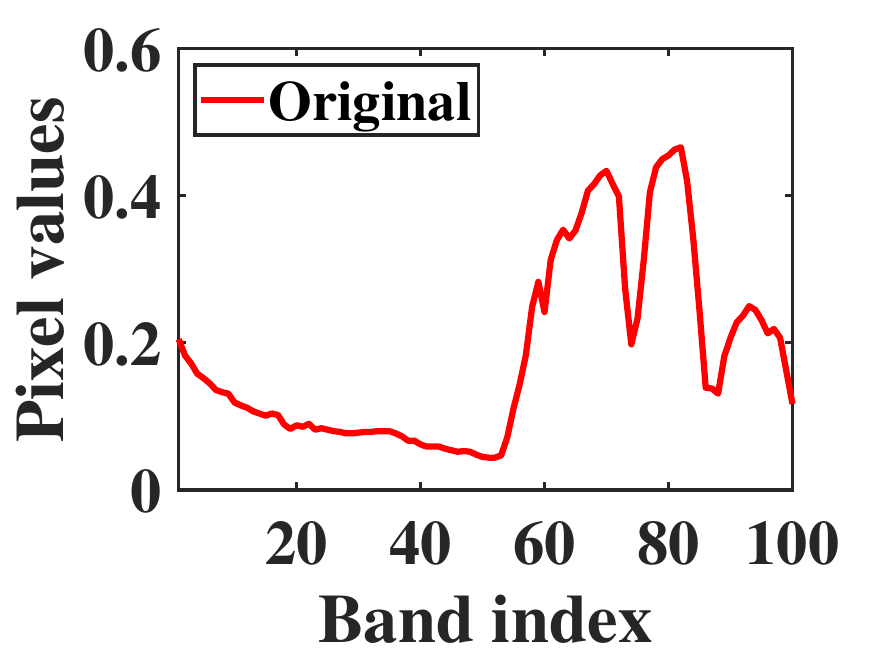}
  \\
  \includegraphics[width=0.14\textwidth]{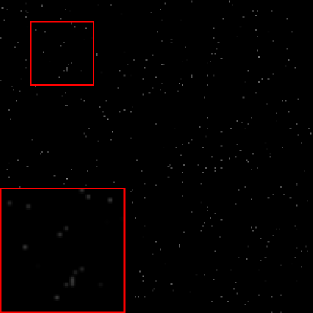}&
  \includegraphics[width=0.14\textwidth]{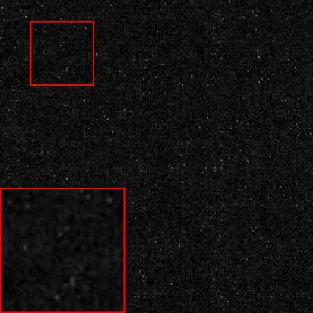}&
  \includegraphics[width=0.14\textwidth]{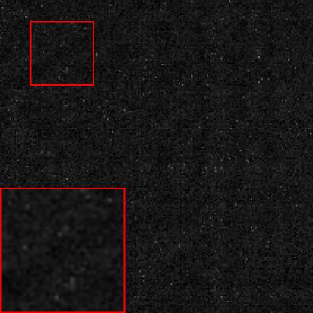}&
  \includegraphics[width=0.14\textwidth]{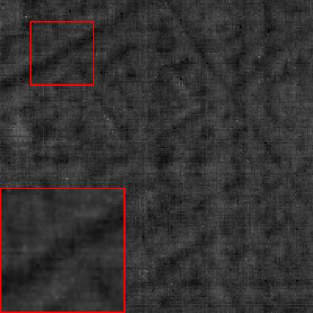}&
  \includegraphics[width=0.14\textwidth]{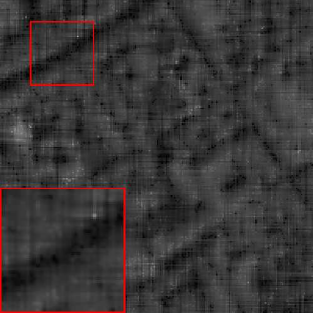}&
  \includegraphics[width=0.14\textwidth]{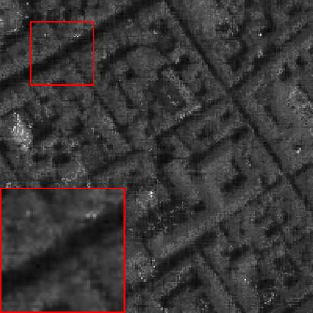}&
  \includegraphics[width=0.14\textwidth]{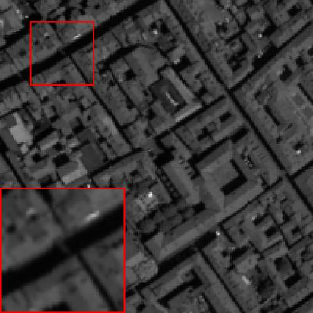}
    \\
  \includegraphics[width=0.14\textwidth]{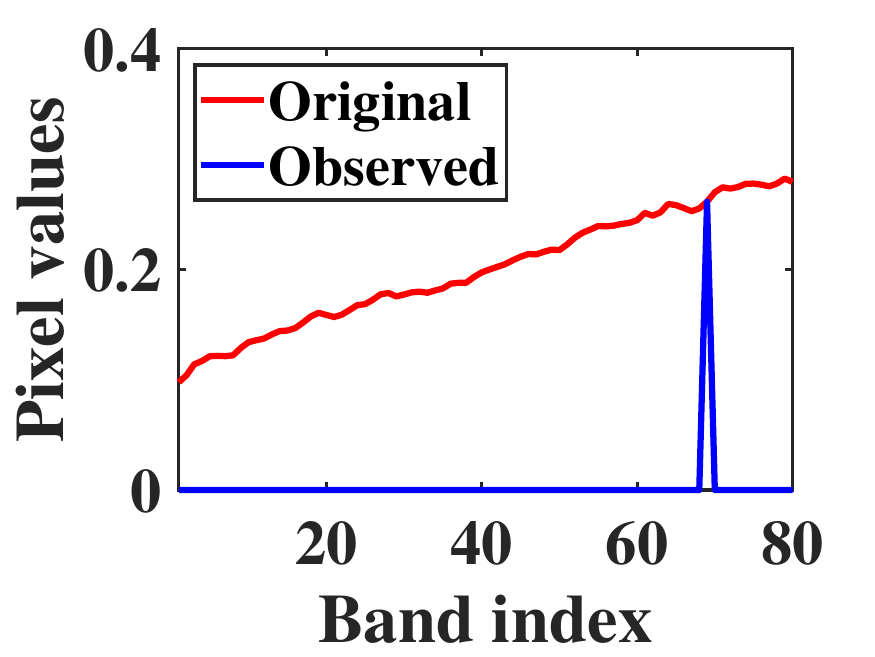}&
  \includegraphics[width=0.14\textwidth]{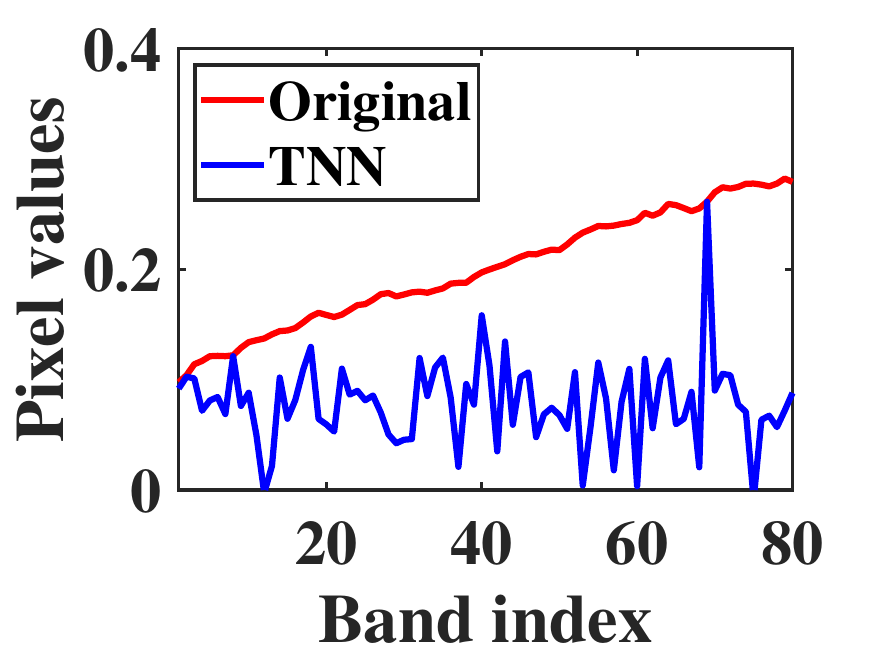}&
  \includegraphics[width=0.14\textwidth]{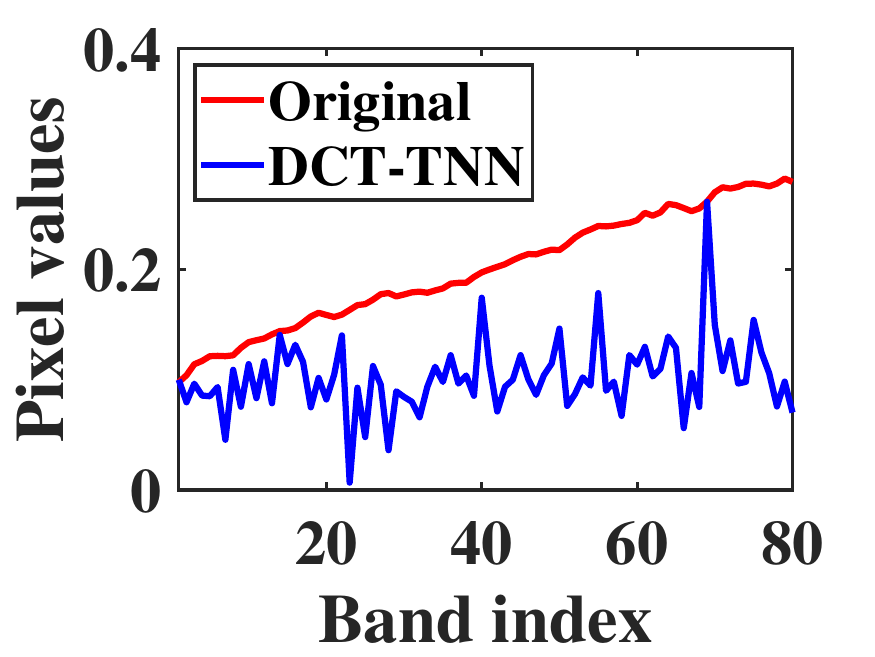}&
  \includegraphics[width=0.14\textwidth]{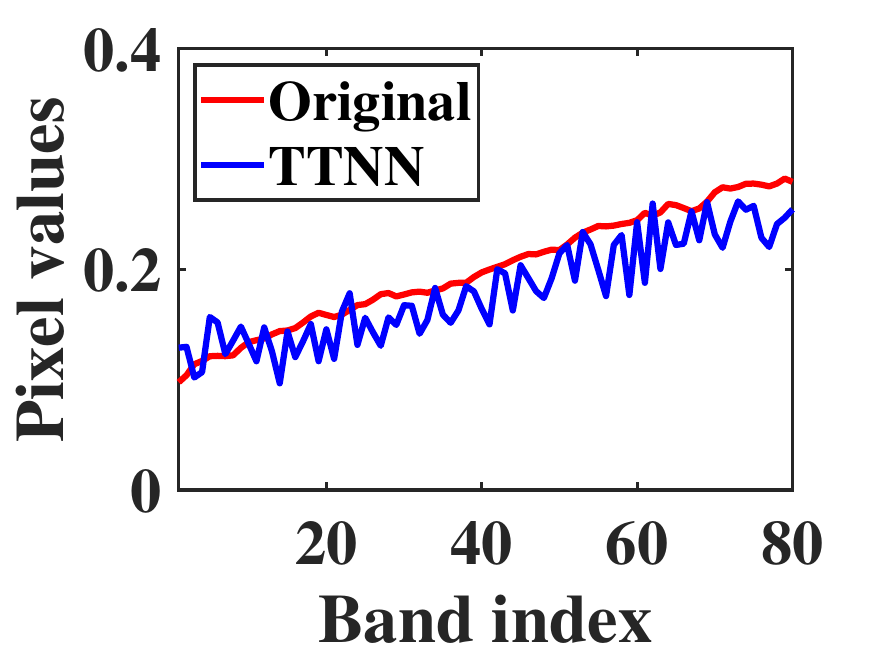}&
  \includegraphics[width=0.14\textwidth]{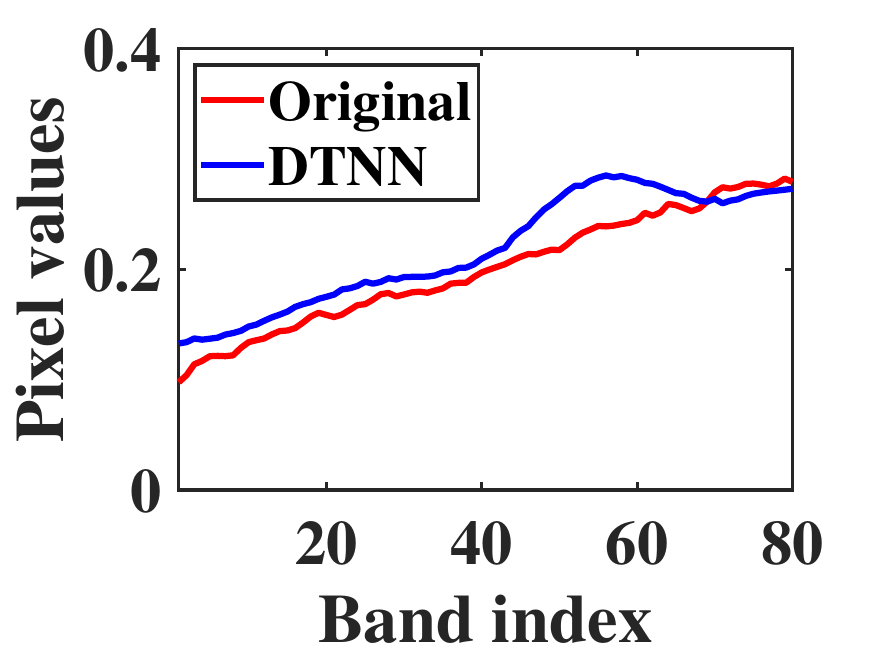}&
  \includegraphics[width=0.14\textwidth]{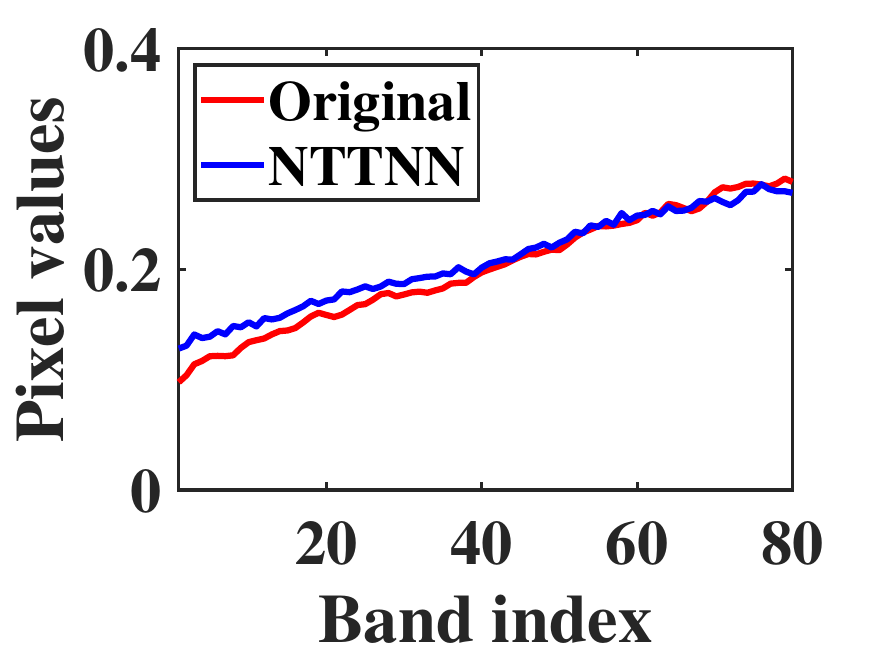}&
  \includegraphics[width=0.14\textwidth]{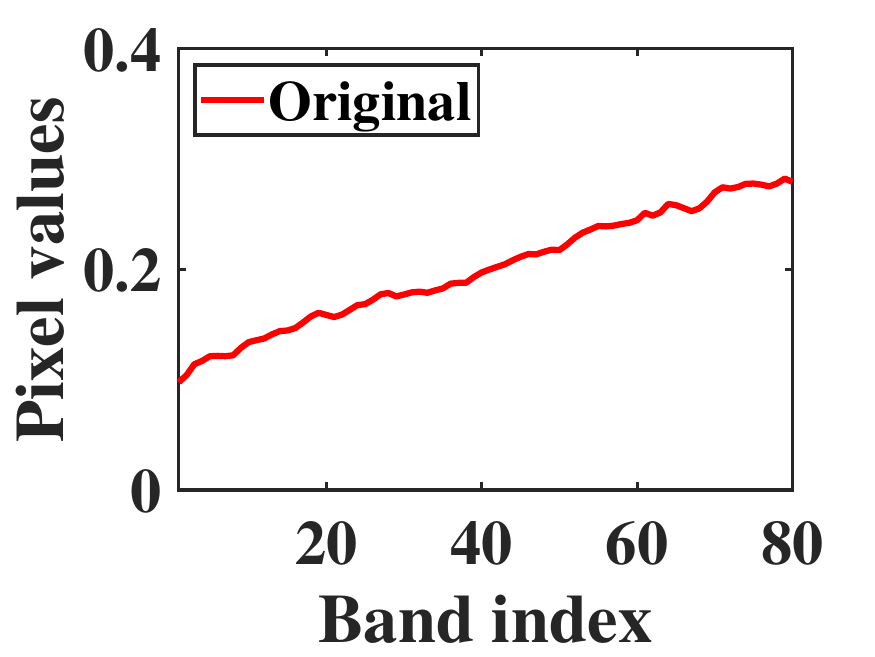}
  \\
   Observed &  TNN & DCT-TNN  &  TTNN & DTNN  & NTTNN & Original
\end{tabular}
   \caption{The results of one band and spectrum profiles at one spatial location of HSIs by different methods for SR = 1\%.
From top to bottom: $\mathit{WDC~Mall}$ and $\mathit{Pavia~City}$, respectively.
From left to right: the observed data, the reconstructed results by TNN, DCT-TNN, TTNN,
DTNN, NTTNN, and the original data, respectively.
   }\label{HSIpic}
\end{figure*}
\subsection{Experiments on MSIs Data}
In this part, we evaluate different methods on five MSIs from the CAVE database\footnote{\url{https://www.cs.columbia.edu/CAVE/databases/multispectral/}}: $\mathit{Balloons}$, $\mathit{Beads}$, $\mathit{Toy}$, $\mathit{Cloth}$, and $\mathit{Feathers}$. All MSIs have been resized to $256\times{256}\times{31}$ in the experiments. The SRs are set to 5\%, 10\%, and 15\%, respectively.

In Table \ref{MSItable}, we show the PSNR, SSIM, and SAM values of the recovered MSIs by different methods for different SRs. We
can note that NTTNN obtains the highest quality results for different MSIs with different SRs. In addition, Fig. \ref{MSIpic} displays the recovered results of one band and spectrum profiles of MSIs by different methods for SR = 5\%. From the visual comparison, it is clear that NTTNN performs best in preserving image edges and details, e.g., the symbols in the zoom-in regions of $\mathit{Toy}$. Moreover, we clearly observe that the spectral curves obtained by NTTNN better approximate the original ones than those obtained by the compared methods.
\begin{table}[htp!]
\centering
\renewcommand\arraystretch{1.5}
\setlength{\tabcolsep}{1.4mm}
\caption{The PSNR, SSIM, and SAM values of the recovered MSIs by different methods for different SRs.}\vspace{0.5cm}
\begin{tabular}{c|c|lll|lll|lll}
\Xhline{1.5pt}
\multirow{2}{*}{Data Index} & \multirow{2}{*}{methods} & \multicolumn{3}{c|}{SR=5\%} & \multicolumn{3}{c|}{SR=10\%} & \multicolumn{3}{c}{SR=15\%} \\ \cline{3-11}
                            &                          & PSNR    & SSIM    & SAM     & PSNR    & SSIM    & SAM     & PSNR     & SSIM    & SAM     \\ \Xhline{1.5pt}
\multirow{6}{*}{$\mathit{Balloons}$}   & \multicolumn{1}{c|}{Observed} &13.349 & 0.0959 & 1.5163 &13.762&0.1613&1.2774&14.011&0.1896&1.1934\\
                            & \multicolumn{1}{c|}{TNN}                 &31.642&0.8665&0.1940&36.270&0.9412&0.1230&39.375&0.9681&0.0894\\
                            & \multicolumn{1}{c|}{DCT-TNN}             &32.639&0.8911&0.1720&37.171&0.9521&0.1099&40.524&0.9758&0.0772\\
                            & \multicolumn{1}{c|}{TTNN}                &33.641&0.9259&0.1581&37.814&0.9608&0.1005& 41.957&0.9830&0.0651\\
                            & \multicolumn{1}{c|}{DTNN}                &33.418&0.9218&0.1580&37.394&0.9559&0.1145&42.982&0.9831&0.0671\\
                            & \multicolumn{1}{c|}{NTTNN}                    &\textBF{35.425}&\textBF{0.9387}&\textBF{0.1268}&\textBF{40.458}
                            &\textBF{0.9757}&\textBF{0.0784}&\textBF{43.633}&\textBF{0.9865}
                            &\textBF{0.0600}   \\ \hline
\multirow{6}{*}{$\mathit{Beads}$} & \multicolumn{1}{c|}{Observed}                 &14.416&0.1188&1.4032&14.651&0.1548&1.2956&14.898&0.1926&1.2105\\
                            & \multicolumn{1}{c|}{TNN}                      &19.364&0.4086&0.5922&23.508&0.6604&0.4270&26.052&0.7741&0.3381\\
                            & \multicolumn{1}{c|}{DCT-TNN}                  &19.696&0.4272&0.5629&23.434&0.6593&0.4109&26.238&0.7847&0.3181\\
                            & \multicolumn{1}{c|}{TTNN}                     &22.934&0.6789&0.4033&25.786&0.8086&0.3122&28.071&0.8458&0.2662\\
                            & \multicolumn{1}{c|}{DTNN}                     &22.827&0.6950&0.3933&25.659&0.8262&0.2987&30.145&0.9181&0.1895\\
                            & \multicolumn{1}{c|}{NTTNN}                    &\textBF{23.917}&\textBF{0.7162}&\textBF{0.3784}&\textBF{28.106}
                            &\textBF{0.8659}&\textBF{0.2484}&\textBF{31.327}&\textBF{0.9251}
                            &\textBF{0.1831}\\ \hline
\multirow{6}{*}{$\mathit{Toy}$} & \multicolumn{1}{c|}{Observed}             &10.631&0.2565&1.3874&10.866&0.2925&1.2821&11.114&0.3271&1.1991\\
                            & \multicolumn{1}{c|}{TNN}                      &28.749&0.8471&0.3287&32.549&0.9197&0.2295&35.453&0.9520&0.1734\\
                            & \multicolumn{1}{c|}{DCT-TNN}                  &28.462&0.8508&0.3054&33.487&0.9396&0.1903&36.599&0.9650&0.1390\\
                            & \multicolumn{1}{c|}{TTNN}                     &29.271&0.8653&0.2986&34.270&0.9435&0.1875&37.801&0.9704&0.1323\\
                            & \multicolumn{1}{c|}{DTNN}                     &29.023&0.8857&0.2998&32.838&0.9306&0.2451&38.470&0.9783&0.1086\\
                            & \multicolumn{1}{c|}{NTTNN}                    &\textBF{30.636}&\textBF{0.9165}&\textBF{0.2333}&\textBF{35.058}
                            &\textBF{0.9572}&\textBF{0.1659}&\textBF{39.711}&\textBF{0.9798}
                            &\textBF{0.1222}\\ \hline
\multirow{6}{*}{$\mathit{Cloth}$} & \multicolumn{1}{c|}{Observed}             &11.699&0.0336&1.3939&11.933&0.0578&1.2821&12.181&0.0829&1.1963\\
                            & \multicolumn{1}{c|}{TNN}                      &20.085&0.4275&0.2685&24.889&0.7225&0.1699&27.857&0.8344&0.1285\\
                            & \multicolumn{1}{c|}{DCT-TNN}                  &21.777&0.5225&0.2237&26.227&0.7790&0.1385&29.223&0.8743&0.1026\\
                            & \multicolumn{1}{c|}{TTNN}                     &22.749&0.6077&0.2155&25.458&0.7511&0.1474&28.622&0.8655&0.1070\\
                            & \multicolumn{1}{c|}{DTNN}                     &24.036&0.7139&0.1875&27.883&0.8666&0.1247&31.452&0.9330&0.0874\\
                            & \multicolumn{1}{c|}{NTTNN}                    &\textBF{25.109}&\textBF{0.7608}&\textBF{0.1464}&\textBF{29.779}
                            &\textBF{0.8971}&\textBF{0.0931}&\textBF{33.345}&\textBF{0.9457}
                            &\textBF{0.0684}\\ \hline
\multirow{6}{*}{$\mathit{Feathers}$} & \multicolumn{1}{c|}{Observed}             &13.356&0.1907&1.4062&13.590&0.2310&1.3008&13.838&0.2693&1.2162\\
                            & \multicolumn{1}{c|}{TNN}                      &25.029&0.7053&0.3531&31.624&0.8733&0.2063&34.571&0.9235&0.1529\\
                            & \multicolumn{1}{c|}{DCT-TNN}                  &27.842&0.7861&0.2713&32.581&0.8980&0.1697&35.763&0.9428&0.1223\\
                            & \multicolumn{1}{c|}{TTNN}                     &28.650&0.8119&0.2548&33.267&0.9133&0.1539&36.689&0.9546&0.1062\\
                            & \multicolumn{1}{c|}{DTNN}                     &28.164&0.8391&0.3085&33.109&0.9313&0.1725&37.198&0.9635&0.1179\\
                            & \multicolumn{1}{c|}{NTTNN}                    &\textBF{30.515}&\textBF{0.8801}&\textBF{0.2063}&\textBF{35.581}
                            &\textBF{0.9474}&\textBF{0.1225}&\textBF{39.229}&\textBF{0.9712}
                            &\textBF{0.0883}\\ \Xhline{1.5pt}
\end{tabular}\label{MSItable}
\end{table}

\begin{figure*}[htp!]
\vspace{-0.7cm}
  \footnotesize
  \setlength{\tabcolsep}{1.2pt}
  \centering
  \begin{tabular}{ccccccc}
  \includegraphics[width=0.14\textwidth]{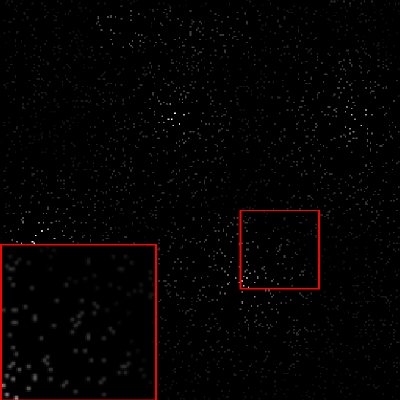}&
  \includegraphics[width=0.14\textwidth]{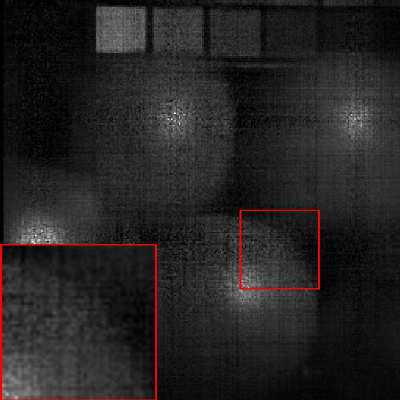}&
  \includegraphics[width=0.14\textwidth]{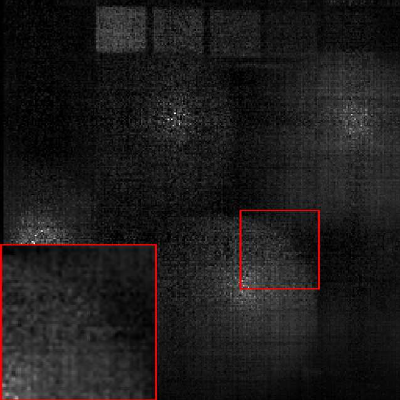}&
  \includegraphics[width=0.14\textwidth]{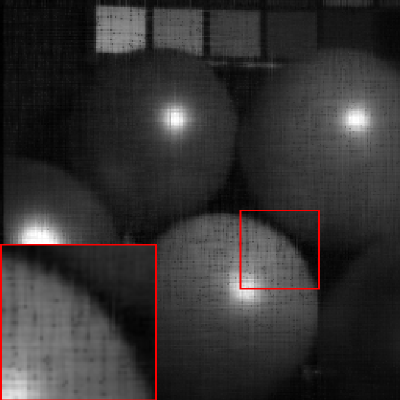}&
  \includegraphics[width=0.14\textwidth]{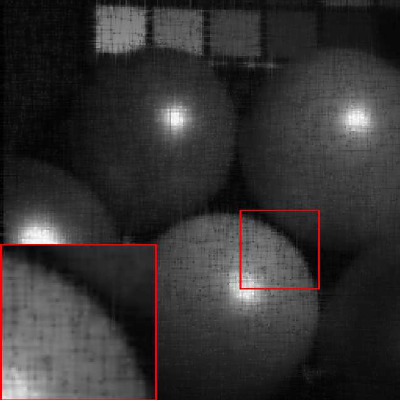}&
  \includegraphics[width=0.14\textwidth]{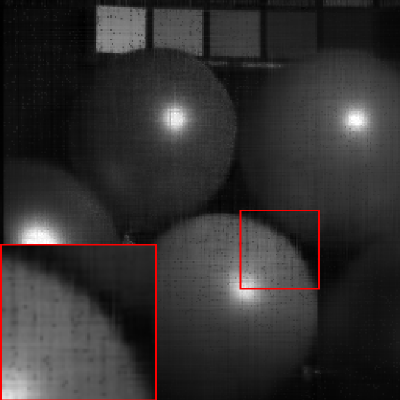}&
  \includegraphics[width=0.14\textwidth]{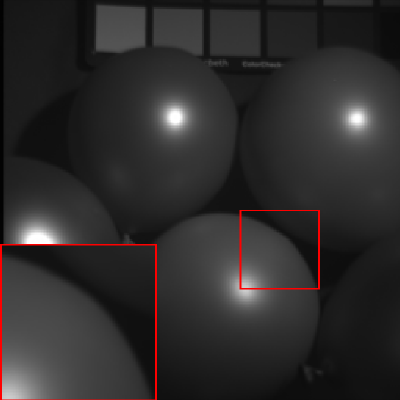}
  \\
  \includegraphics[width=0.14\textwidth]{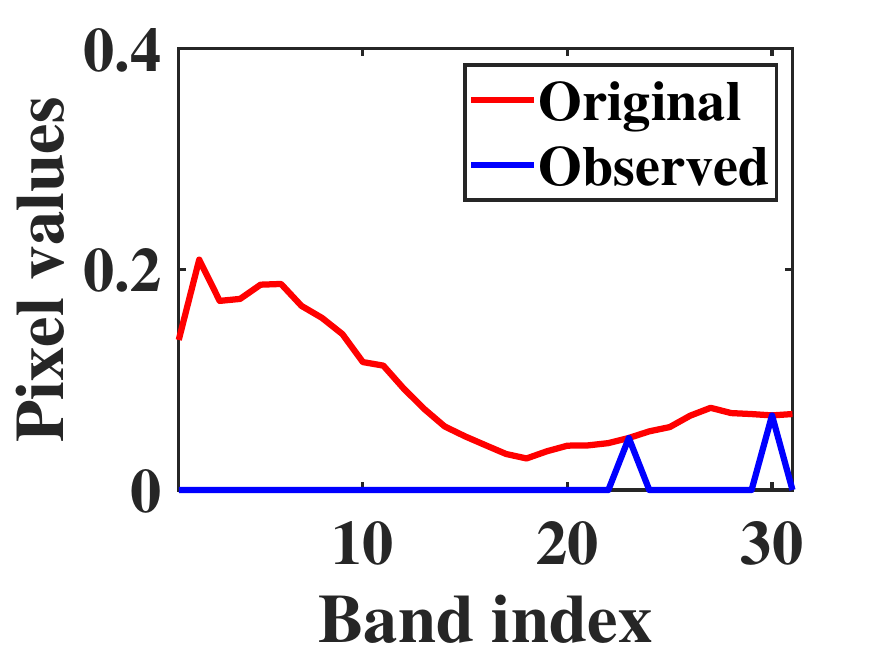}&
  \includegraphics[width=0.14\textwidth]{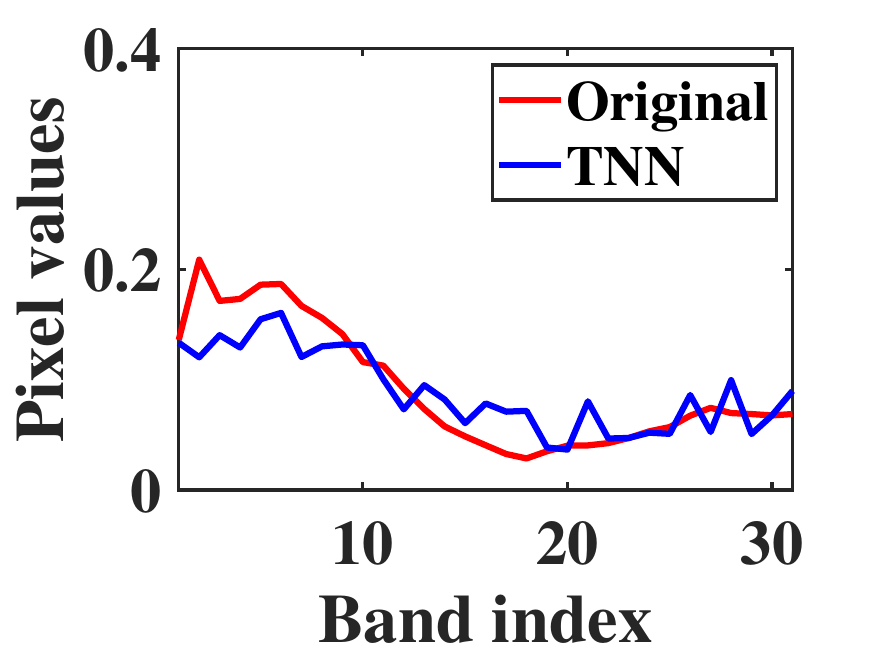}&
  \includegraphics[width=0.14\textwidth]{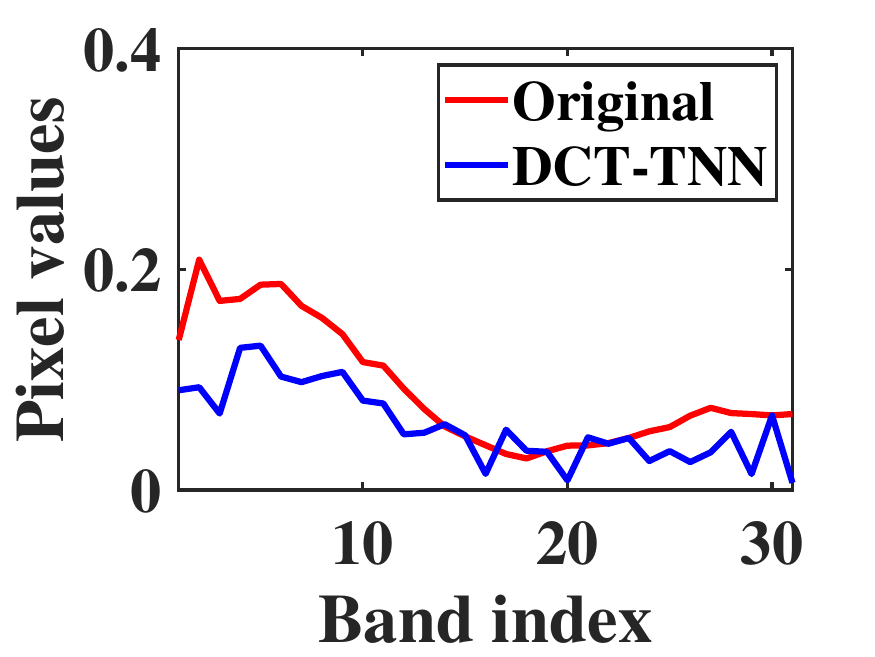}&
  \includegraphics[width=0.14\textwidth]{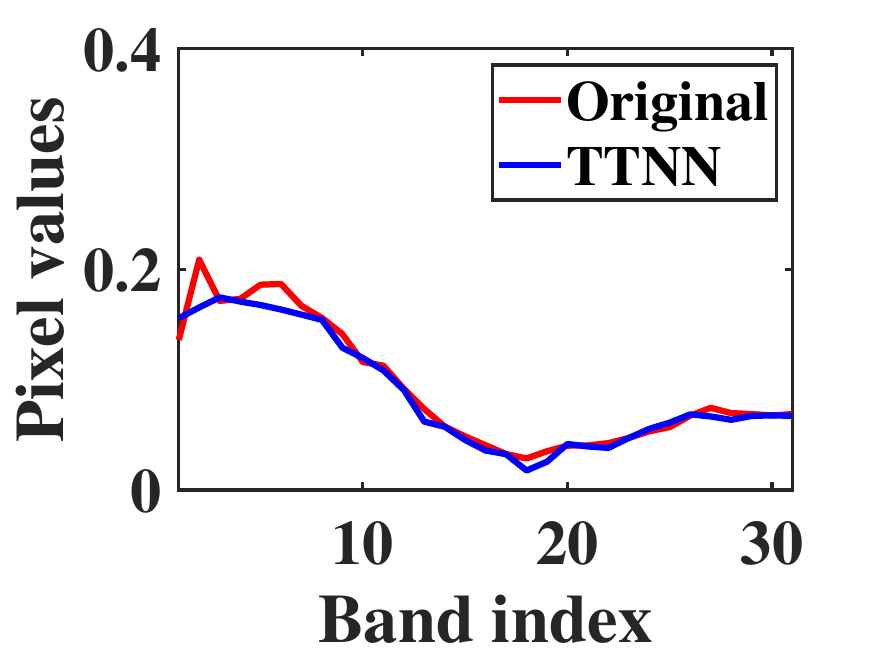}&
  \includegraphics[width=0.14\textwidth]{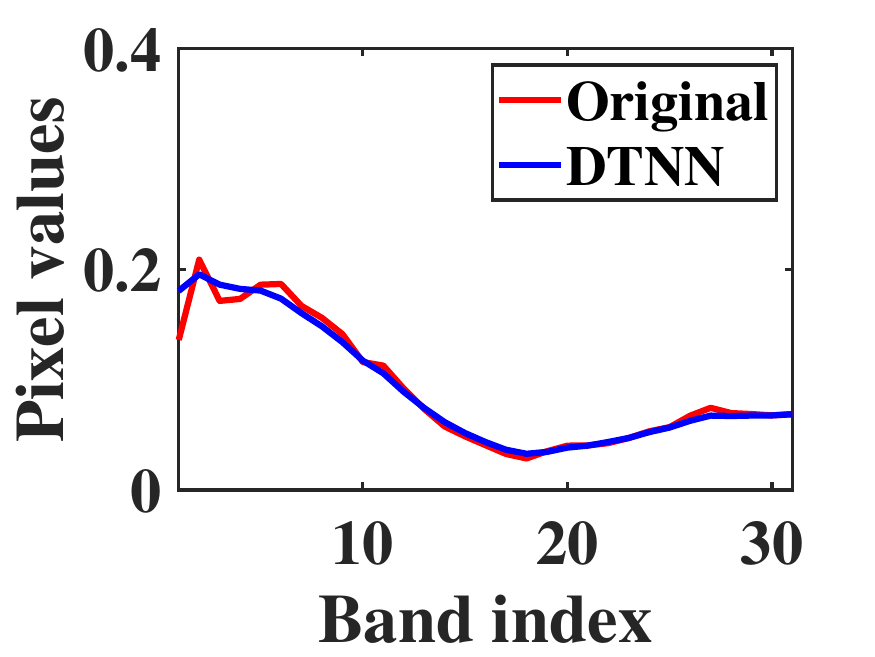}&
  \includegraphics[width=0.14\textwidth]{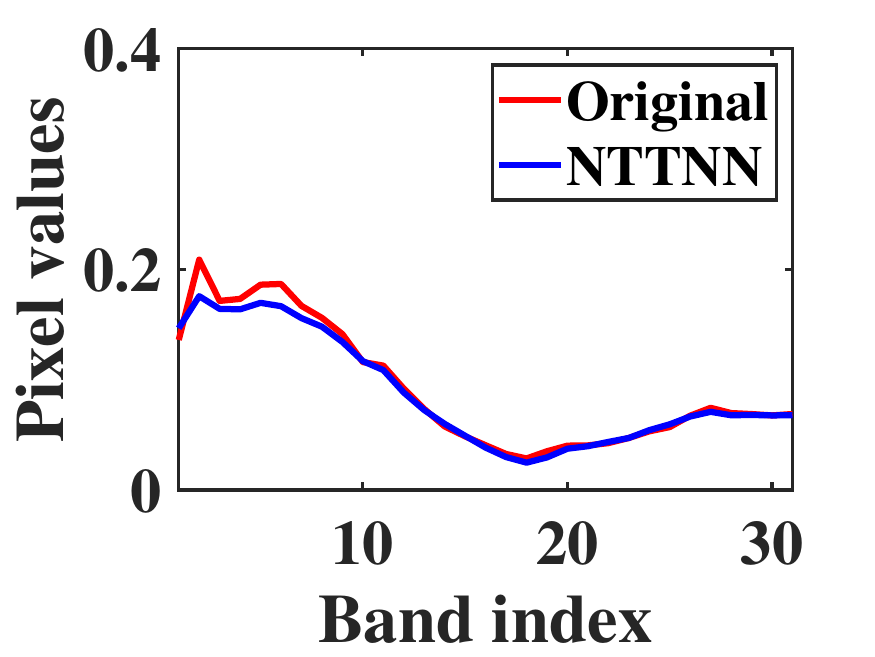}&
  \includegraphics[width=0.14\textwidth]{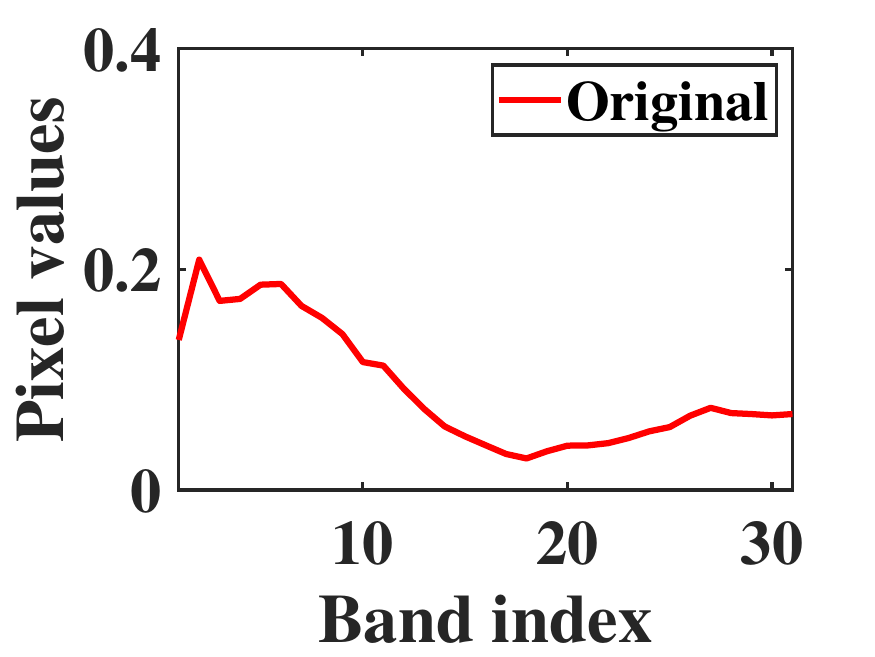}
  \\
  \includegraphics[width=0.14\textwidth]{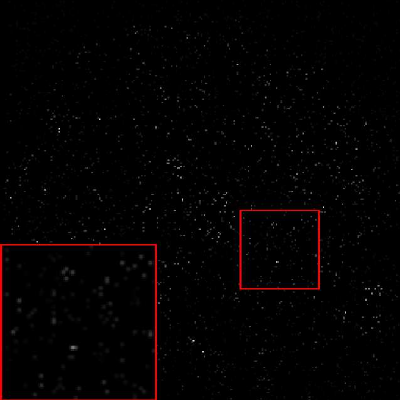}&
  \includegraphics[width=0.14\textwidth]{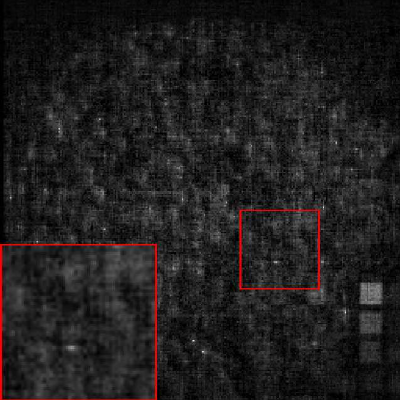}&
  \includegraphics[width=0.14\textwidth]{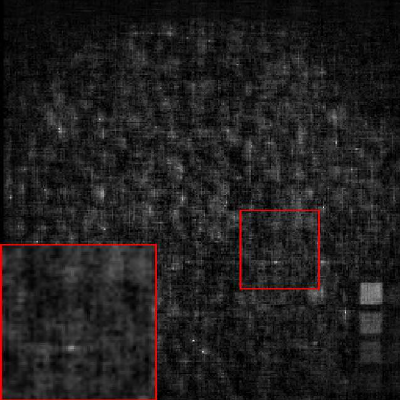}&
  \includegraphics[width=0.14\textwidth]{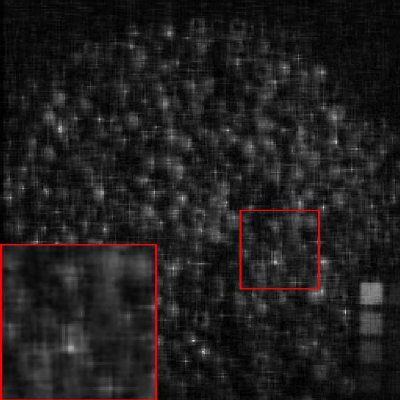}&
  \includegraphics[width=0.14\textwidth]{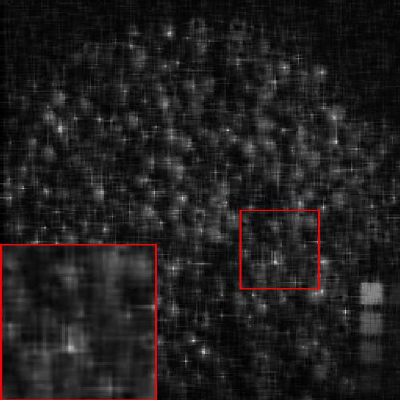}&
  \includegraphics[width=0.14\textwidth]{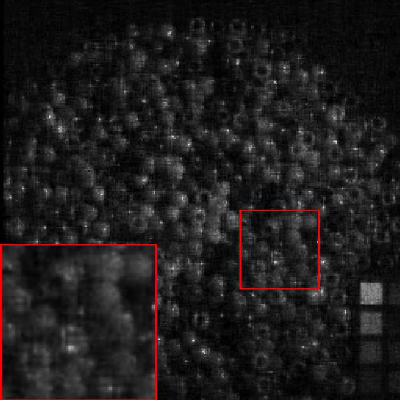}&
  \includegraphics[width=0.14\textwidth]{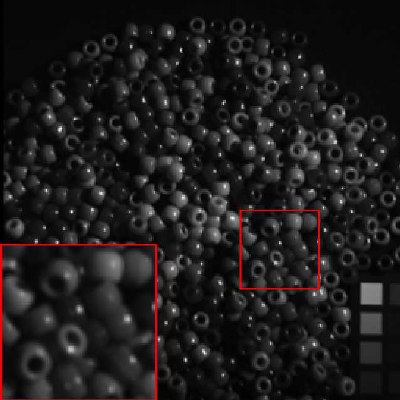}
  \\
  \includegraphics[width=0.14\textwidth]{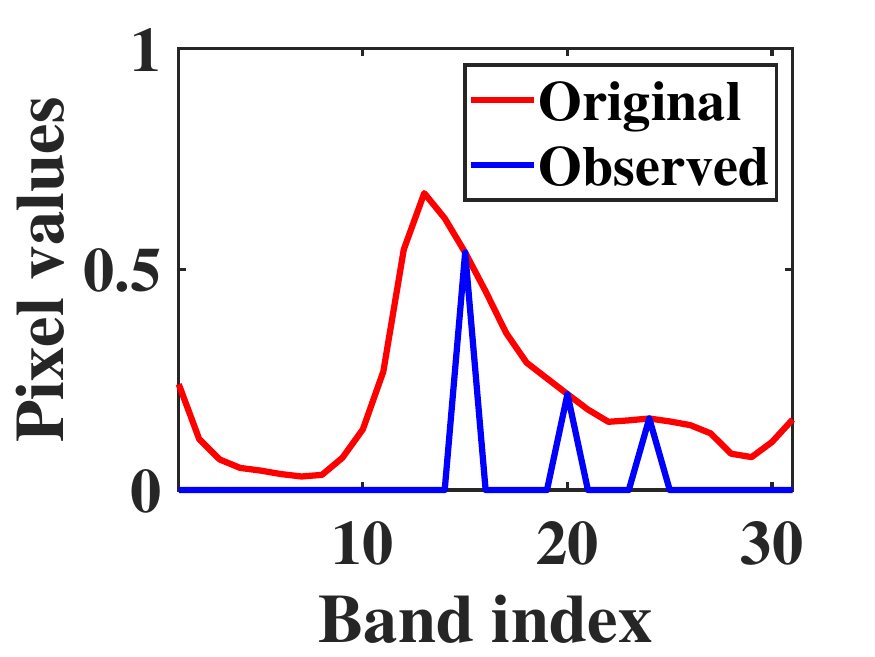}&
  \includegraphics[width=0.14\textwidth]{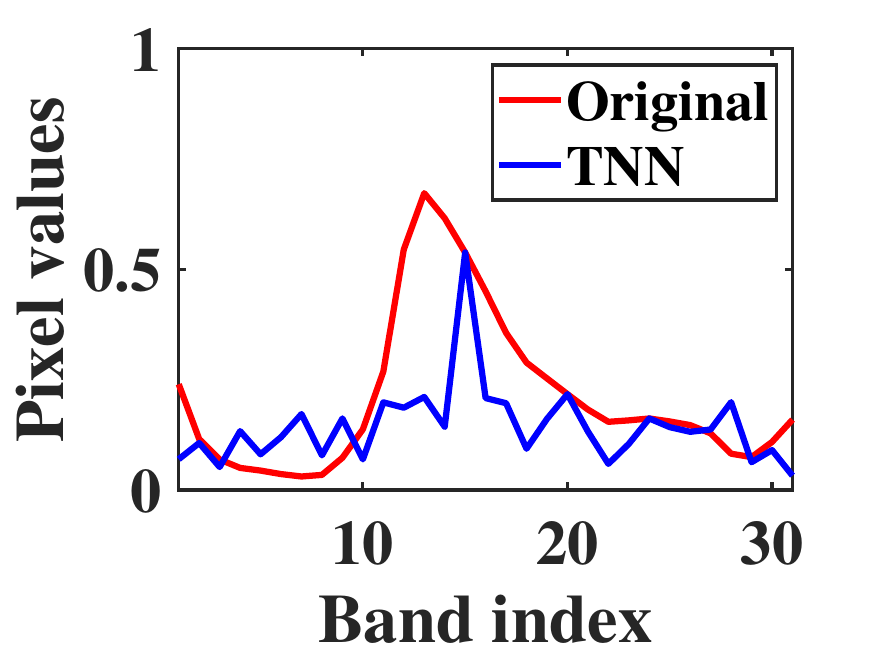}&
  \includegraphics[width=0.14\textwidth]{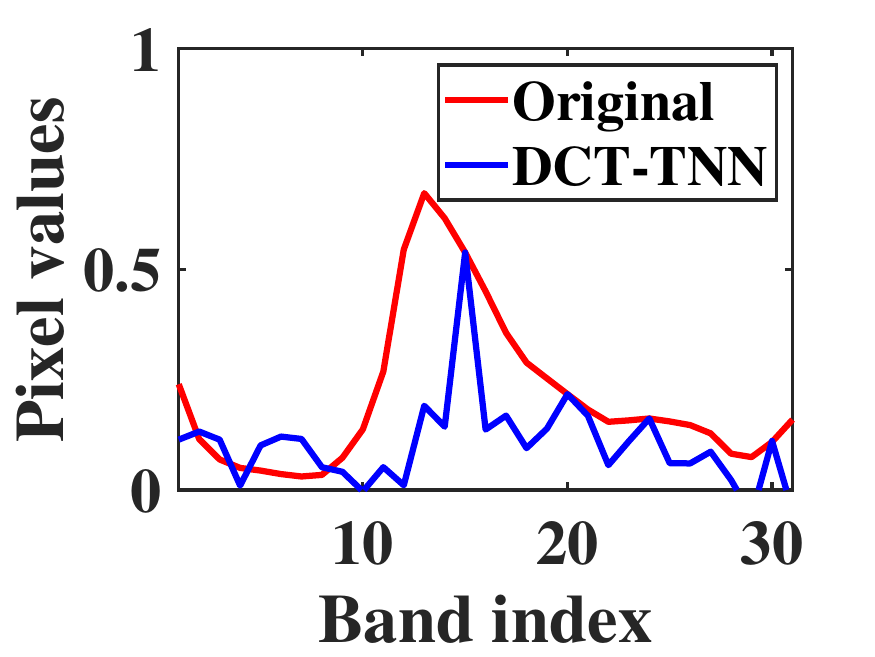}&
  \includegraphics[width=0.14\textwidth]{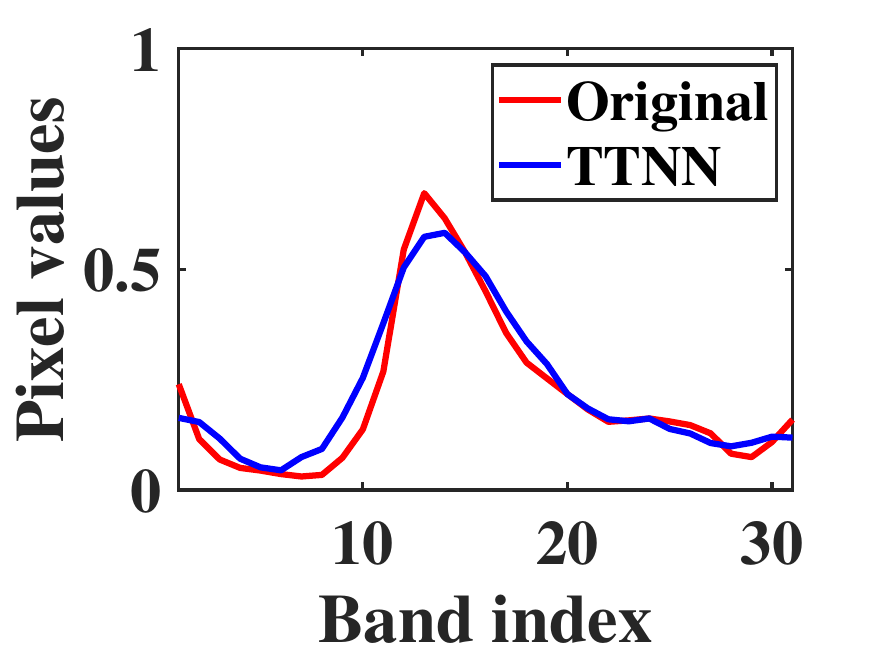}&
  \includegraphics[width=0.14\textwidth]{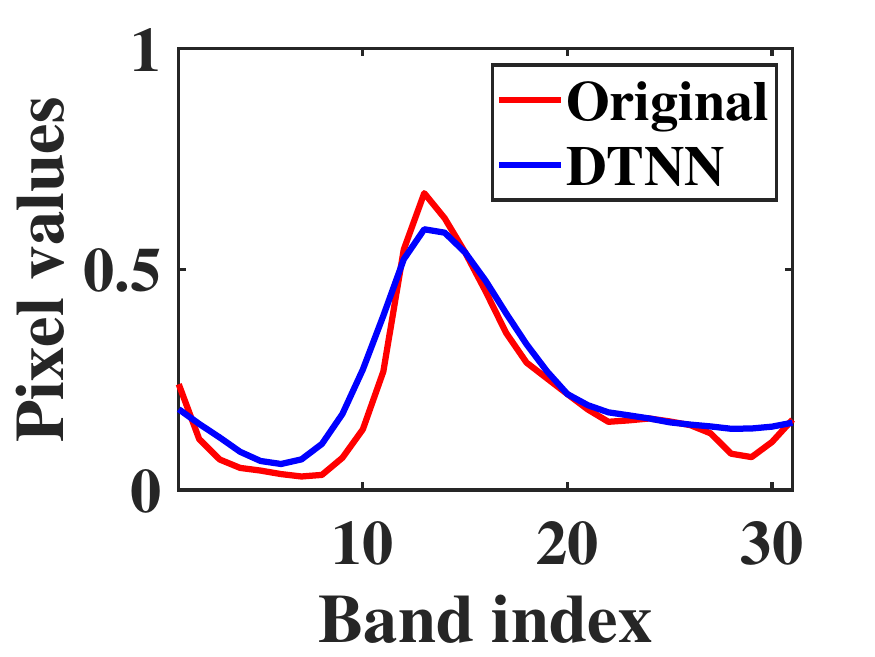}&
  \includegraphics[width=0.14\textwidth]{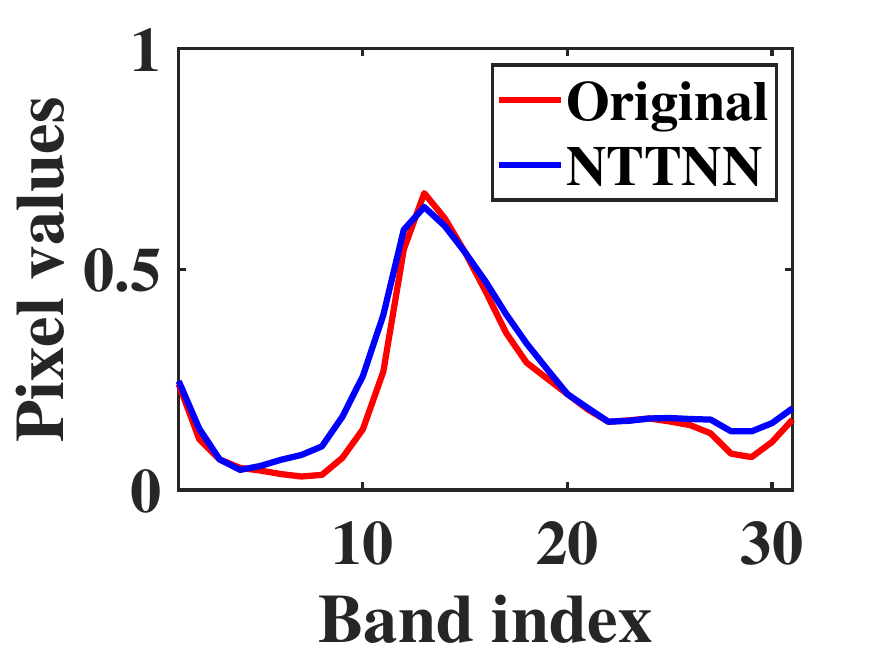}&
  \includegraphics[width=0.14\textwidth]{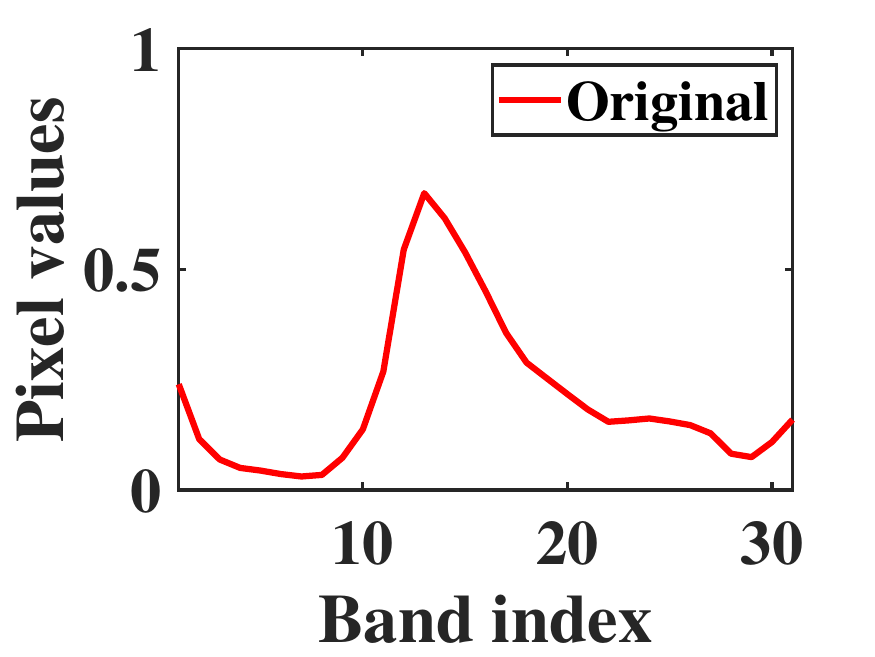}
  \\
  \includegraphics[width=0.14\textwidth]{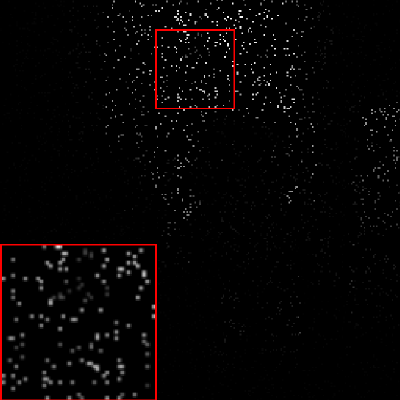}&
  \includegraphics[width=0.14\textwidth]{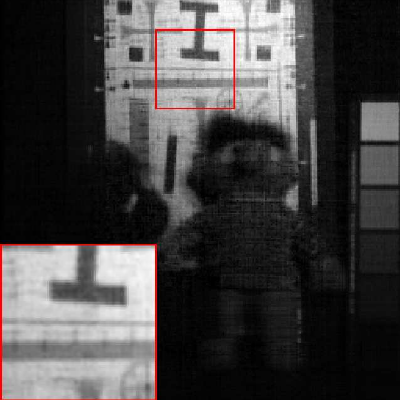}&
  \includegraphics[width=0.14\textwidth]{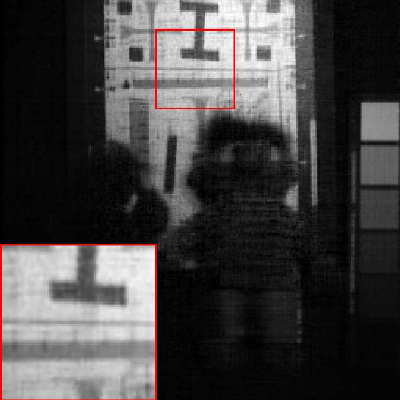}&
  \includegraphics[width=0.14\textwidth]{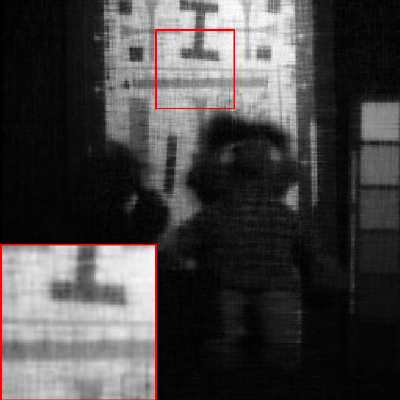}&
  \includegraphics[width=0.14\textwidth]{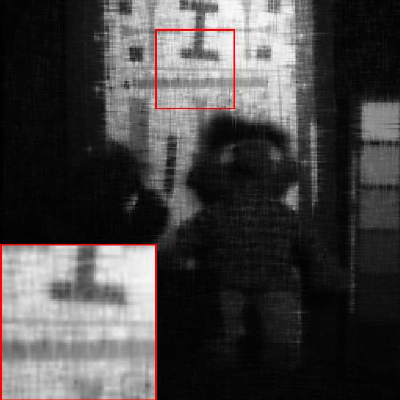}&
  \includegraphics[width=0.14\textwidth]{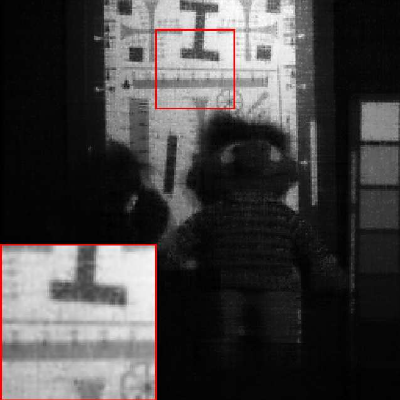}&
  \includegraphics[width=0.14\textwidth]{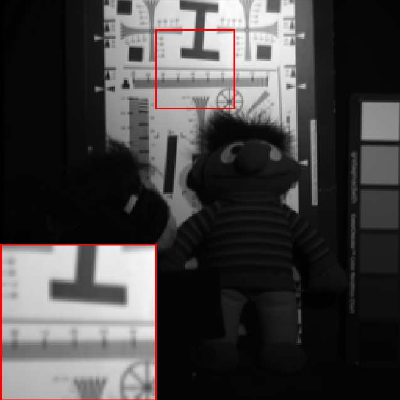}
  \\
  \includegraphics[width=0.14\textwidth]{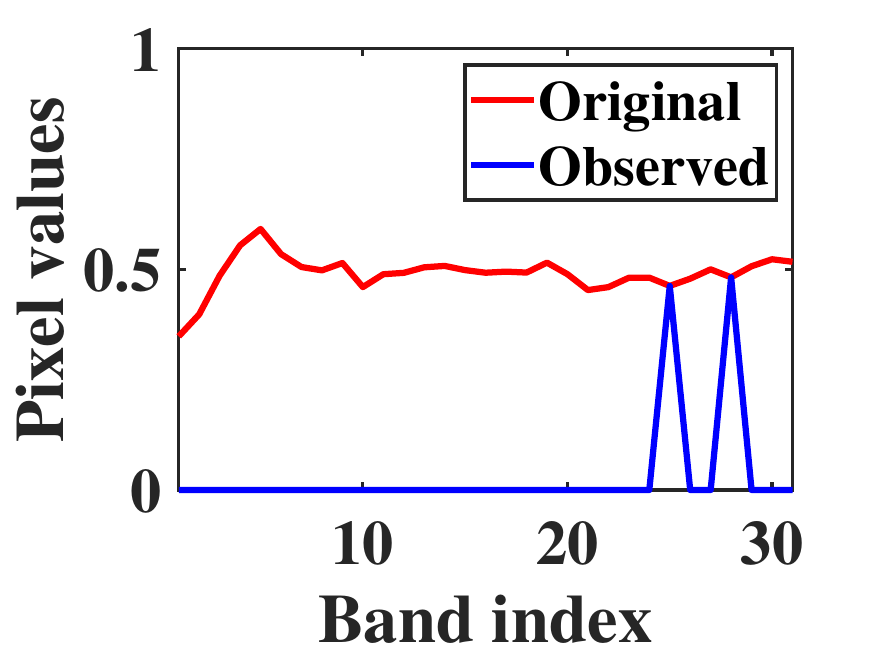}&
  \includegraphics[width=0.14\textwidth]{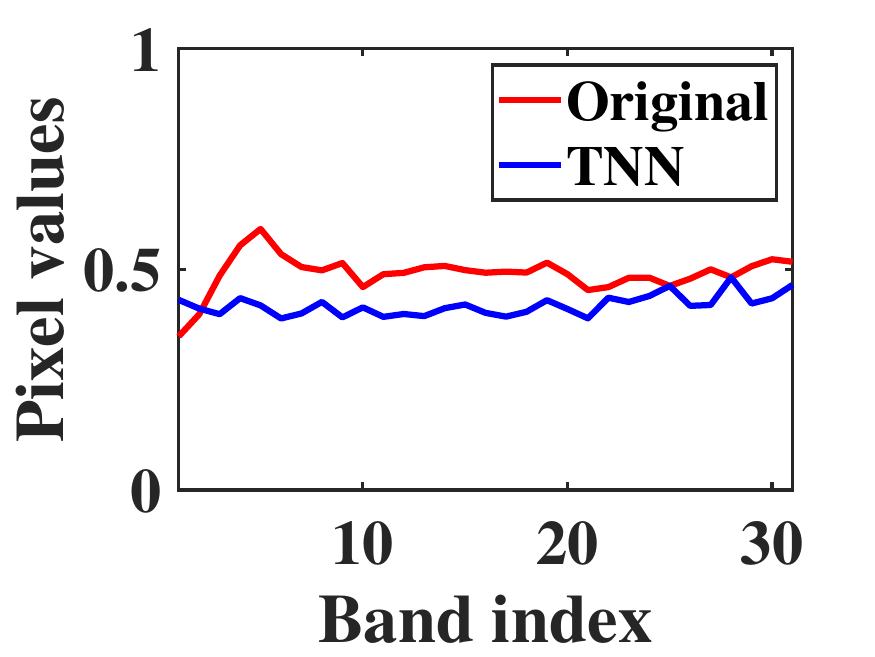}&
  \includegraphics[width=0.14\textwidth]{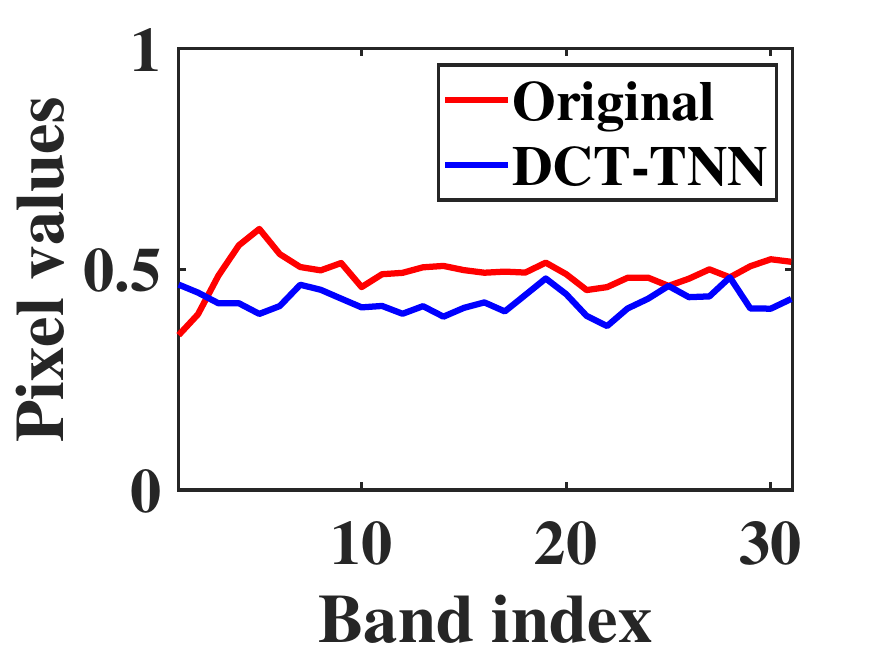}&
  \includegraphics[width=0.14\textwidth]{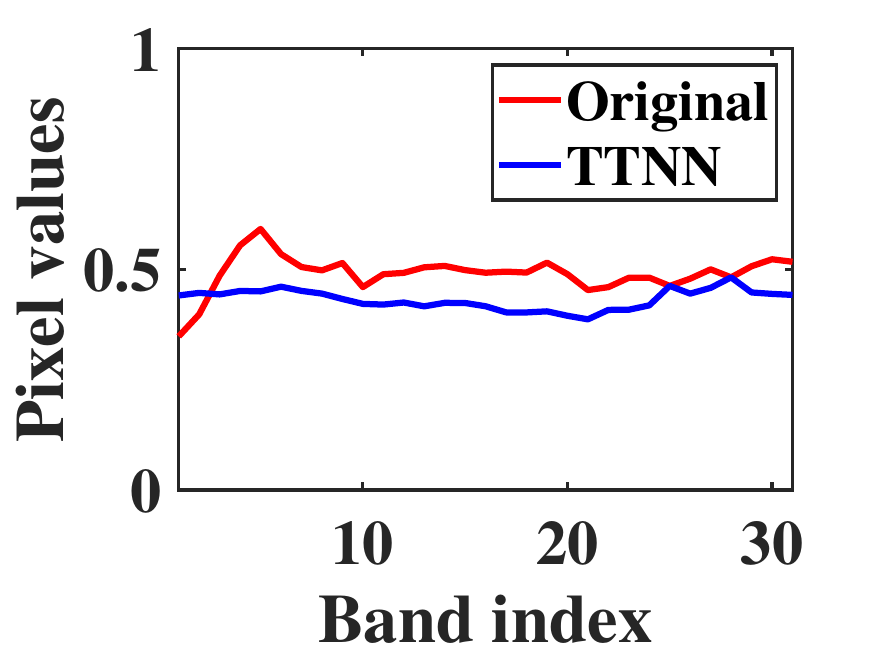}&
  \includegraphics[width=0.14\textwidth]{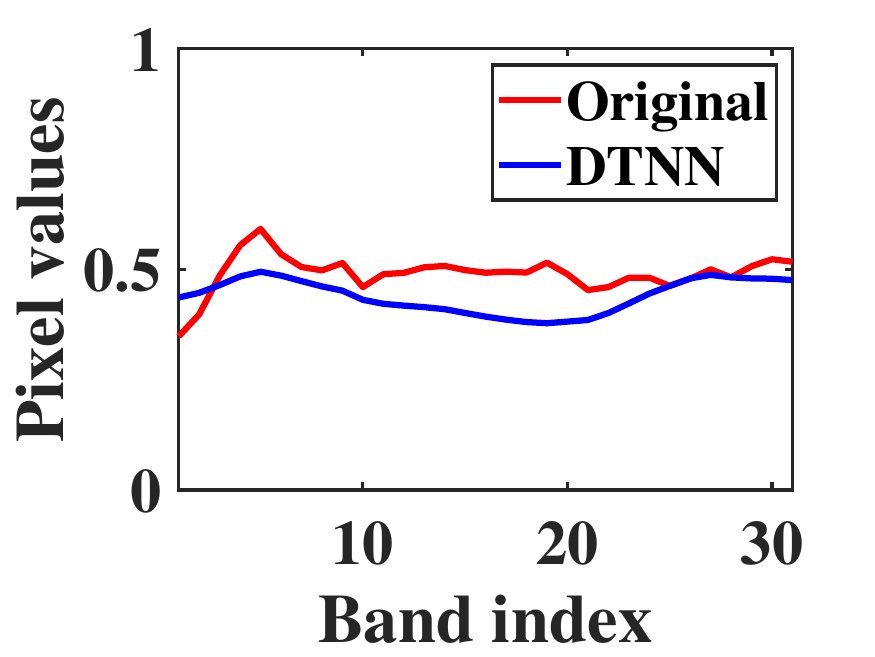}&
  \includegraphics[width=0.14\textwidth]{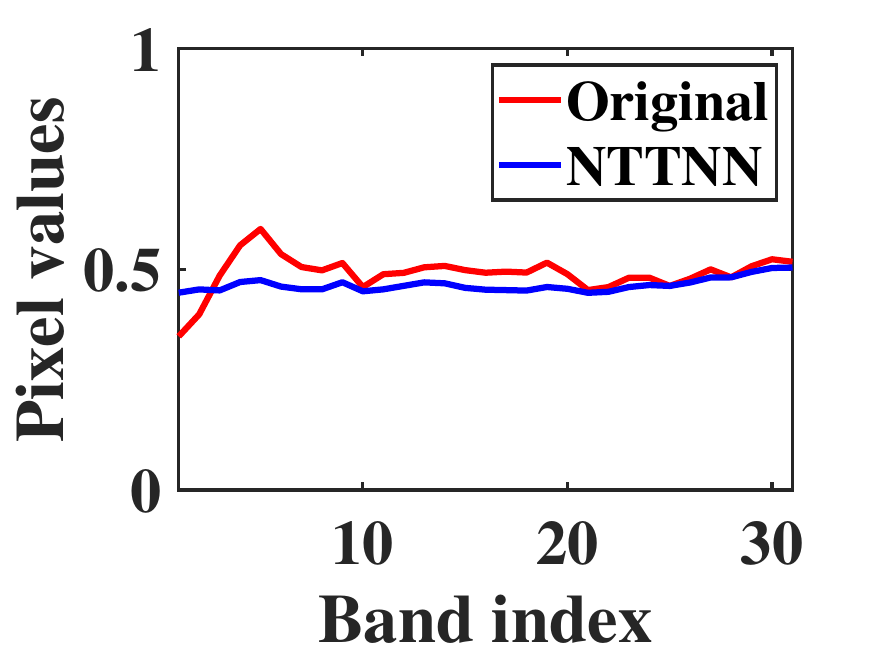}&
  \includegraphics[width=0.14\textwidth]{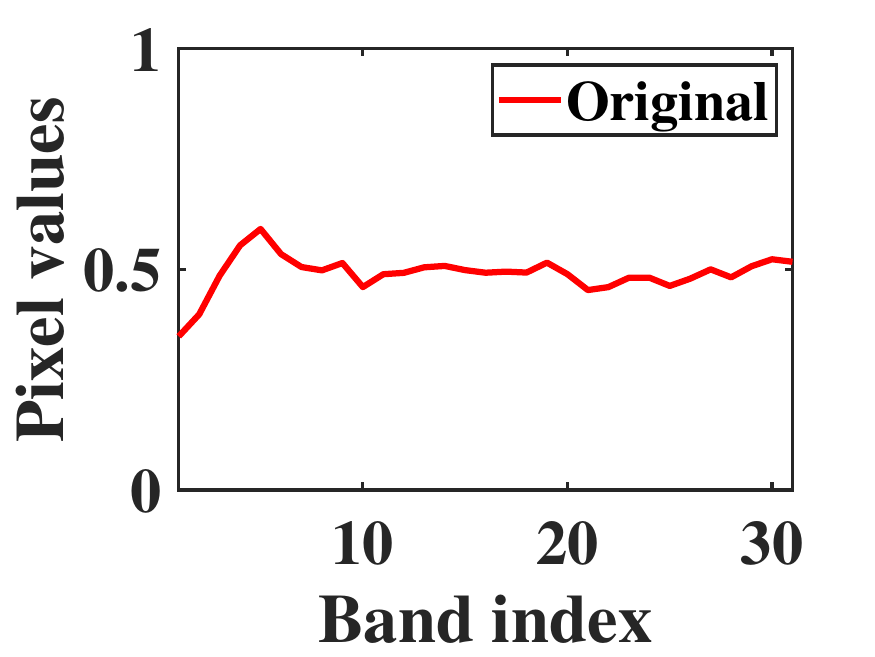}
  \\
  \includegraphics[width=0.14\textwidth]{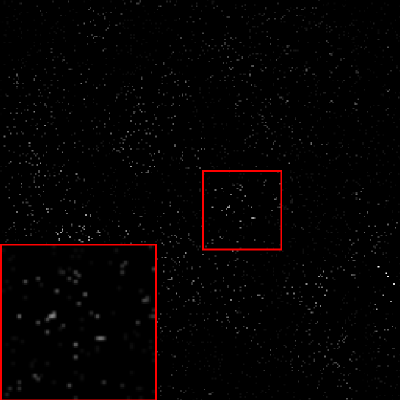}&
  \includegraphics[width=0.14\textwidth]{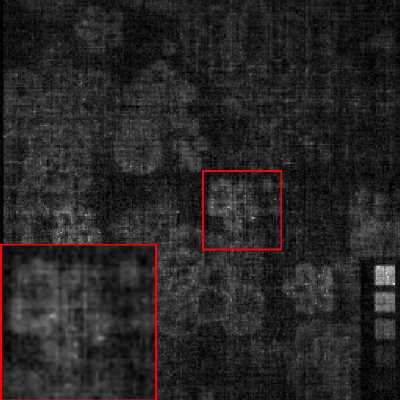}&
  \includegraphics[width=0.14\textwidth]{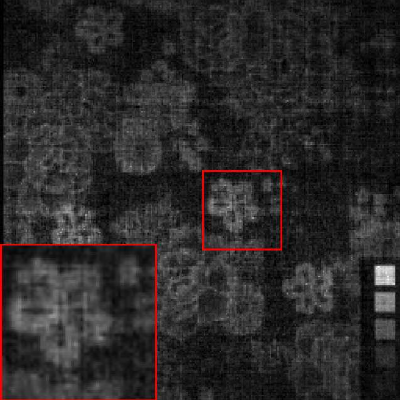}&
  \includegraphics[width=0.14\textwidth]{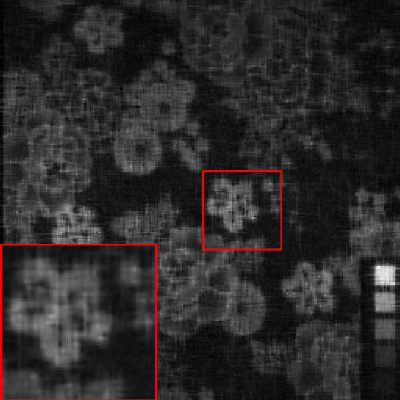}&
  \includegraphics[width=0.14\textwidth]{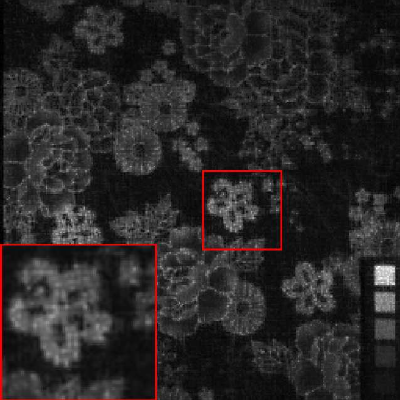}&
  \includegraphics[width=0.14\textwidth]{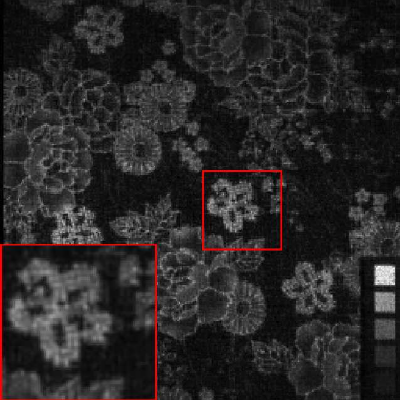}&
  \includegraphics[width=0.14\textwidth]{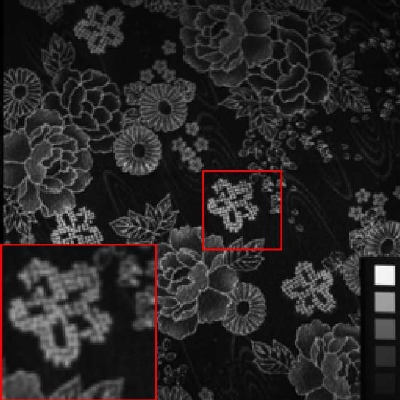}
  \\
  \includegraphics[width=0.14\textwidth]{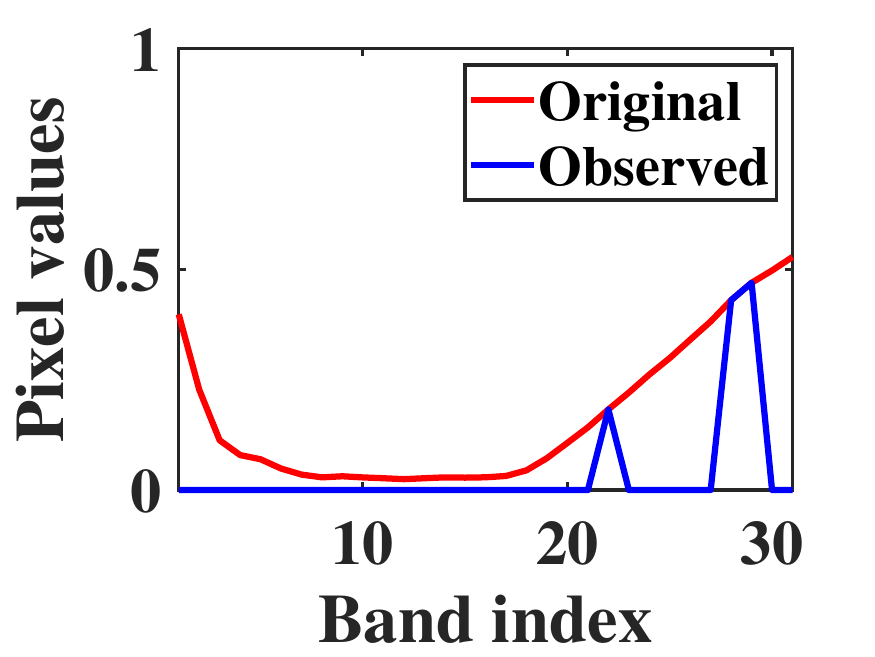}&
  \includegraphics[width=0.14\textwidth]{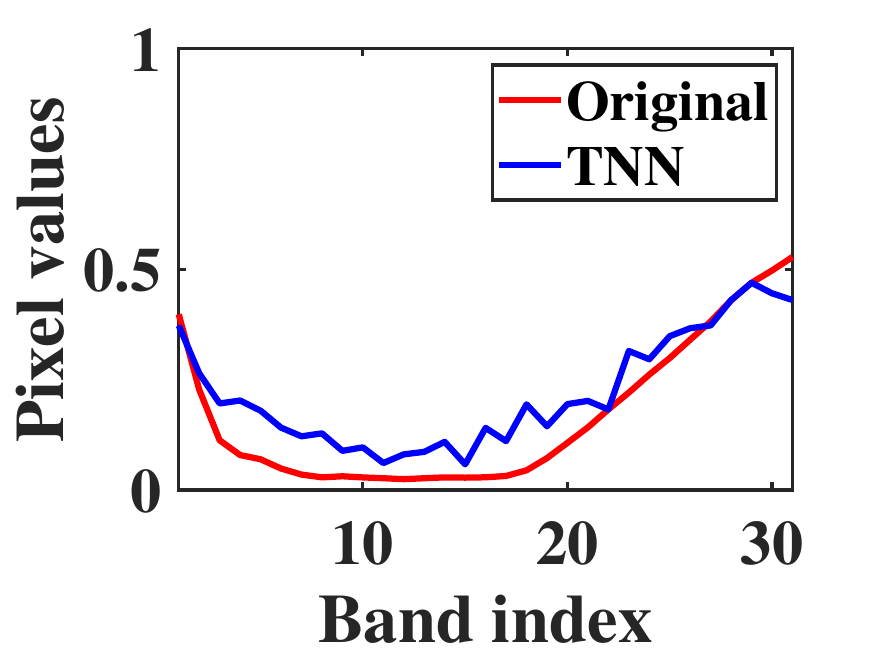}&
  \includegraphics[width=0.14\textwidth]{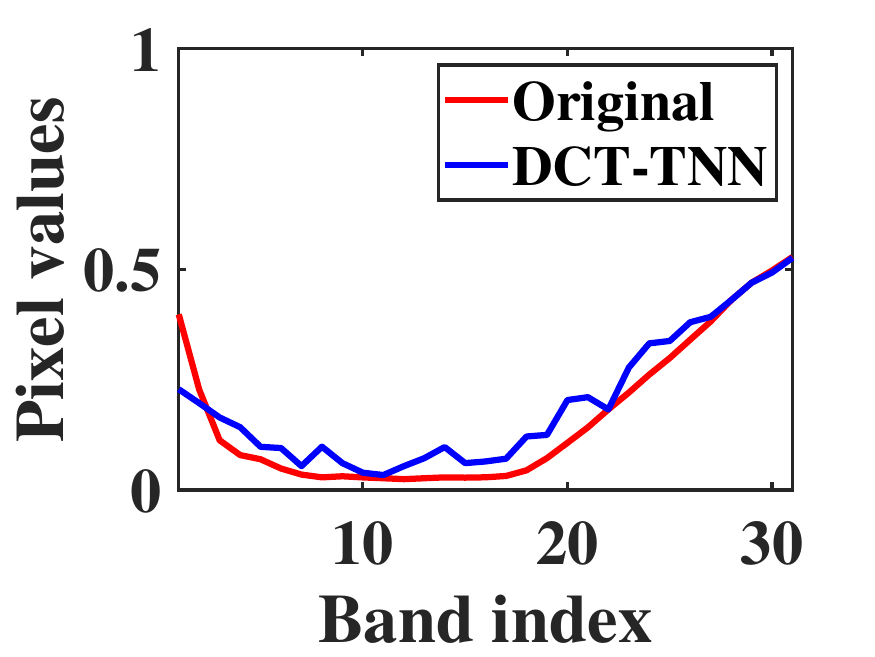}&
  \includegraphics[width=0.14\textwidth]{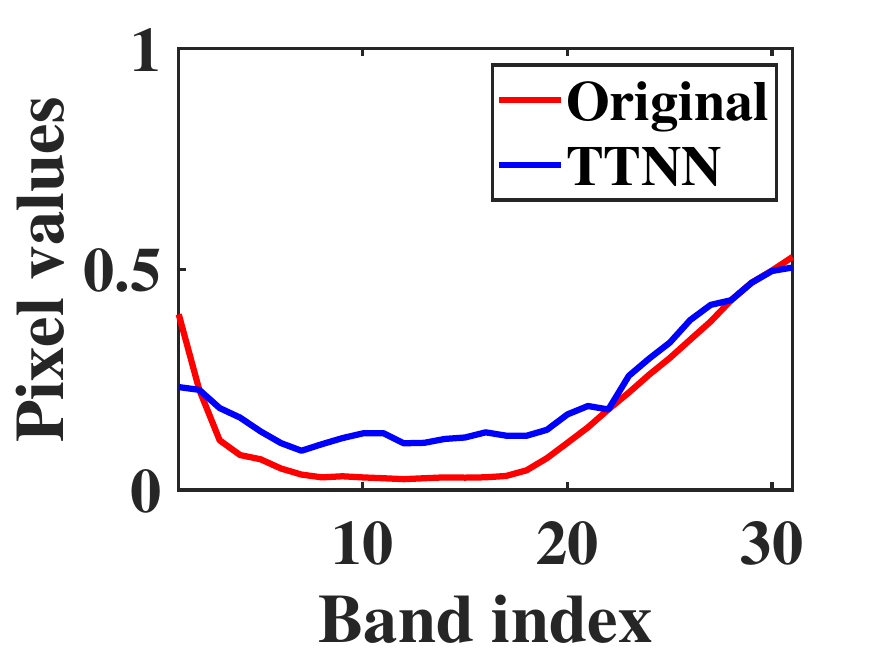}&
  \includegraphics[width=0.14\textwidth]{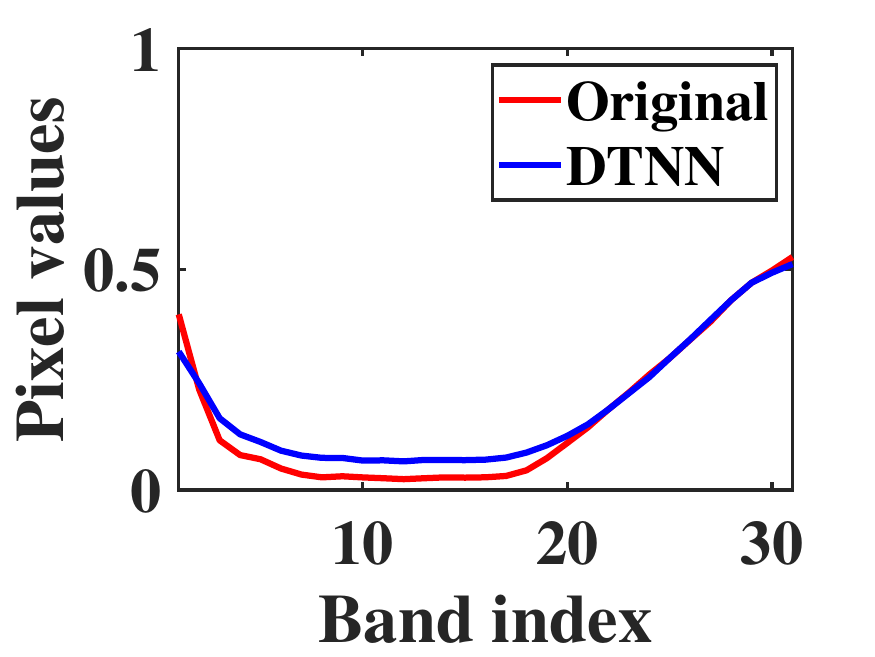}&
  \includegraphics[width=0.14\textwidth]{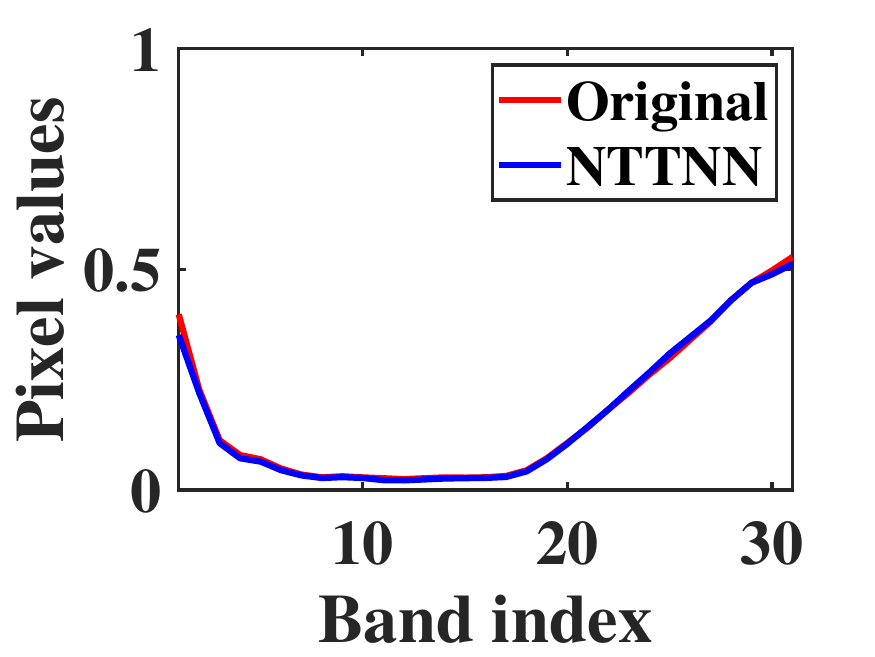}&
  \includegraphics[width=0.14\textwidth]{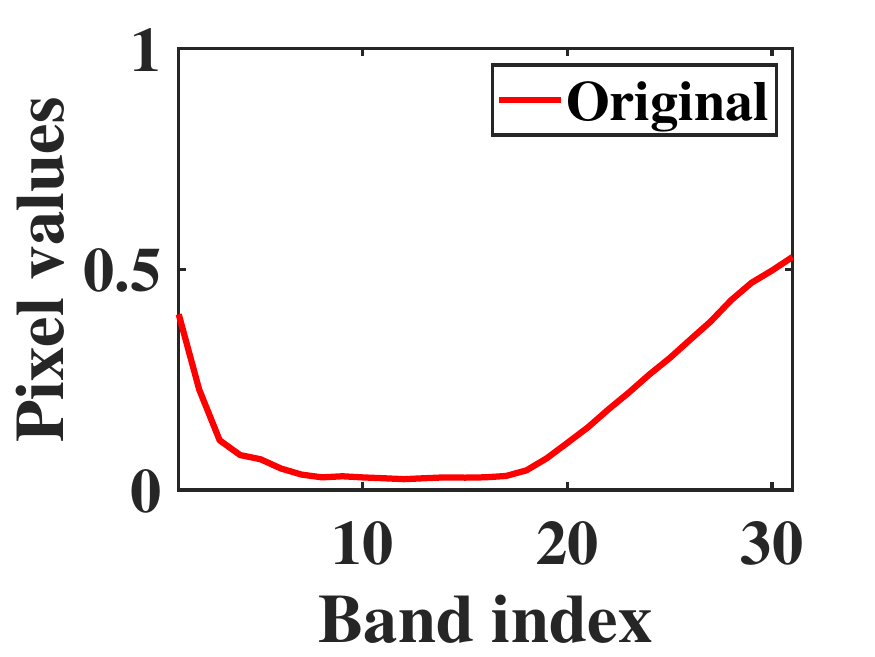}
  \\
  \includegraphics[width=0.14\textwidth]{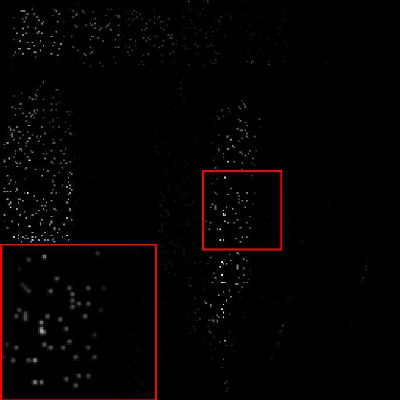}&
  \includegraphics[width=0.14\textwidth]{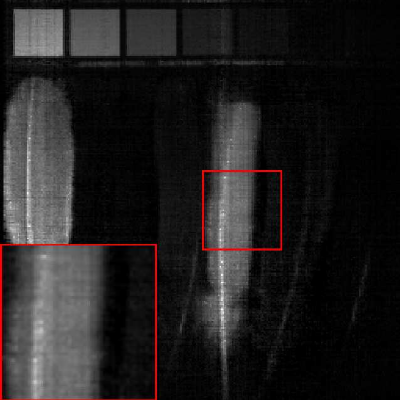}&
  \includegraphics[width=0.14\textwidth]{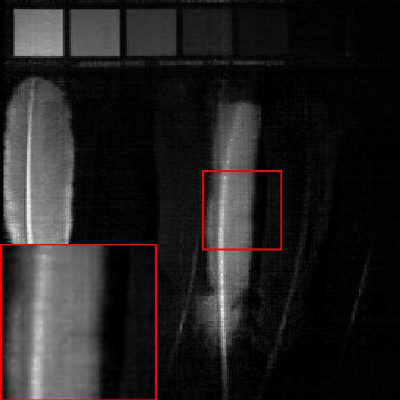}&
  \includegraphics[width=0.14\textwidth]{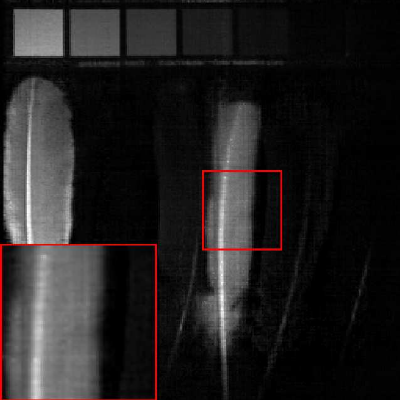}&
  \includegraphics[width=0.14\textwidth]{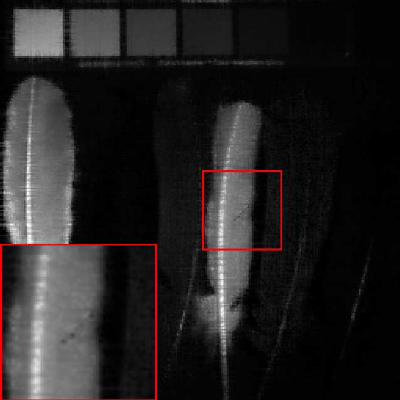}&
  \includegraphics[width=0.14\textwidth]{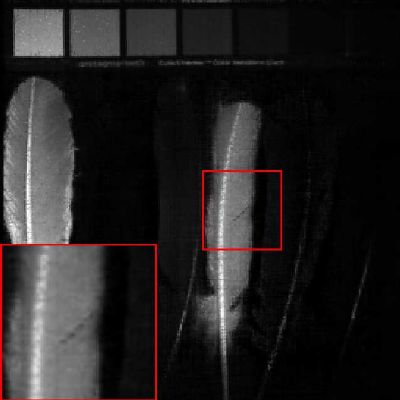}&
  \includegraphics[width=0.14\textwidth]{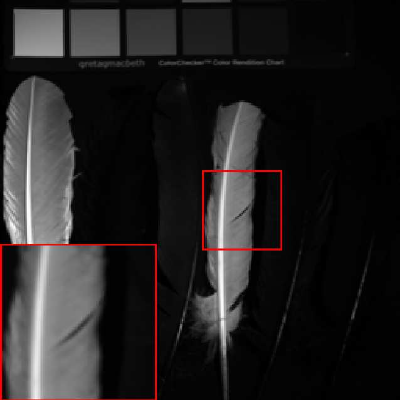}
  \\
  \includegraphics[width=0.14\textwidth]{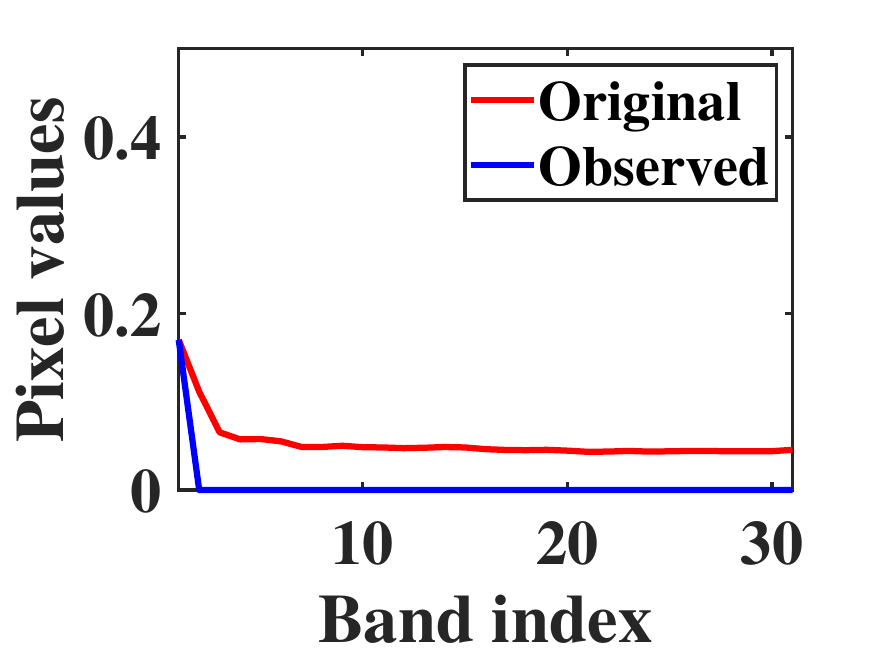}&
  \includegraphics[width=0.14\textwidth]{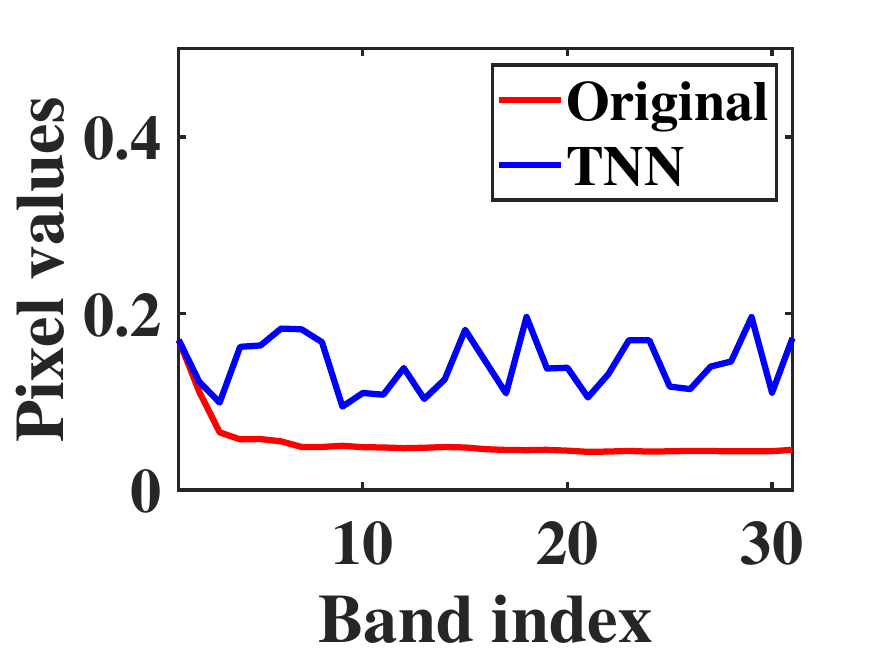}&
  \includegraphics[width=0.14\textwidth]{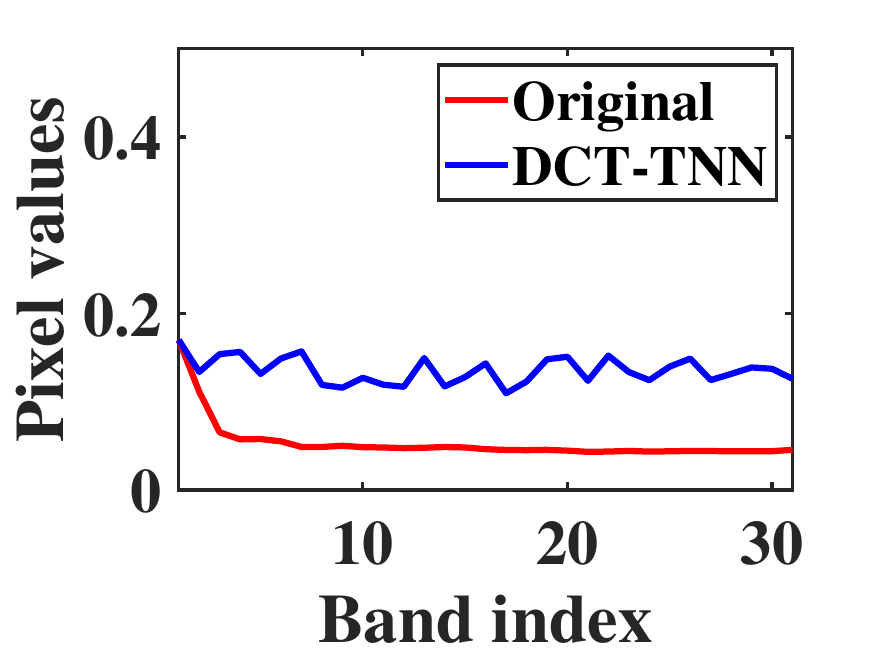}&
  \includegraphics[width=0.14\textwidth]{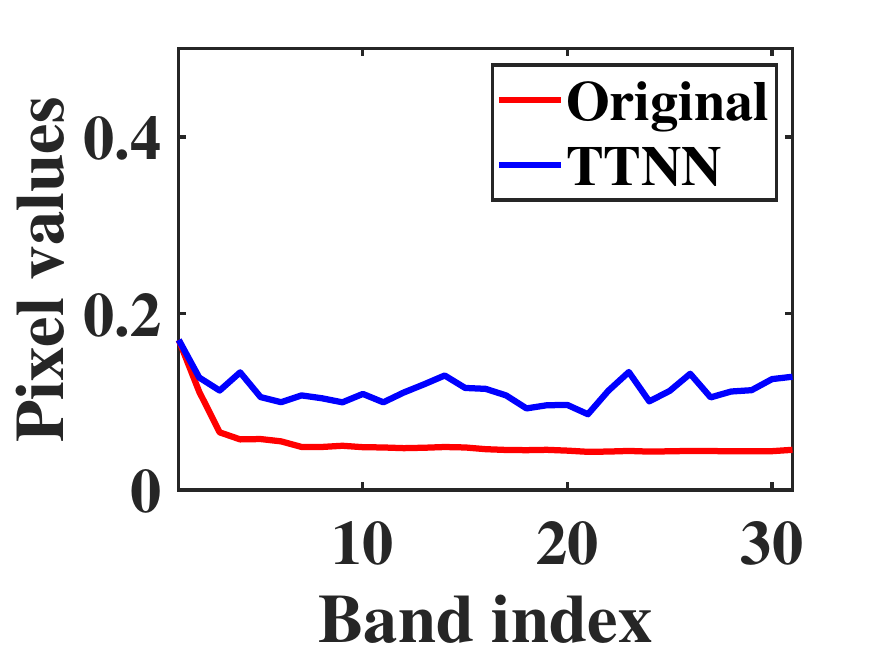}&
  \includegraphics[width=0.14\textwidth]{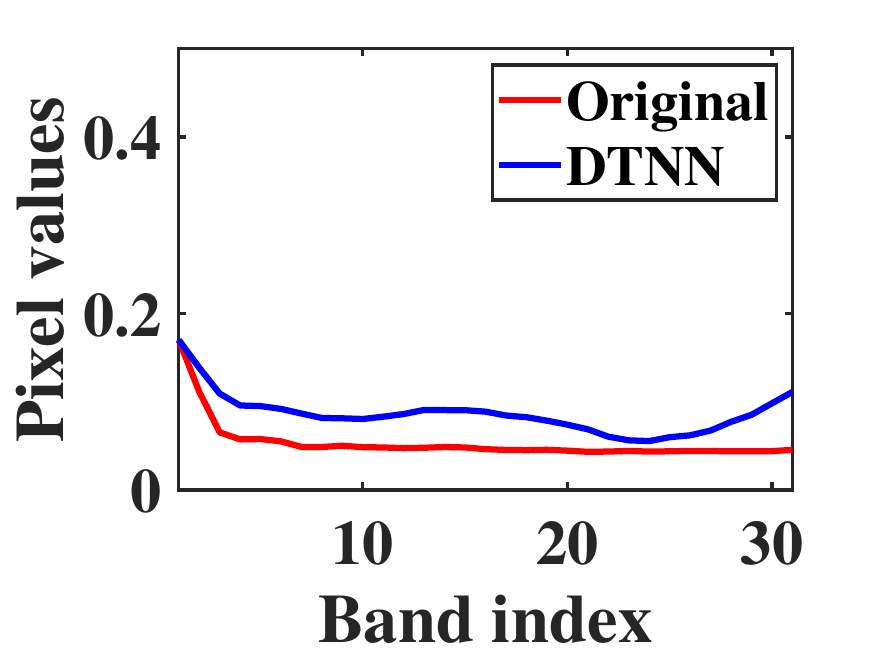}&
  \includegraphics[width=0.14\textwidth]{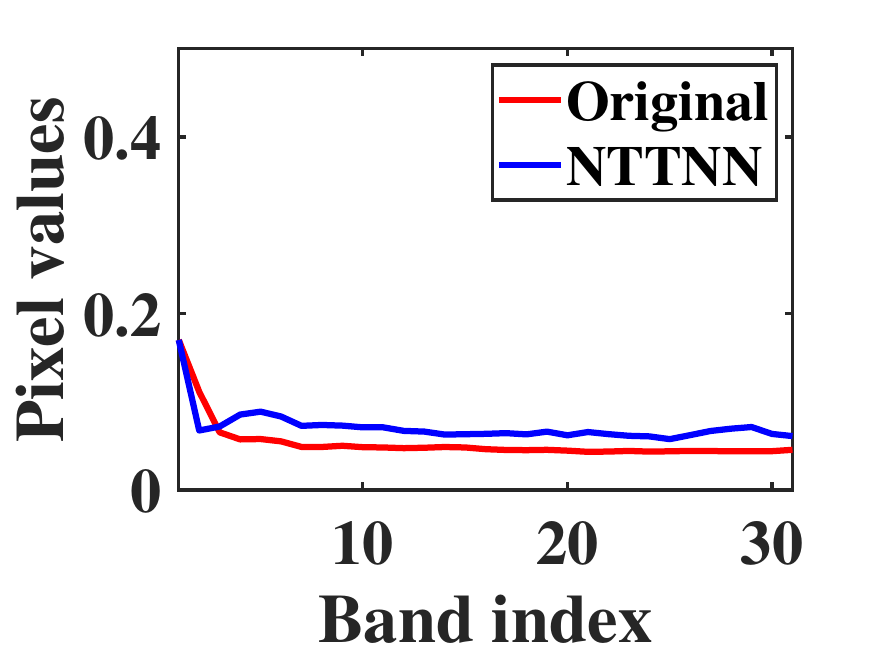}&
  \includegraphics[width=0.14\textwidth]{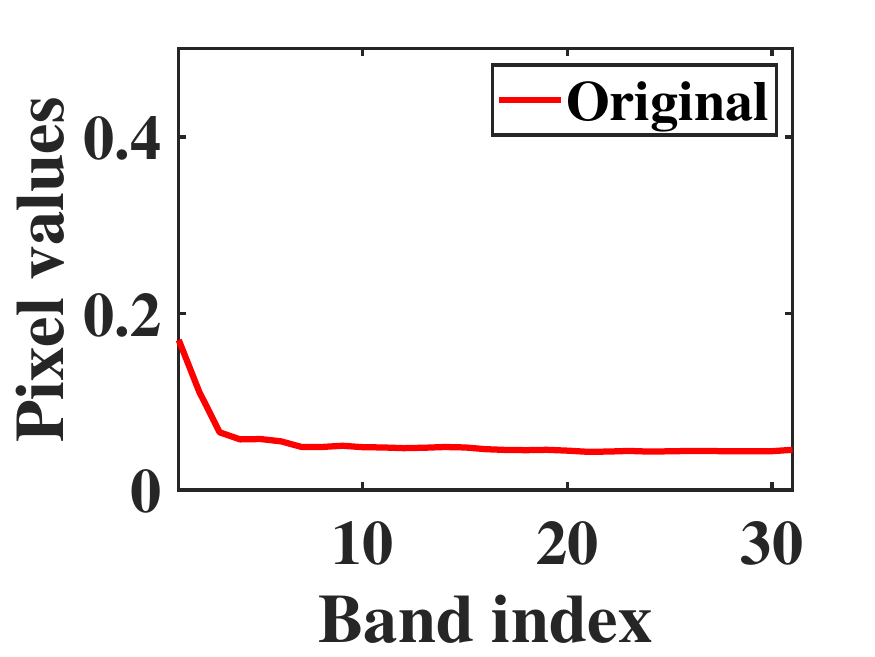}
  \\
   Observed & TNN & DCT-TNN & TTNN & DTNN  & NTTNN & Original
\end{tabular}
   \caption{The results of one band and spectrum profiles at one spatial location of MSIs by different methods for SR = 5\%.
From top to bottom: $\mathit{Balloons}$, $\mathit{Beads}$, $\mathit{Toy}$, $\mathit{Cloth}$, and $\mathit{Feathers}$, respectively.
From left to right: the observed data, the reconstructed results by TNN, DCT-TNN, TTNN,
DTNN, NTTNN, and the original data, respectively.
   }\label{MSIpic}
\end{figure*}
\subsection{Experiments on Videos Data}
In this part, we verify the effectiveness of the proposed NTTNN on three videos\footnote{\url{http://trace.eas.asu.edu/yuv/}.}: $\mathit{Carphone},\mathit{Hall},$ and $\mathit{News}$. All videos have been resized to $144\times{176}\times{100}$ in the experiments. The SRs are set to 5\%, 10\%, and 15\%, respectively.

Table \ref{videotable} shows the quantitative metrics of the recovered videos obtained by different methods for different SRs. We can observe that the proposed NTTNN clearly outperforms the other compared linear transform-based TNN methods for all SRs.
For visual comparisons, we show the recovered results of one band and one mode-3 tube of videos by different methods for SR = 5\% in Fig. \ref{videopic}. From Fig. \ref{videopic}, we can observe that NTTNN outperforms the compared methods in preserving details
and structures, e.g., the dancer in zoom-in regions of $\mathit{News}$. Moreover, NTTNN yields the closest spectral curves in all cases.
\begin{table}[htp!]
\centering
\renewcommand\arraystretch{1.5}
\setlength{\tabcolsep}{1.2mm}
\caption{The PSNR, SSIM, and SAM values of the recovered videos by different methods for different SRs.}\vspace{0.5cm}
\begin{tabular}{c|c|lll|lll|lll}
\Xhline{1.5pt}
\multirow{2}{*}{Data Index} & \multirow{2}{*}{methods} & \multicolumn{3}{c|}{SR=5\%} & \multicolumn{3}{c|}{SR=10\%} & \multicolumn{3}{c}{SR=15\%} \\ \cline{3-11}
                            &                          & PSNR    & SSIM    & SAM     & PSNR    & SSIM    & SAM     & PSNR     & SSIM    & SAM     \\ \Xhline{1.5pt}
\multirow{6}{*}{$\mathit{Carphone}$}   & \multicolumn{1}{c|}{Observed} &6.814&0.0143&1.3521&7.048&0.0231&1.2532&7.296&0.0311&1.1758\\
                            & \multicolumn{1}{c|}{TNN}                      &25.122&0.7222&0.1138&27.229&0.7909&0.0942&28.637&0.8309&0.0825\\
                            & \multicolumn{1}{c|}{DCT-TNN}                  &25.516&0.7395&0.1070&27.633&0.8074&0.0881&29.049&0.8459&0.0769\\
                            & \multicolumn{1}{c|}{TTNN}                     &26.597&0.8162&0.0852&28.708&0.8666&0.0724&30.146&0.8932&0.0644\\
                            & \multicolumn{1}{c|}{DTNN}                     &26.941&0.8328&0.0833&29.178&0.8756&0.0694&30.635&0.8983&0.0604\\
                            & \multicolumn{1}{c|}{NTTNN}                    &\textBF{27.460}&\textBF{0.8355}&\textBF{0.0814}&\textBF{29.614}
                            &\textBF{0.8833}&\textBF{0.0676}&\textBF{31.059}&\textBF{0.9097}
                            &\textBF{0.0588}\\ \hline
\multirow{6}{*}{$\mathit{Hall}$} & \multicolumn{1}{c|}{Observed}            &4.835&0.0071&1.3516&5.070&0.0123&1.2529&5.319&0.0179&1.1757\\
                            & \multicolumn{1}{c|}{TNN}                      &28.033&0.9010&0.0434&30.868&0.9387&0.0350&32.436&0.9522&0.0309\\
                            & \multicolumn{1}{c|}{DCT-TNN}                  &28.042&0.9034&0.0434&30.842&0.9380&0.0352&32.369&0.9510&0.0311\\
                            & \multicolumn{1}{c|}{TTNN}                     &28.781&0.9163&0.0413&31.283&0.9431&0.0343&32.824&0.9548&0.0303\\
                            & \multicolumn{1}{c|}{DTNN}                     &27.765&0.9040&0.0429&31.722&0.9496&0.0343&33.814&0.9629&0.0295\\
                            & \multicolumn{1}{c|}{NTTNN}                    &\textBF{30.140}&\textBF{0.9410}&\textBF{0.0352}&\textBF{32.713}
                            &\textBF{0.9595}&\textBF{0.0298}&\textBF{34.261}&\textBF{0.9688}
                            &\textBF{0.0264}   \\ \hline
\multirow{6}{*}{$\mathit{News}$} & \multicolumn{1}{c|}{Observed}            &8.991&0.0207&1.3516&9.227&0.0347&1.2528&9.476&0.0477&1.1756\\
                            & \multicolumn{1}{c|}{TNN}                      &26.715&0.8208&0.1002&29.085&0.8810&0.0814&30.762&0.9119&0.0691\\
                            & \multicolumn{1}{c|}{DCT-TNN}                  &27.069&0.8335&0.0944&29.501&0.8923&0.0756&31.206&0.9222&0.0636\\
                            & \multicolumn{1}{c|}{TTNN}                     &27.674&0.8547&0.0881&30.011&0.9054&0.0705&31.708&0.9319&0.0592\\
                            & \multicolumn{1}{c|}{DTNN}                     &26.277&0.8605&0.0848&29.844&0.9300&0.0646&32.384&0.9479&0.0546\\
                            & \multicolumn{1}{c|}{NTTNN}                    &\textBF{28.195}&\textBF{0.8896}&\textBF{0.0727}&\textBF{31.028}
                            &\textBF{0.9309}&\textBF{0.0574}&\textBF{32.912}&\textBF{0.9512}
                            &\textBF{0.0479}\\ \Xhline{1.5pt}

\end{tabular}\label{videotable}
\end{table}
\begin{figure*}[htp!]
\vspace{-0.9cm}
  \footnotesize
  \setlength{\tabcolsep}{1.2pt}
  \centering
  \begin{tabular}{ccccccc}
  \includegraphics[width=0.13\textwidth]{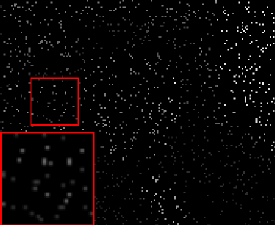}&
  \includegraphics[width=0.13\textwidth]{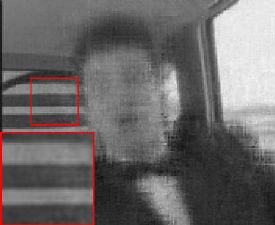}&
  \includegraphics[width=0.13\textwidth]{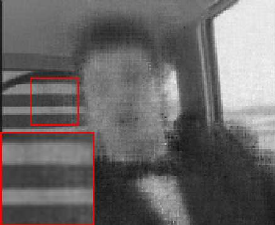}&
  \includegraphics[width=0.13\textwidth]{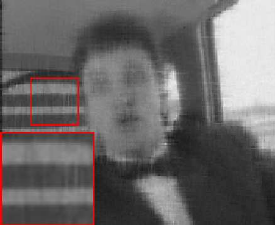}&
  \includegraphics[width=0.13\textwidth]{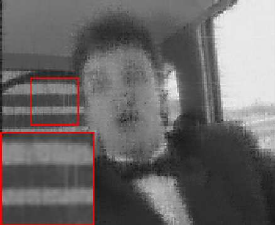}&
  \includegraphics[width=0.13\textwidth]{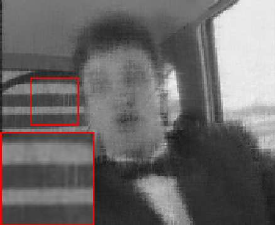}&
  \includegraphics[width=0.13\textwidth]{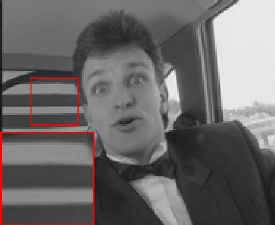}
  \\
  \includegraphics[width=0.13\textwidth]{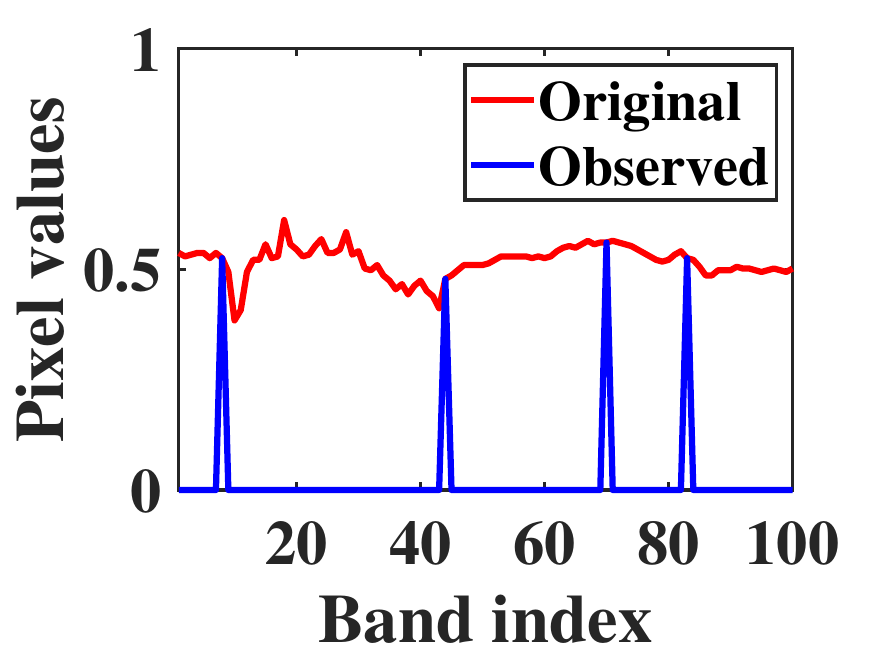}&
  \includegraphics[width=0.13\textwidth]{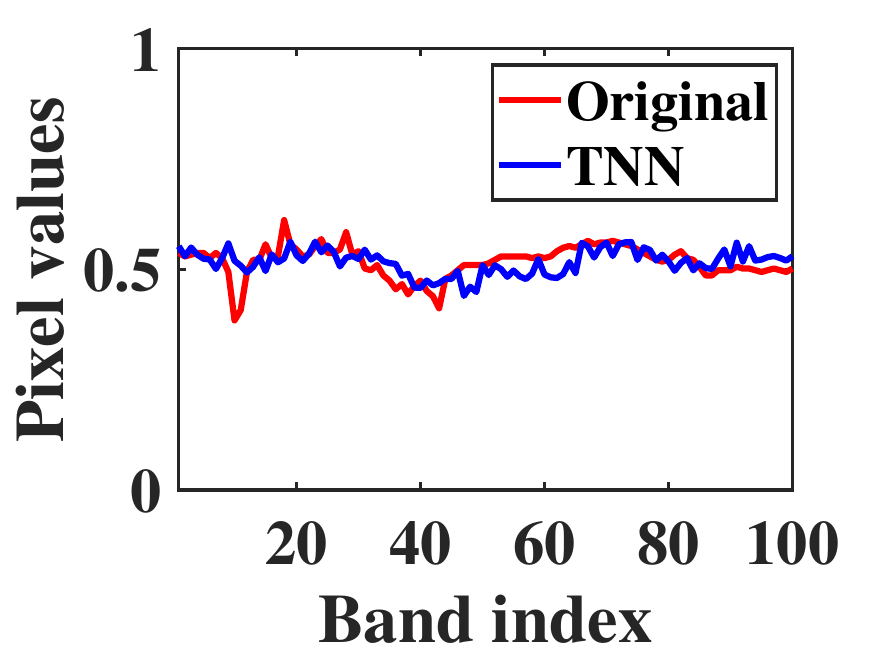}&
  \includegraphics[width=0.13\textwidth]{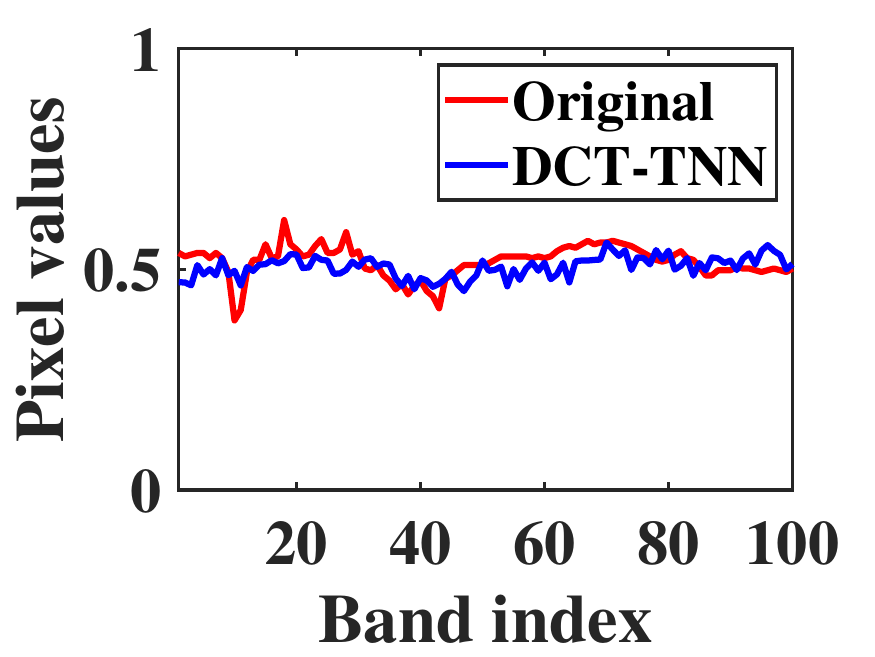}&
  \includegraphics[width=0.13\textwidth]{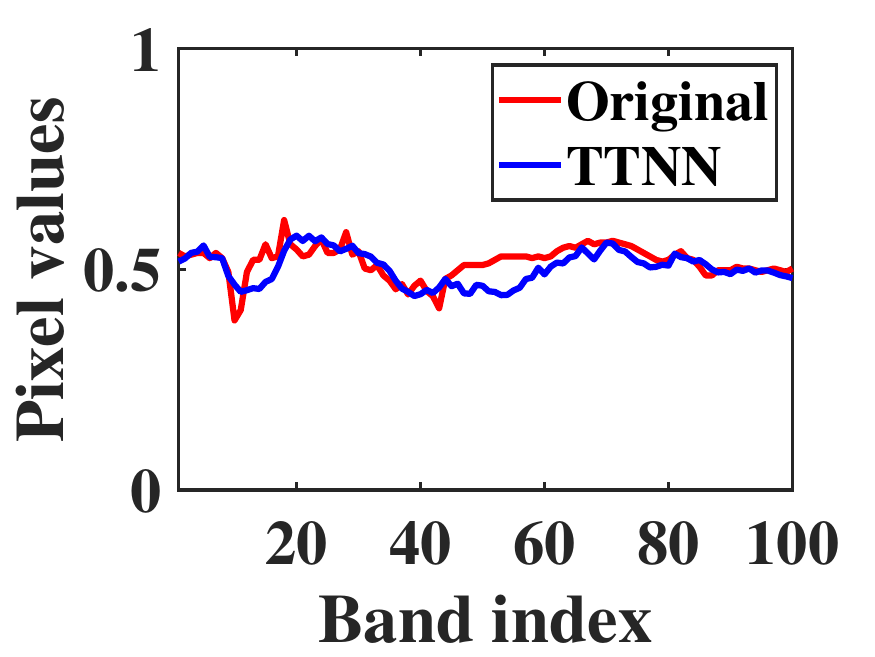}&
  \includegraphics[width=0.13\textwidth]{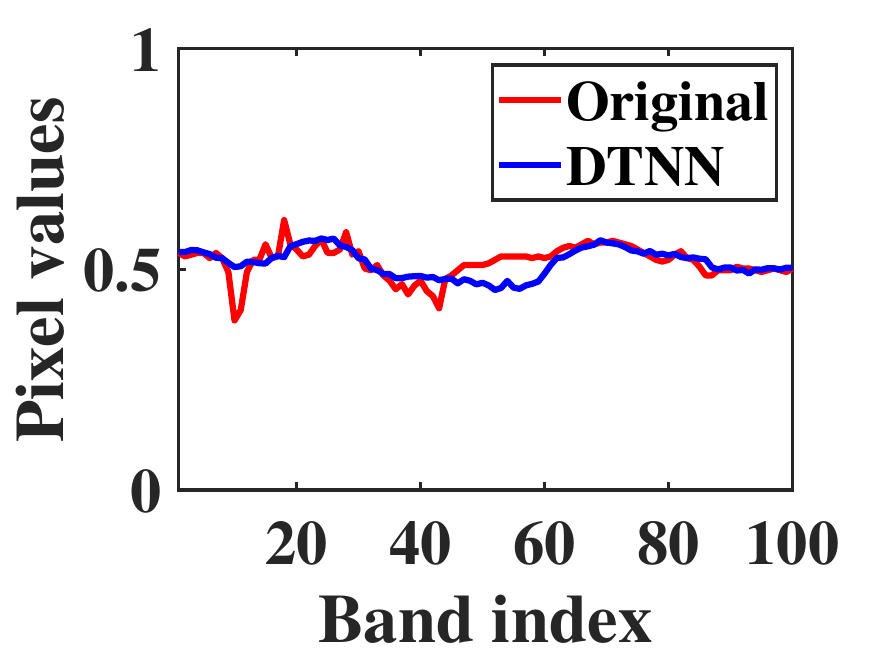}&
  \includegraphics[width=0.13\textwidth]{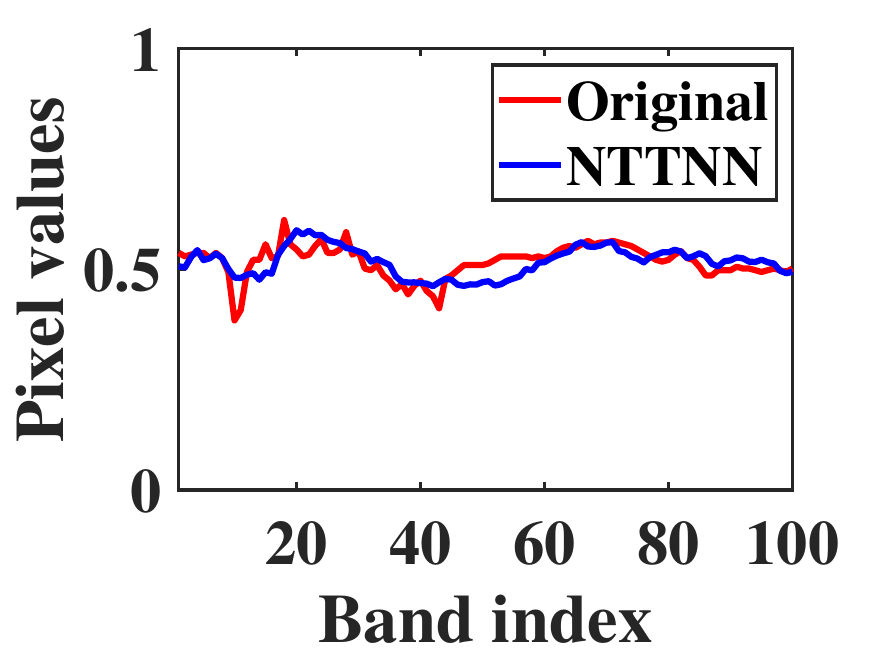}&
  \includegraphics[width=0.13\textwidth]{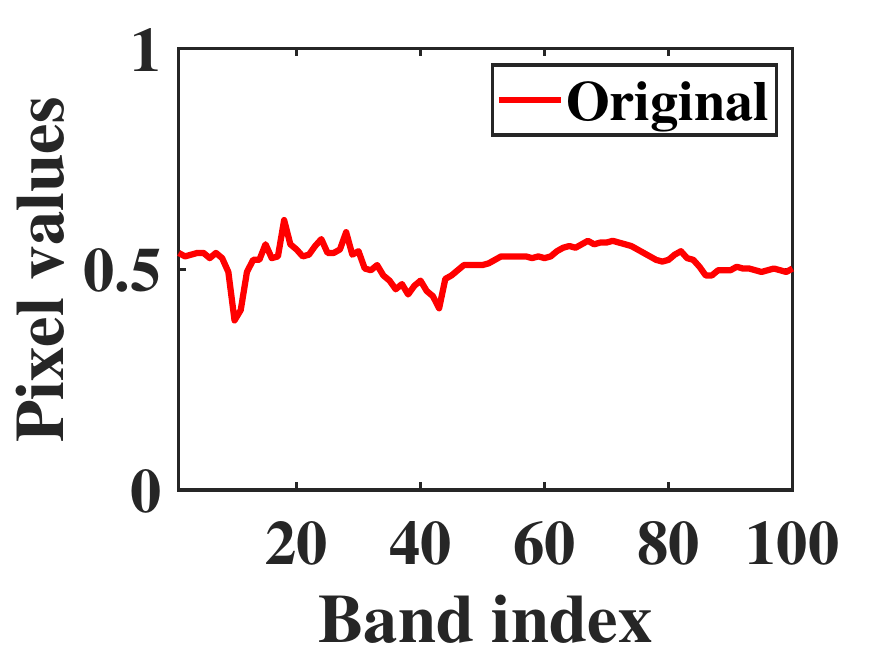}
  \\
  \includegraphics[width=0.13\textwidth]{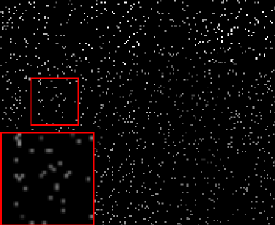}&
  \includegraphics[width=0.13\textwidth]{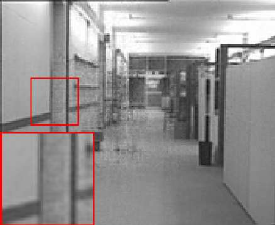}&
  \includegraphics[width=0.13\textwidth]{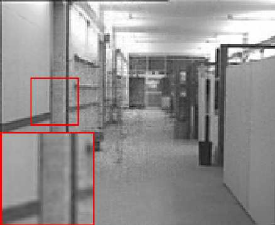}&
  \includegraphics[width=0.13\textwidth]{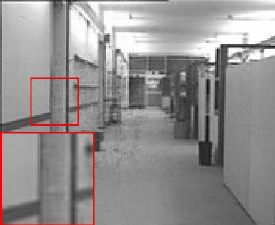}&
  \includegraphics[width=0.13\textwidth]{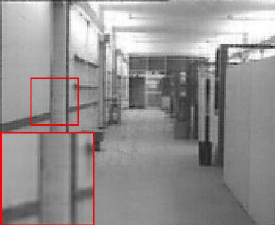}&
  \includegraphics[width=0.13\textwidth]{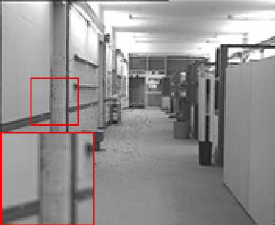}&
  \includegraphics[width=0.13\textwidth]{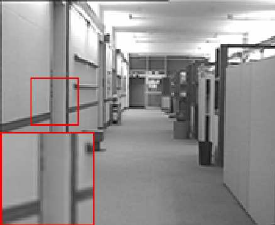}
  \\
  \includegraphics[width=0.13\textwidth]{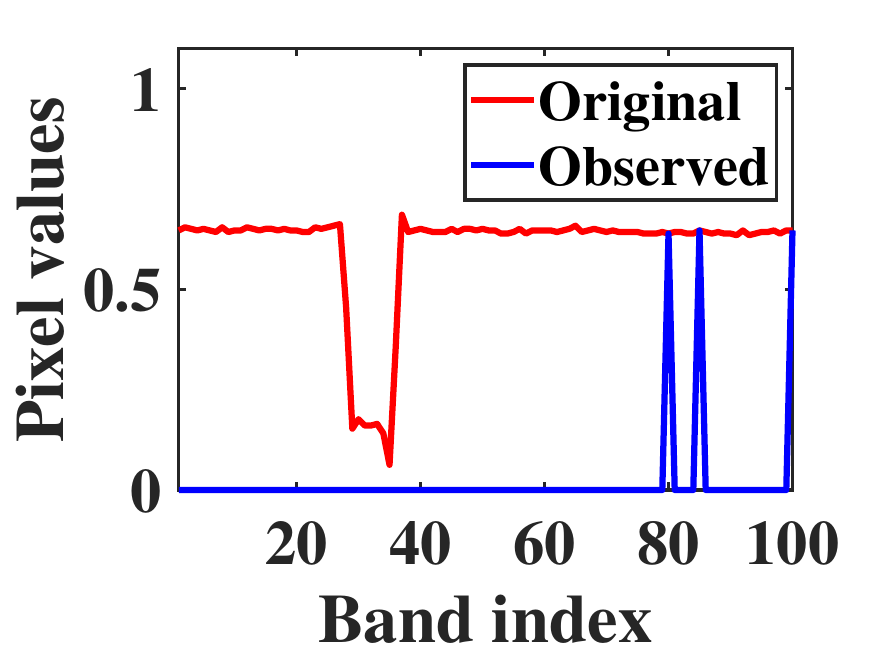}&
  \includegraphics[width=0.13\textwidth]{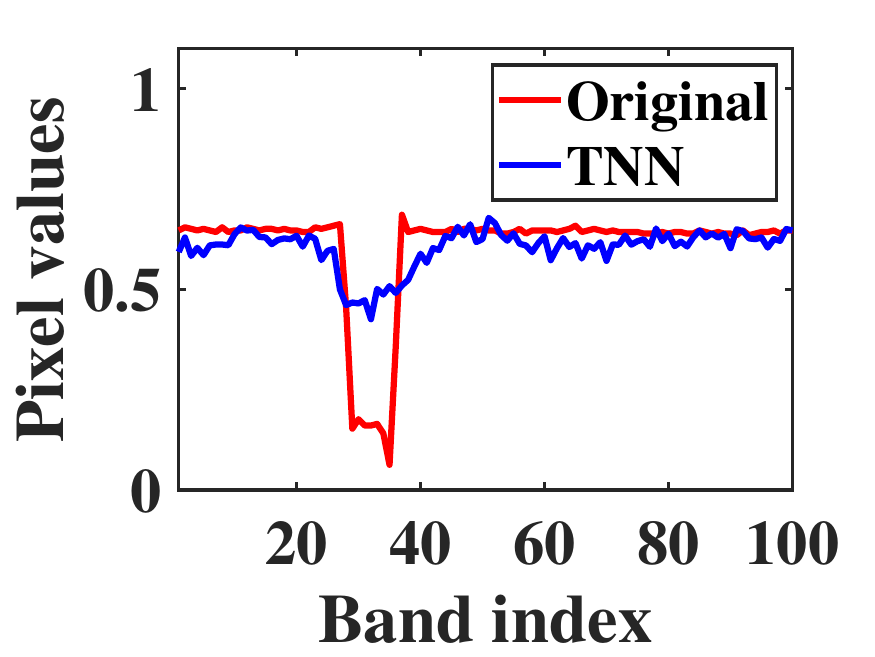}&
  \includegraphics[width=0.13\textwidth]{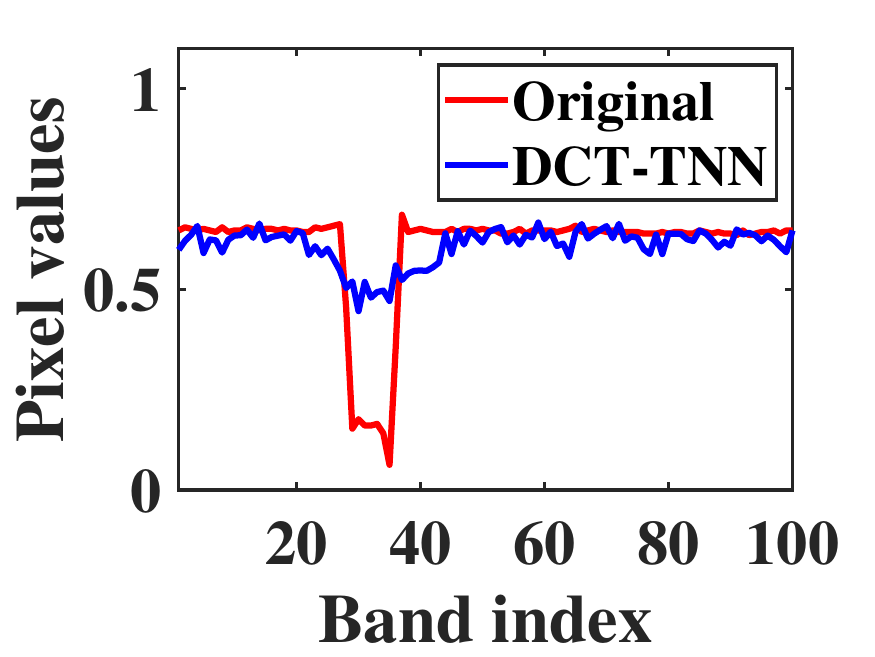}&
  \includegraphics[width=0.13\textwidth]{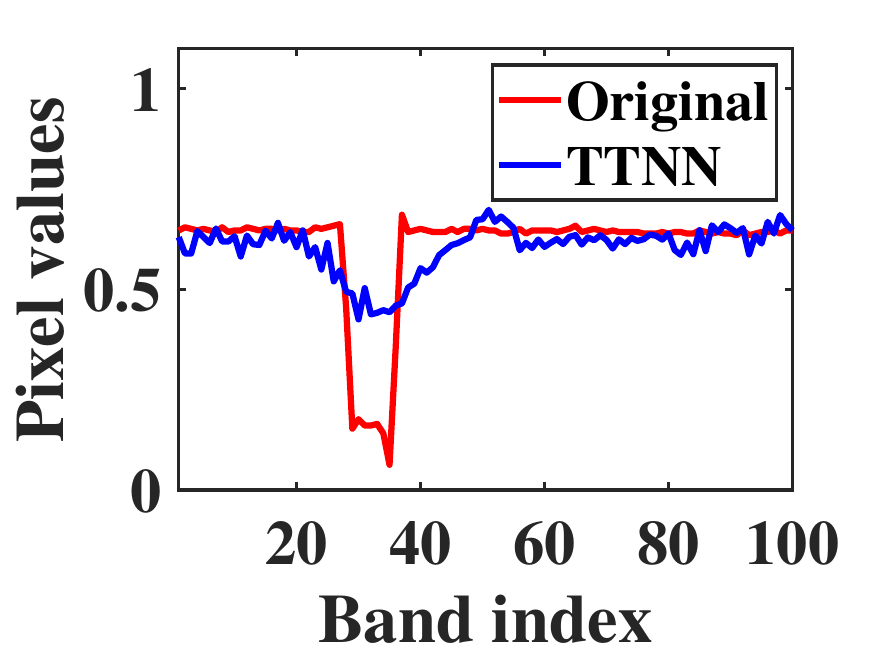}&
  \includegraphics[width=0.13\textwidth]{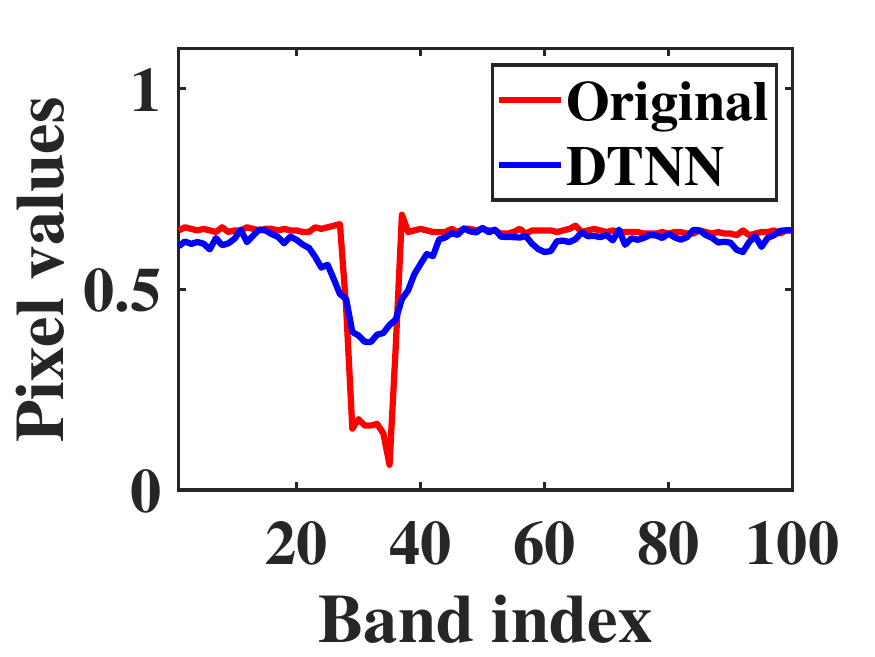}&
  \includegraphics[width=0.13\textwidth]{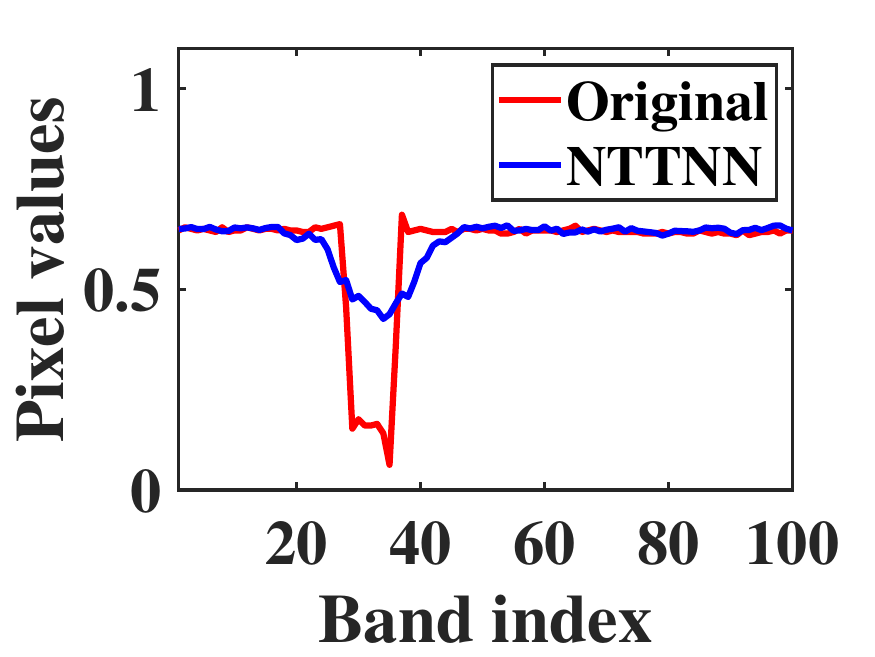}&
  \includegraphics[width=0.13\textwidth]{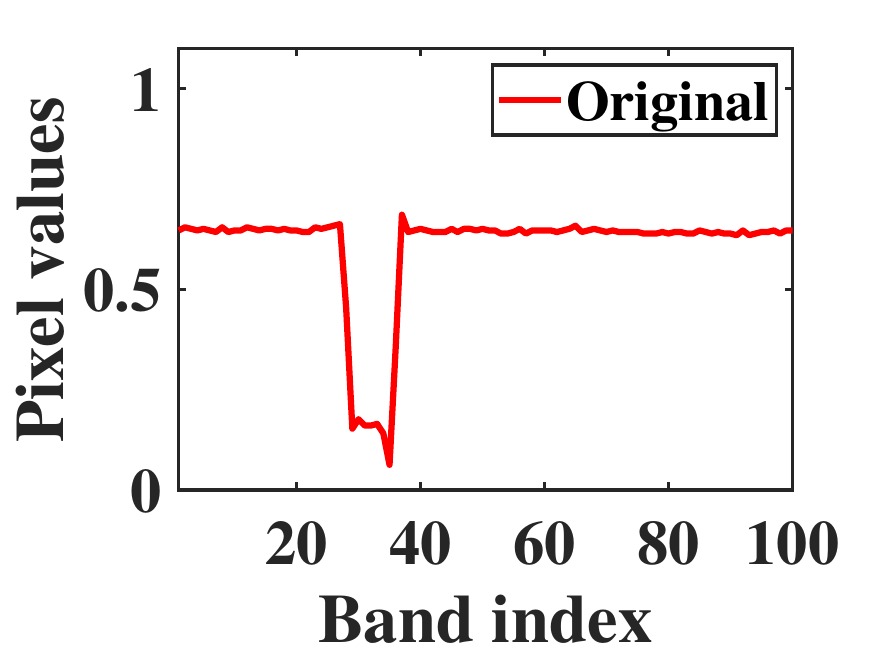}
  \\
  \includegraphics[width=0.13\textwidth]{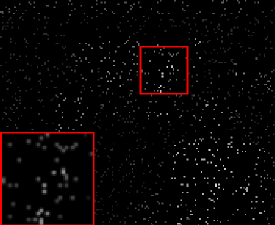}&
  \includegraphics[width=0.13\textwidth]{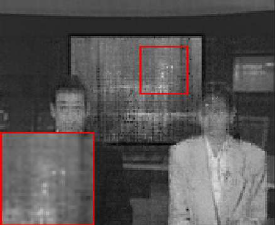}&
  \includegraphics[width=0.13\textwidth]{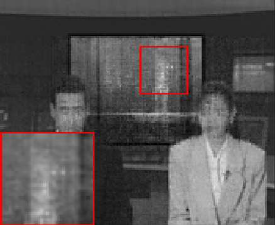}&
  \includegraphics[width=0.13\textwidth]{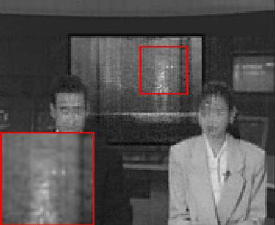}&
  \includegraphics[width=0.13\textwidth]{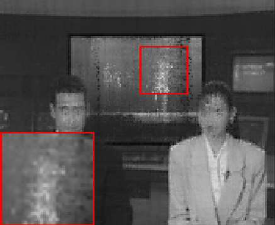}&
  \includegraphics[width=0.13\textwidth]{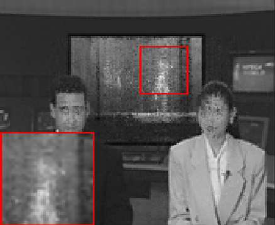}&
  \includegraphics[width=0.13\textwidth]{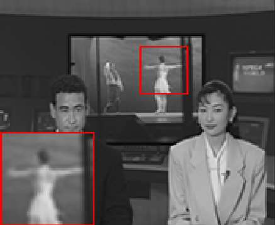}
  \\
  \includegraphics[width=0.13\textwidth]{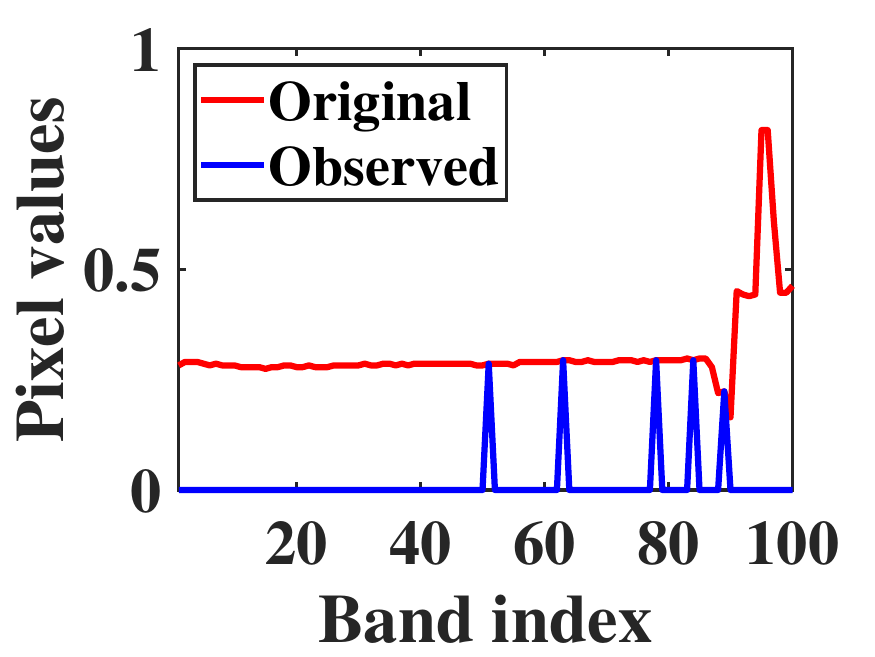}&
  \includegraphics[width=0.13\textwidth]{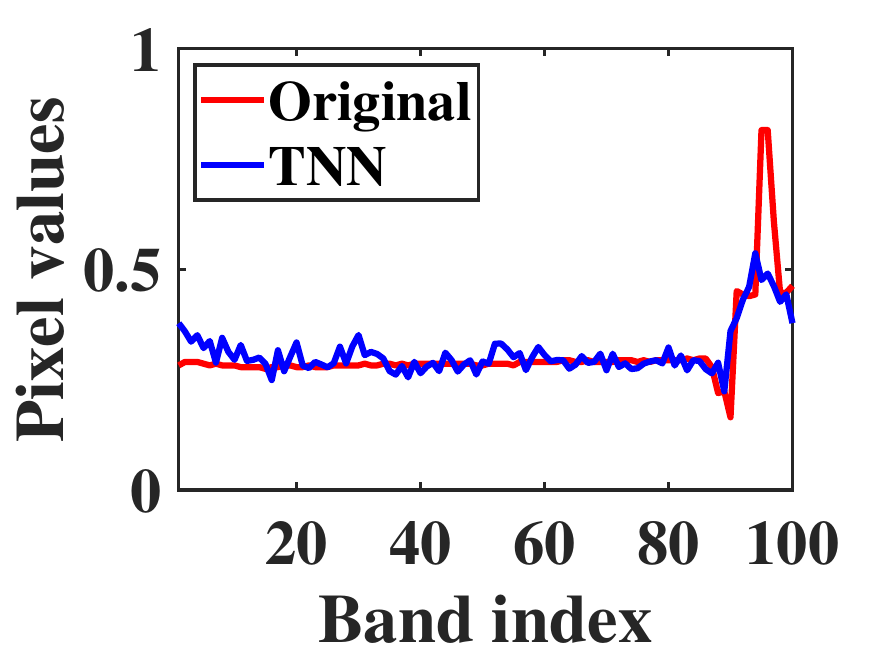}&
  \includegraphics[width=0.13\textwidth]{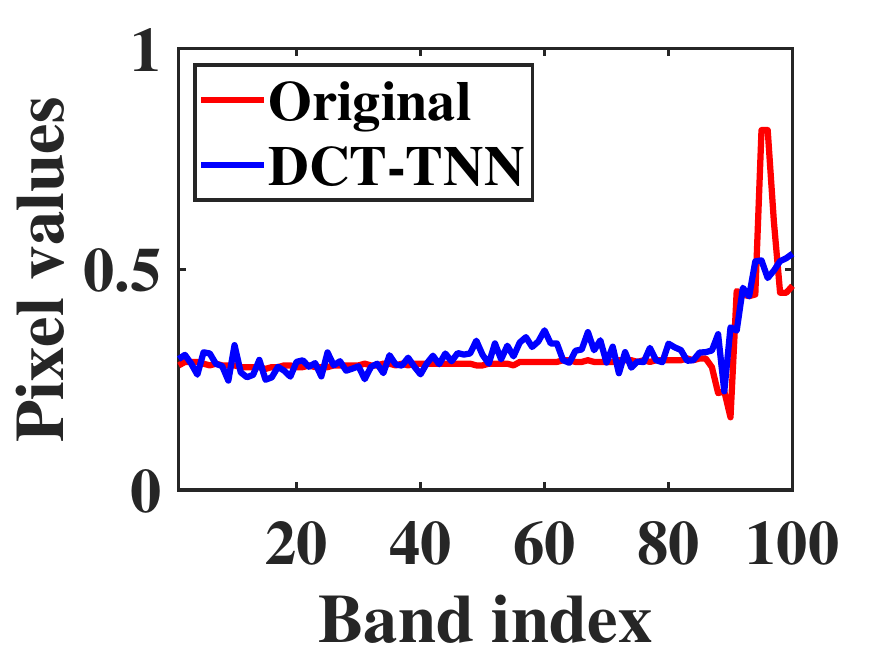}&
  \includegraphics[width=0.13\textwidth]{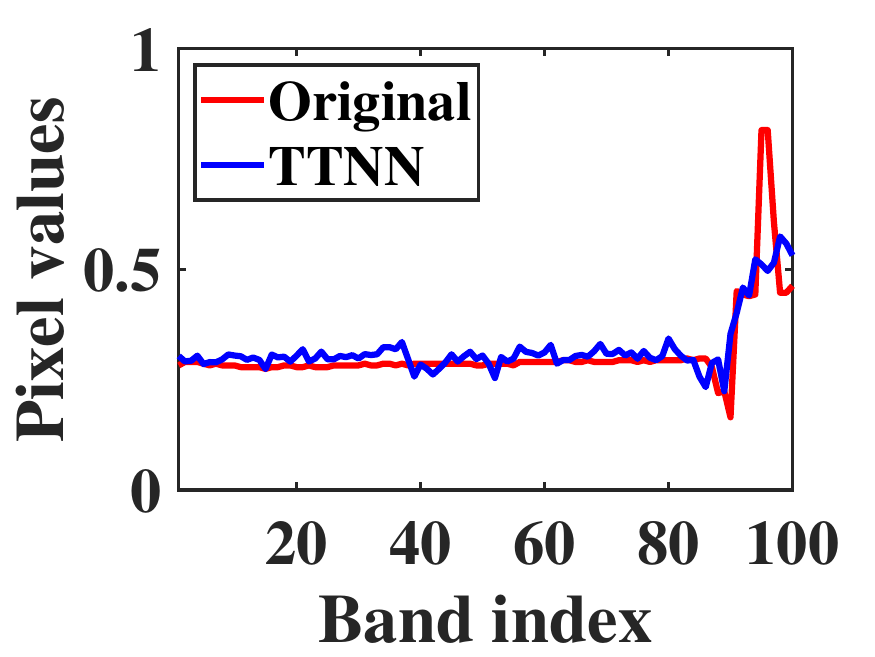}&
  \includegraphics[width=0.13\textwidth]{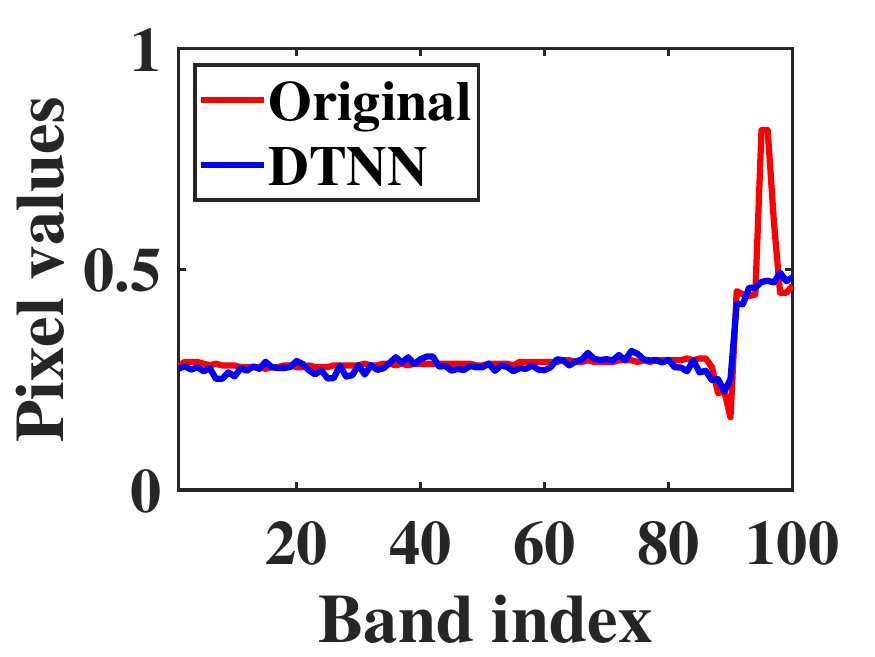}&
  \includegraphics[width=0.13\textwidth]{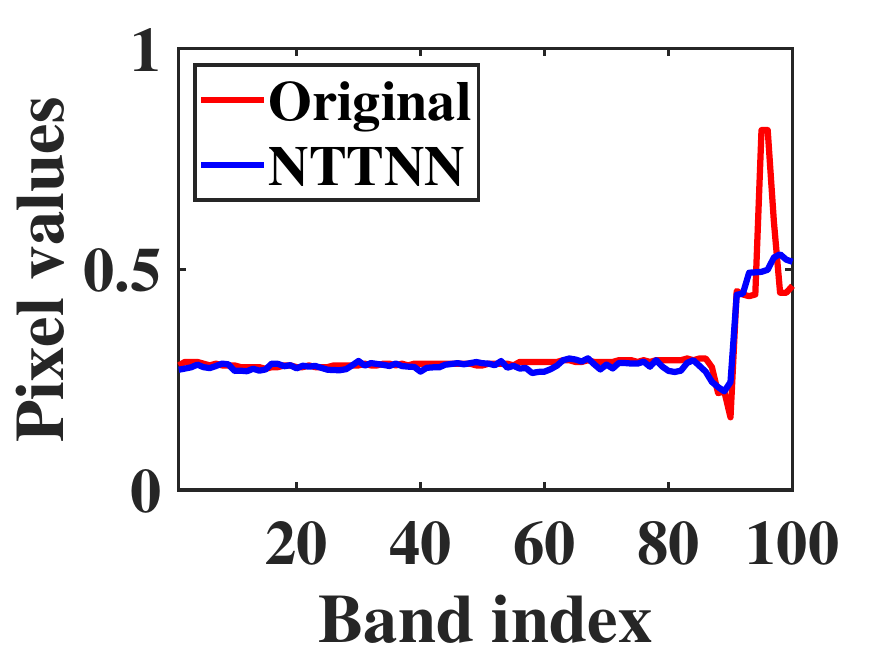}&
  \includegraphics[width=0.13\textwidth]{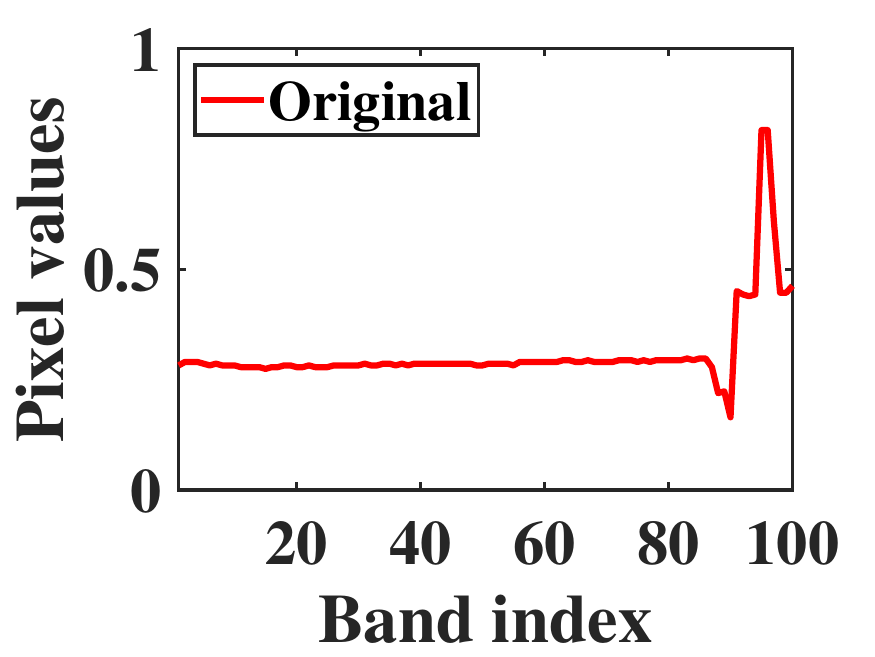}
  \\
  Observed & TNN & DCT-TNN & TTNN & DTNN  & NTTNN &  Original
\end{tabular}
   \caption{The results of one band and spectral curves at one spatial location of videos by different methods for SR = 5\%.
From top to bottom: $\mathit{Carphone}$, $\mathit{Hall}$, and $\mathit{News}$, respectively.
From left to right:  the observed data, the reconstructed results by TNN, DCT-TNN, TTNN,
DTNN, NTTNN, and the original data, respectively.
   }\label{videopic}\vspace{-0cm}
\end{figure*}
\vspace{-0cm}
\section{Discussion}\label{section5}
\subsection{Analysis of row
 number $r$ of $\mathbf{T}$}
In this subsection, we discuss the influence of row number $r$ of $\mathbf{T}$ on MSI $\mathit{Toy}$ with SR=5\%. From the Fig. \ref{disr}(a), we can observe that the energy of singular values of recovered result of NTTNN with less row number $r$ is more concentrated, which implies that the recovered result of NTTNN with less row number $r$ is more low-rank. Furthermore, Fig. \ref{disr}(b) plots PSNR and SSIM values of recovered MSIs by NTTNN with different row number $r$ of $\mathbf{T}$. From the Fig. \ref{disr}(b), we can observed that NTTNN with $r=5$ obtain the best recovered result in terms of PSNR and SSIM values. Therefore, in all experiments, the row number $r$ of $\mathbf{T}$ is selected from \{3,4,5,6,7,8,9,10\}, which is much less than $n_3$.
\begin{figure}[htp!]
  \setlength{\tabcolsep}{0.01cm}
  \centering
  \begin{tabular}{cc}
  \includegraphics[width=0.38\textwidth]{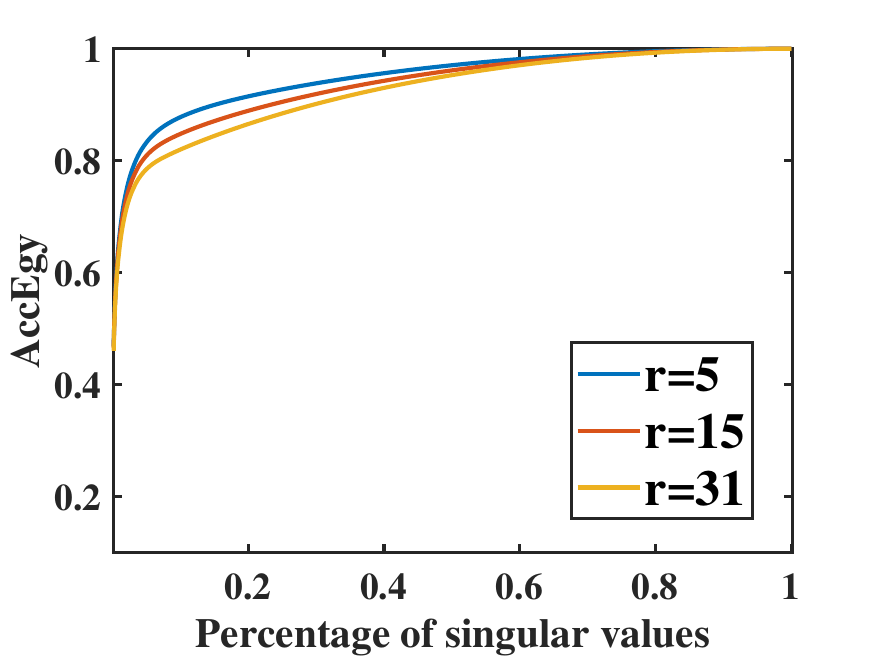}&
  \includegraphics[width=0.38\textwidth]{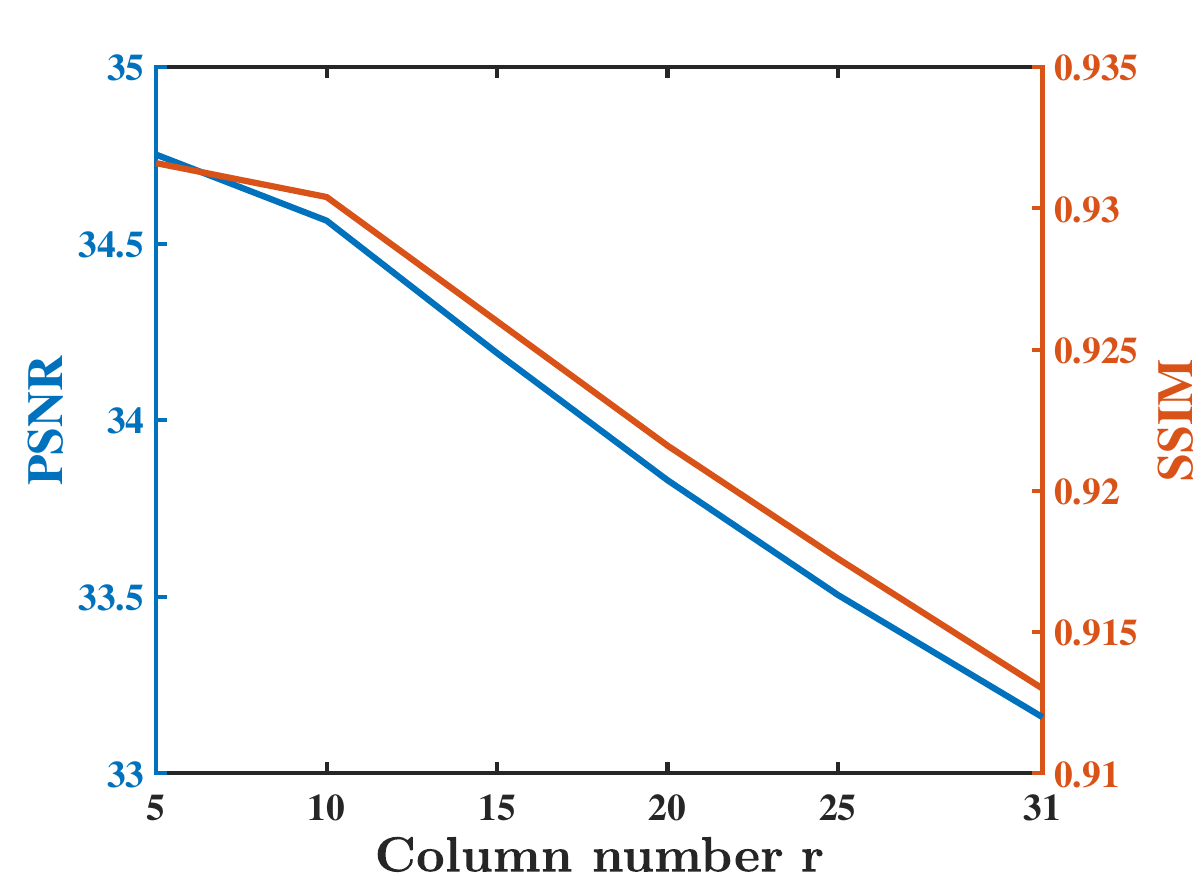}
  \\ (a) & (b)
\end{tabular}
\caption{(a) The AccEgy with the corresponding percentage of singular values of recovered results by NTTNN with different row number $r$ of $\mathbf{T}$. (b) The PSNR and SSIM values of NTTNN with different row number $r$ of $\mathbf{T}$. }\label{disr}
\end{figure}
\subsection{The indispensability of $\mathbf{T}$ and $\phi$}\label{disTandpsipart}
In this part, we analyze the effectiveness of $\mathbf{T}$ and $\phi$ in the proposed nonlinear transform $\psi$ by reserving only $\mathbf{T}$ or $\phi$, which are denoted as NTTNN(linear) and NTTNN(nonlinear), respectively. We conduct the numerical experiment on MSI $\mathit{Toy}$ by NTTNN(linear), NTTNN(nonlinear), and NTTNN with SR 5\%, 10\%, and 15\%, respectively.

Fig. \ref{disTandpsi} plots curves of the AccEgy with the corresponding percentage of singular values of recovered results by NTTNN(linear), NTTNN(nonlinear), and NTTNN with SR 5\%, 10\%, and 15\%, respectively. We can observe that the linear transform $\mathbf{T}$ and nonlinear transform $\phi$ together contribute to the most concentrate energy of
the singular values of recovered results of the proposed NTTNN as compared with the linear transform $\mathbf{T}$ alone and the nonlinear transform $\phi$ alone, i.e., NTTNN(linear) and NTTNN(nonlinear), respectively.

Moreover, Table \ref{discuss3} reports PSNR, SSIM, and SAM values of the recovered MSI $\mathit{Balloons}$ by NTTNN(linear), NTTNN(nonlinear), and NTTNN for different SRs. We can observe that NTTNN outperforms NTTNN(linear) and NTTNN(nonlinear) in terms of PSNR, SSIM,
and SAM values due to NTTNN can obtain the most concentrate energy of
the singular values of recovered results.
Therefore, we suggest the composite nonlinear transform $\psi$ consisting of $\mathbf{T}$ and $\phi$ to obtain a better low-rank approximation of the transformed tensor.
\begin{figure}[htp!]
  \setlength{\tabcolsep}{0.01cm}
  \centering
  \begin{tabular}{ccc}
  \includegraphics[width=0.33\textwidth]{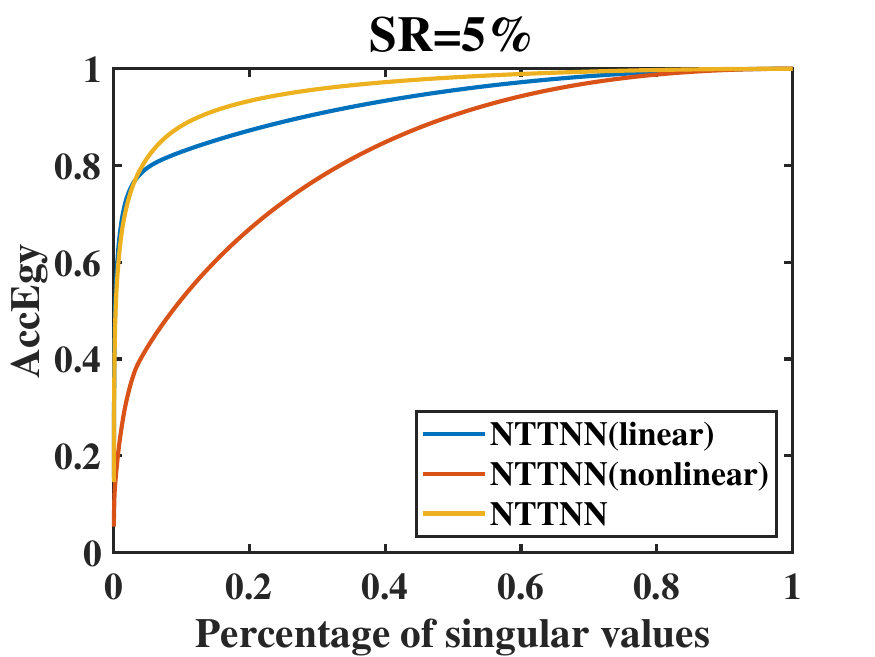}&
  \includegraphics[width=0.33\textwidth]{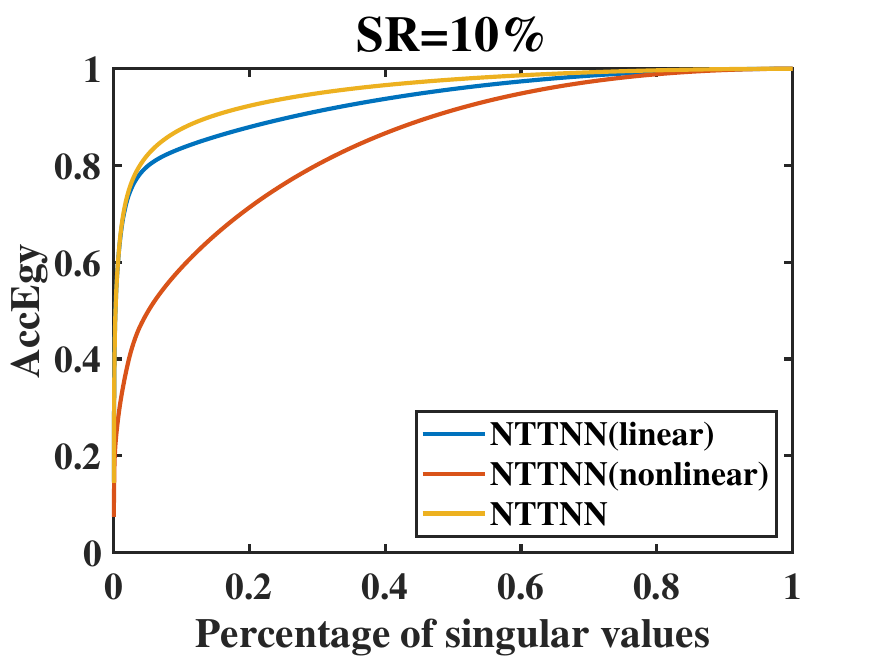}&
  \includegraphics[width=0.33\textwidth]{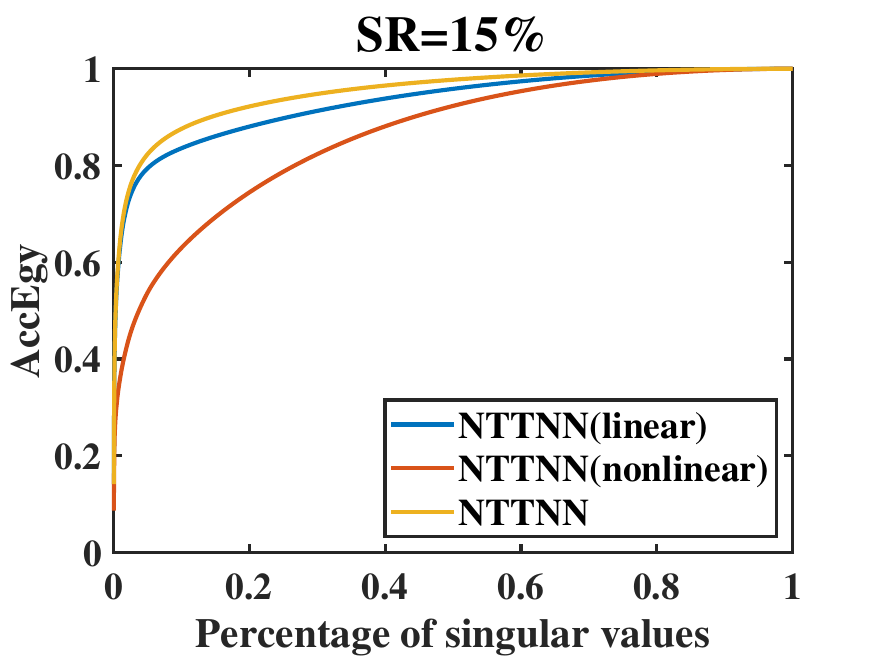}
\end{tabular}
\caption{The AccEgy with the corresponding percentage of singular values of recovered results by NTTNN(linear), NTTNN(nonlinear), and NTTNN with SR 5\%, 10\%, and 15\%, respectively.}\label{disTandpsi}
\end{figure}
\begin{table}[htp!]
\centering
\renewcommand\arraystretch{1.5}
\setlength{\tabcolsep}{1.3mm}
\caption{The PSNR, SSIM, and SAM values of the recovered MSI $\mathit{Balloons}$ by NTTNN(linear), NTTNN(nonlinear), and NTTNN for different SRs.}\vspace{0.5cm}
\begin{tabular}{c|lll|lll|lll}
\Xhline{1.5pt}
 \multirow{2}{*}{Methods} & \multicolumn{3}{c|}{SR=5\%} & \multicolumn{3}{c|}{SR=10\%} & \multicolumn{3}{c}{SR=15\%} \\ \cline{2-10}
                                                      & PSNR    & SSIM    & SAM     & PSNR    & SSIM    & SAM     & PSNR     & SSIM    & SAM     \\ \Xhline{1.5pt}
 \multicolumn{1}{c|}{Observed} &13.349&0.0959 &1.5163 &13.762&0.1613&1.2774 &14.011 &0.1896& 1.1934\\
 \multicolumn{1}{c|}{NTTNN(linear)} &34.205&0.9322&0.1476&37.890&0.9643&0.1115&39.870&0.9746&0.0958\\
 \multicolumn{1}{c|}{NTTNN(nonlinear)} &19.823&0.4974&0.3798&24.280&0.6990&0.2521&27.477&0.8090&0.1898\\
  \multicolumn{1}{c|}{NTTNN} &\textBF{35.425}&\textBF{0.9387}&\textBF{0.1268}&\textBF{40.458}&\textBF{0.9757}&
  \textBF{0.0784}&\textBF{43.633}&\textBF{0.9865}&\textBF{0.0600}\\\Xhline{1.5pt}
\end{tabular}\label{discuss3}
\end{table}

\subsection{Effectiveness of nonlinear transform}\label{nonlinearcomparison}
In this subsection, we further verify the effectiveness of nonlinear transform in the proposed framework. Specifically,
we compare the performance of NTTNN without nonlinear function $\phi$ (denoted as NTTNN(linear)) and NTTNN with different
nonlinear transforms, i.e., Sigmoid function \cite{nonlinearfunction}, Softplus function \cite{softplus}, and Hyperbolic tangent (Tanh) function \cite{nonlinearfunction}.

Table \ref{discuss1} reports the PSNR, SSIM, and SAM values of the recovered HSI $\mathit{WDC~Mall}$ by NTTNN with different nonlinear function for different SRs. We can observe that NTTNN(Tanh) obtains the best recovered results for SR 5\% and 10\%, while NTTNN(Sigmoid) obtains the best recovered results for challenging case SR 1\%. Additionally,
the performance of NTTNN with different nonlinear functions compared with NTTNN(linear) is improved, which demonstrates the nonlinear function play an important role in our NTTNN framework.

\begin{table}[htp!]
\centering
\renewcommand\arraystretch{1.5}
\setlength{\tabcolsep}{1.3mm}
\caption{The PSNR, SSIM, and SAM values of the recovered HSI $\mathit{WDC~Mall}$ by NTTNN with different nonlinear functions for different SRs.}\vspace{0.5cm}
\begin{tabular}{c|lll|lll|lll}
\Xhline{1.5pt}
 \multirow{2}{*}{Methods} & \multicolumn{3}{c|}{SR=1\%} & \multicolumn{3}{c|}{SR=5\%} & \multicolumn{3}{c}{SR=10\%} \\ \cline{2-10}
                                                      & PSNR    & SSIM    & SAM     & PSNR    & SSIM    & SAM     & PSNR     & SSIM    & SAM     \\ \Xhline{1.5pt}
 \multicolumn{1}{c|}{Observed} &13.370& 0.0083& 1.4969&13.549&0.0278& 1.3559& 13.784&0.0491& 1.2556\\
 \multicolumn{1}{c|}{NTTNN(linear)} &24.863&0.6188&0.1928&33.882&0.9337&0.0787&38.635&0.9766&0.0474\\
 \multicolumn{1}{c|}{NTTNN(Sigmoid)} &\textBF{26.523}&\textBF{0.7142}&\textBF{0.1723}&34.287&0.9414&0.0663&41.462&0.9885&0.0260\\
  \multicolumn{1}{c|}{NTTNN(Softplus)} &25.987&0.6866&0.1881&34.713&0.9465&0.0651&42.159&0.9904&0.0247\\
 \multicolumn{1}{c|}{TTNN(Tanh)}&25.558&0.6749&0.1849&\textBF{36.402}&\textBF{0.9643}&\textBF{0.0536}
&\textBF{43.251}&\textBF{0.9930}&\textBF{0.0219}\\ \Xhline{1.5pt}

\end{tabular}\label{discuss1}
\end{table}
\subsection{Comparison of different initialization}\label{lowrankcomparison}
In this subsection, we discuss the performance of NTTNN with different initialization of $\mathcal{X}$.
We consider initializing $\mathcal{X}^0$ by the following: the observed tensor, the result of TNN method, and the linear interpolation of $\mathcal{X}^0$, which denote as NTTNN(Observed), NTTNN(TNN), and NTTNN(Interpolation), respectively.

Table \ref{discuss2} reports the PSNR, SSIM, and SAM values of the recovered HSI $\mathit{WDC~Mall}$ by NTTNN(Observed), NTTNN(TNN), and NTTNN(Interpolation) for different SRs.
We can observe that NTTNN(Interpolation) and NTTNN(TNN) outperform NTTNN(Observed) for all SRs, which demonstrates using good and low computational cost initialization can improve the performance
of NTTNN. Additionally, NTTNN(Interpolation) outperforms NTTNN(TNN) for extremely low SRs 1\% and 5\%, and both of them obtain good performance for relatively high SR 10\%.
The reason behind this phenomenon is the interpolation method outperforms the TNN method for extremely low SRs.
Therefore, throughout all the experiments in this paper, we employ good and low computational cost linear interpolation strategy to fill in the missing pixels and obtain $\mathcal{X}^0$  for TTNN, DTNN, and our method.

\begin{table}[htp!]
\centering
\renewcommand\arraystretch{1.4}
\setlength{\tabcolsep}{1.4mm}
\caption{The PSNR, SSIM, and SAM values of the recovered HSI $\mathit{WDC~Mall}$ by NTTNN with different initialization for different SRs.}\vspace{0.5cm}
\begin{tabular}{c|lll|lll|lll}
\Xhline{1.5pt}
 \multirow{2}{*}{methods} & \multicolumn{3}{c|}{SR=1\%} & \multicolumn{3}{c|}{SR=5\%} & \multicolumn{3}{c}{SR=10\%} \\ \cline{2-10}
                                                      & PSNR    & SSIM    & SAM     & PSNR    & SSIM    & SAM     & PSNR     & SSIM    & SAM     \\ \Xhline{1.5pt}
 \multicolumn{1}{c|}{NTTNN(Observed)} &21.391&0.4825&0.3212&33.980&0.9347&0.0758&40.193&0.9802&0.0510\\
 \multicolumn{1}{c|}{NTTNN(TNN)}                      &22.651&0.5544&0.2283&35.863&0.9544&0.0554&\textBF{43.969}&0.9919&0.0256\\
 \multicolumn{1}{c|}{NTTNN(Interpolation)}                  &\textBF{25.558}&\textBF{0.6749}&\textBF{0.1849}&\textBF{36.402}&\textBF{0.9643}&\textBF{0.0536}
 &43.251&\textBF{0.9930}&\textBF{0.0219}\\ \Xhline{1.5pt}
\end{tabular}\label{discuss2}
\end{table}
\subsection{Numerical convergence}
In this subsection, we evaluate the numerical convergence of the PAM-based algorithm for the proposed method to validate the theoretical convergence. Taking the HSI $\mathit{WDC~Mall}$, MSI $\mathit{Balloons}$, and video $\mathit{Carphone}$ for different SRs as examples, Fig. \ref{convergence} displays the relative change curves of the proposed PAM-based algorithm. We can clearly observe that the relative error decreases as the number of iterations increase, demonstrating the numerical convergence of the proposed PAM-based algorithm.
\begin{figure}[htp!]
  \setlength{\tabcolsep}{0.01cm}
  \centering
  \begin{tabular}{ccc}
  \includegraphics[width=0.33\textwidth]{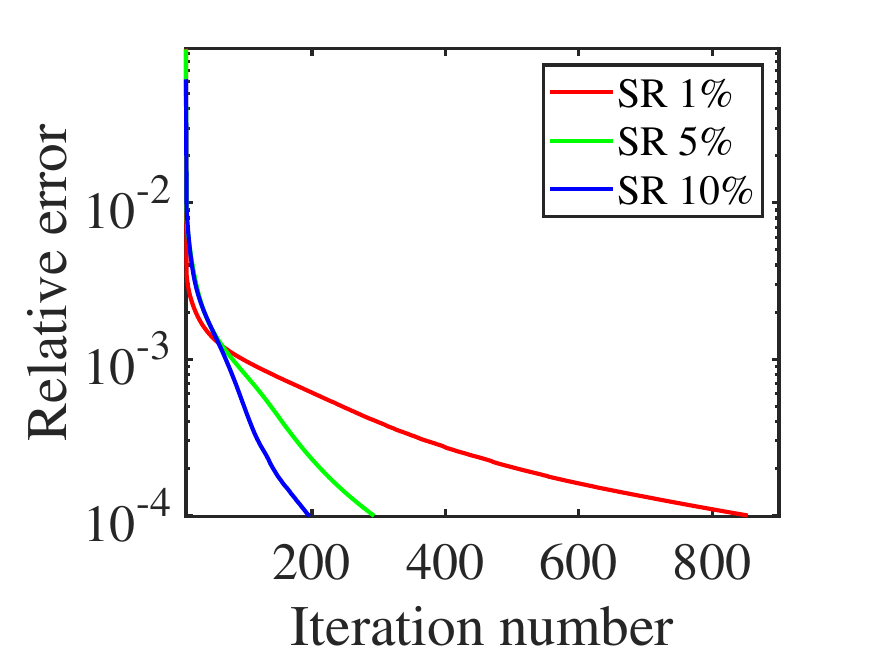}&
  \includegraphics[width=0.33\textwidth]{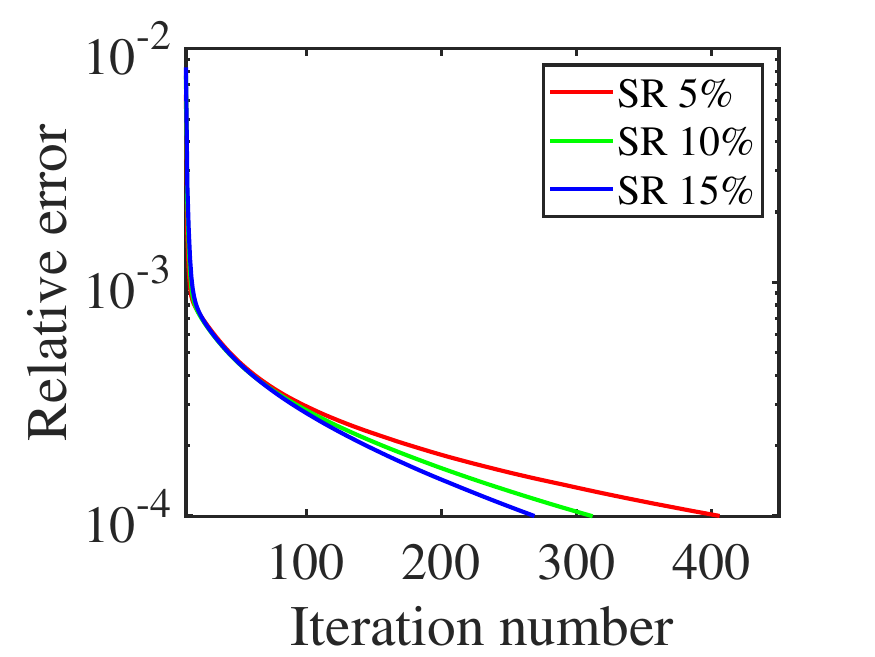}&
  \includegraphics[width=0.33\textwidth]{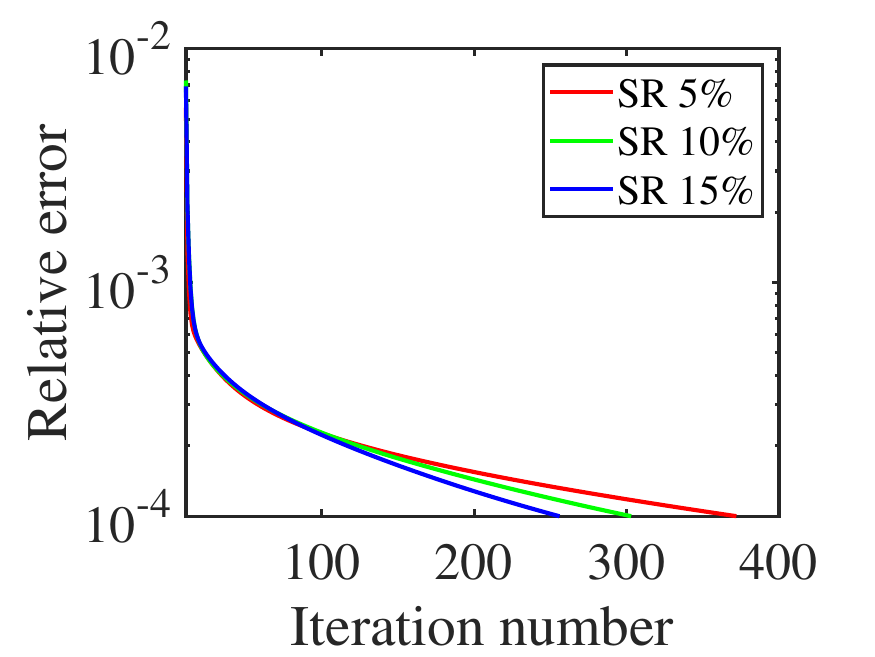}
  \\ (a) & (b) & (c)
\end{tabular}
\caption{Curves of relative errors versus iterations. (a) HSI $\mathit{WDC~Mall}$. (b) MSI $\mathit{Balloons}$. (c) Video $\mathit{Carphone}$.}\label{convergence}
\end{figure}
\section{Conclusion}\label{section6}
In this paper, we proposed the nonlinear transform for the underlying tensor and developed the corresponding nonlinear transform-based TNN (NTTNN). More concretely, the proposed
nonlinear transform is a composite transform consisting  of the linear semi-orthogonal transform along the third mode and the element-wise nonlinear transform on frontal slices of the tensor under the linear semi-orthogonal transform, which are indispensable and complementary in the composite transform to fully exploit the underlying low-rankness.
The proposed NTTNN could enhance the low-rank approximation of the underlying tensor and can be regarded as a unified transform-based TNN family including
many classic transform-based TNN methods.
Moreover, based on the suggested low-rank metric, i.e., NTTNN, we proposed the corresponding LRTC model and developed an efficient PAM-based algorithm. Theoretically, we proved that the sequence generated by the proposed method is bounded and converges to a critical point.
Massive experimental results on different types of
multi-dimensional images show that NTTNN reconstructs better results compared to the state-of-the-art linear transform-based TNN methods quantitatively and visually.
{\small
\bibliographystyle{IEEEtran}
\bibliography{NTTNN_reference}
}

\end{document}